\documentclass{article}

\usepackage[
  journal=FoCM,
  lang=american,
  mode=submission,
  cite=numeric
]{ems-journal}

\usepackage{mathtools}
\mathtoolsset{showonlyrefs}
\usepackage{physics}
\usepackage{bm}
\usepackage{tikz}
\usepackage{makecell}
\usepackage{pgfplots}
\usepackage{dsfont}
\usepackage{bigints}
\usepackage{mathrsfs}
\DeclareFontFamily{U}{rsfs}{\skewchar\font127 }
\DeclareFontShape{U}{rsfs}{m}{n}{%
  <5> <5.5> <6> rsfs5
  <7> rsfs7
  <8> <9> <10> <10.95> <12> <14.4> <17.28> <20.74> <24.88> rsfs10
}{}
\usepackage{cleveref}

\newcommand{\BlackBox}{\rule{1.5ex}{1.5ex}}
\newenvironment{noheadproof}[1][]{\par\noindent\ignorespaces}{\par}

\theoremstyle{plain}
\newtheorem{theorem}{Theorem}
\newtheorem{lemma}{Lemma}[section]
\newtheorem{prop}{Proposition}
\newtheorem{corollary}{Corollary}
\theoremstyle{definition}
\newtheorem{assumption}{Assumption}
\newtheorem{definition}{Definition}
\newtheorem{remark}{Remark}

\crefname{prop}{proposition}{propositions}
\Crefname{prop}{Proposition}{Propositions}
\crefname{assumption}{assumption}{assumptions}
\Crefname{assumption}{Assumption}{Assumptions}

\usetikzlibrary{arrows.meta, positioning}
\pgfplotsset{compat=1.18} % adjust if needed

\newcommand{\mS}{\m{S}}
\newcommand{\im}{\operatorname{Im}}
\newcommand{\mA}{\m{A}} % sigma algebra A
\newcommand{\N}{\mb{N}} % natural numbers
\newcommand{\bomega}{\bm{\omega}} % bold omega
\newcommand{\R}{\mb{R}} % reals
\newcommand{\Rnn}{\R^{\ge 0}} % non-negative reals
\newcommand{\Zp}{\mb{Z}^{>0}} % positive integers
\renewcommand{\S}{\mb{S}} % sphere
\newcommand{\I}{\mathds{1}} % indicator function: 1_A
\renewcommand{\d}{\mm{d}} % dx
\renewcommand{\b}{\mathbf} % bold maths
\newcommand{\tb}{\textbf} % bold text
\newcommand{\kB}{k^{(\mm{B})}} % Brownian kernel 
\newcommand{\ku}{k_{u}} 
 
\newcommand{\T}{\top} % transpose
\renewcommand{\t}{\text} % text
\newcommand{\lf}{f^{(1:L)}} % elements of an L-BKL
\newcommand{\lmu}{\mu^{(1:L-1)}}
\newcommand{\lmuh}{\hat\mu^{(1:L-1)}}
\newcommand{\Slmu}{\mS_{\mu}^{(l)}}
\newcommand{\Simu}{\mS_{\mu}^{(i)}}
\newcommand{\Sonemu}{\mS_{\mu}^{(1)}}
\newcommand{\Slpmu}{\mS_{\mu}^{(l+1)}}
\newcommand{\sSmu}{\mS^{(1:L-1)}_{\mu}}
\newcommand{\sSmuh}{\mS^{(1:L-1)}_{\hat\mu}}
\newcommand{\slSmu}{\mS^{(1:l-1)}_{\mu}}

\newcommand{\LHmuh}{\LH_{\hat\mu}} % l = ladder
\newcommand{\LHmu}{\LH_{\mu}} % l = ladder
\newcommand{\LH}{\mH^{(1:L)}} % L-level 
\newcommand{\LHh}{\widehat\mH^{(1:L)}} % L-level 
\newcommand{\HLmuh}{\mH_{\hat\mu}^{(L)}}

\newcommand{\Hlmu}{\mH_{\mu}^{(l)}}
\newcommand{\Hlmuh}{\mH_{\hat\mu}^{(l)}}
\newcommand{\SHlmu}{\S_1\left(\Hlmu\right)}
\newcommand{\Honemu}{\mH_{\mu}^{(1)}}
\newcommand{\Honemuh}{\mH_{\hat\mu}^{(1)}}
\newcommand{\Hlpmu}{\mH_{\mu}^{(l+1)}}
\newcommand{\Hlpmuh}{\mH_{\hat \mu}^{(l+1)}}
\newcommand{\SHlpmu}{\S_1\left(\Hlpmu\right)}
\newcommand{\Hlhmu}{\mH_{\hat\mu}^{(l+1)}}
\newcommand{\Hlhmul}{\mH_{\hat\mu}^{(l)}}

\newcommand{\HL}{\widehat \mH^{(L)}} 
\newcommand{\HLplus}{\widehat \mH^{(L+1)}} 
\newcommand{\Hl}{\mH^{(l)}} % l layer function space

\newcommand{\CL}{\widehat C^{(L)}}

\newcommand{\BKLdict}{\mathfrak{M}_L}
\newcommand{\HLdict}{\mH_{\mathfrak{M}_L}^{(L)}}
\newcommand{\CLdict}{C_{\mathfrak{M}_L}^{(L)}}
\newcommand{\BLdictr}{\m{B}_{\mathfrak{M}_L,r}^{(L)}}
\newcommand{\BL}{\m{B}^{(L)}}

\newcommand{\Psph}{\mP_{1,1}}

\newcommand{\eG}{\widehat{\m{G}}_n} % empirical Gaussian complexity
\newcommand{\G}{\m{G}_n} % empirical Gaussian complexity
\newcommand{\DE}{\mb{E}_{\nu_X^n}} % 
\newcommand{\eE}{\mb{E}_{\m{N}^n(0,1)}}% 
\newcommand{\eDE}{\mb{E}_{\m{N}^n(0,1),\nu_X^n}} % 
\newcommand{\E}{\mb{E}} % expectation
\newcommand{\F}{\m{F}} % function space
\newcommand{\B}{\mb{B}} % unit ball of space
\newcommand{\mP}{\m{P}} % probability measures
\newcommand{\mU}{\m{U}} % domain
\newcommand{\mY}{\m{Y}} % output domain
\newcommand{\mX}{\m{X}} % input domain
\newcommand{\mZ}{\m{Z}} % Image of mapp
\newcommand{\mD}{\m{D}} % domain
\newcommand{\mC}{\m{C}} % continuous functions 
\newcommand{\mN}{\m{N}} % normal random variable
\newcommand{\mH}{\m{H}} % RKHS H = H_k
\newcommand{\mM}{\m{M}} % minimizers in RERM
\newcommand{\C}{\mt{C}} % complexity of function class
\newcommand{\mR}{\m{R}} % expected risk
\newcommand{\mE}{\m{E}} % excess risk

\newcommand{\mb}[1]{\mathbb{#1}}
\newcommand{\m}[1]{\mathcal{#1}}

\newcommand{\mm}[1]{\mathrm{#1}}
\newcommand{\mt}[1]{\mathit{#1}}
\newcommand{\BK}{BKL} %BKERNHLs 
\newcommand{\dL}{d^{(L)}}

\renewcommand{\O}{\mathcal{O}} % O()
\renewcommand{\P}{\mathbb{P}} % probability measure P

\DeclareMathOperator*{\argmin}{arg\,min} % argmin
\DeclareMathOperator*{\Span}{Span}

\definecolor{lightgreen1}{rgb}{0.6, 0.8, 0.6} % Light green
\definecolor{pastelgreen}{rgb}{0.6, 1.0, 0.6}

\allowdisplaybreaks

\setenumerate{labelindent=0em,leftmargin=1.7em,topsep=0cm,partopsep=0cm,parsep=0cm,itemsep=1mm}
\setitemize{labelindent=0em,leftmargin=1.3em,topsep=0cm,partopsep=0cm,parsep=0cm,itemsep=1mm}

\newcommand{\figw}{0.36\linewidth}   % Default figure width
\newcommand{\figh}{0.33\linewidth}   % Default figure height

\pgfplotsset{
  compat=1.18, % or your preferred PGFPlots version
  every axis/.append style={
    width=\figw,
    height=\figh,
    xlabel={$u$},
    ylabel={$\ell(u,0)$},
    xmin=-2, xmax=2,
    ymin=0, ymax=4,
    grid=major,
    axis equal image,
    ticklabel style={font=\scriptsize},
    label style={font=\scriptsize},
    title style={font=\scriptsize, yshift=-1ex},
    clip=false,
    scale only axis,
  },
}

\usetikzlibrary{calc,arrows.meta,positioning,backgrounds}

\DeclareMathOperator{\supp}{supp}

\numberwithin{equation}{section}

\hypersetup{
  pdftitle={Brownian Kernel Ladders},
  pdfauthor={Mahdi Mohammadigohari, Giuseppe Di Fatta, Giuseppe Nicosia, Panos M. Pardalos},
  pdfsubject={Hierarchical integral RKHSs and statistical learning theory},
  pdfkeywords={brownian kernel ladders, reproducing kernel Hilbert spaces, hierarchical function spaces, gaussian complexity, statistical learning theory}
}

\begin{document}

\title{Brownian Kernel Ladders}
\titlemark{Brownian Kernel Ladders}
\authormark{M. Mohammadigohari, G. Di Fatta, G. Nicosia, and P. M. Pardalos}

\emsauthor*{1}{
  \givenname{Mahdi}
  \surname{Mohammadigohari}
  \mrid{}
  \zblid{}
  \orcid{}}{M.\ Mohammadigohari}
\emsauthor{2}{
  \givenname{Giuseppe}
  \surname{Di Fatta}
  \mrid{}
  \zblid{}
  \orcid{}}{G.\ Di Fatta}
\emsauthor{3}{
  \givenname{Giuseppe}
  \surname{Nicosia}
  \mrid{}
  \zblid{}
  \orcid{}}{G.\ Nicosia}
\emsauthor{4}{
  \givenname{Panos M.}
  \surname{Pardalos}
  \mrid{}
  \zblid{}
  \orcid{}}{P. M.\ Pardalos}

\Emsaffil{1}{
  \department{Faculty of Engineering}
  \organisation{Free University of Bozen--Bolzano}
  \rorid{}
  \address{Via Bruno Buozzi 1}
  \zip{39100}
  \city{Bolzano}
  \country{Italy}
  \affemail{Mahdi.Mohammadigohari@gmail.com}}
\Emsaffil{2}{
  \department{Faculty of Engineering}
  \organisation{Free University of Bozen--Bolzano}
  \rorid{}
  \address{Via Bruno Buozzi 1}
  \zip{39100}
  \city{Bolzano}
  \country{Italy}
  \affemail{giuseppe.difatta@unibz.it}}
\Emsaffil{3}{
  \department{Department of Biomedical and Biotechnological Sciences}
  \organisation{University of Catania}
  \rorid{}
  \address{Via Santa Sofia 97}
  \zip{95123}
  \city{Catania}
  \country{Italy}
  \affemail{giuseppe.nicosia@unict.it}}
\Emsaffil{4}{
  \department{Center for Applied Optimization, Department of Industrial and Systems Engineering}
  \organisation{University of Florida}
  \rorid{}
  \address{}
  \zip{FL 32611}
  \city{Gainesville}
  \country{USA}
  \affemail{pardalos@ise.ufl.edu}}

\classification[68T05, 68Q32]{46E22}

\begin{abstract}
Constructing hierarchical function spaces that combine analytical tractability, genuinely increasing expressivity with depth, and explicit statistical control remains a central problem at the interface of kernel methods and learning theory. We introduce Brownian kernel ladders (BKLs), a recursive hierarchy of integral reproducing kernel Hilbert spaces generated from linear functionals by repeatedly integrating Brownian pullback kernels indexed by functions from the preceding layer. The nonnegative $1$-homogeneity of the Brownian kernel yields a kernel-preserving canonical spherical normalization of the layer measures and propagates square-root regularity through the hierarchy.

Allowing all canonical ladder measures to vary produces a full adaptive BKL envelope equipped with an infimal complexity functional. For this envelope, we establish depth-dependent H\"older and pointwise estimates, quasi-Banach structure, and nestedness. Under a geometric trace condition, every additional layer strictly enlarges the function class, with a separating function at positive uniform distance from each fixed-radius ball of the preceding level. We also prove existence of regularized empirical-risk minimizers for continuous losses that are uniformly bounded below. Under strict convexity, any population minimizers agree almost everywhere with respect to the input distribution, and agree pointwise when that distribution has full support.

For statistical estimation, we study one realized ladder and finite dictionaries of canonical ladders fixed independently of the estimation sample. For a dictionary of $M$ ladders, the Gaussian complexity of the union of radius-$r$ top-layer RKHS balls has the standard $n^{-1/2}$ dependence, no explicit ambient-dimension factor, and the explicit model-selection factor $1+\sqrt{2\ln M}$. We derive corresponding high-probability penalized oracle and excess-risk bounds for radius-constrained controlled learning; polynomial-size dictionaries retain a near-parametric rate. The resulting theory provides a two-level foundation for hierarchical kernel learning: an adaptive envelope for representational structure and controlled realized families for statistically accountable estimation.

\end{abstract}

\maketitle
\thispagestyle{plain}

\section{Introduction}\label{sec:intro}

Deep neural networks organize computation hierarchically and achieve remarkable empirical success across a broad range of applications \cite{lecun2015deep,goodfellow14generative}. From a function-space perspective, this motivates recursive constructions in which new hypothesis classes are generated layer-by-layer from preceding ones. Reproducing kernel Hilbert spaces (RKHSs) offer analytical transparency, flexible nonparametric approximation, and a mature statistical theory, but a single RKHS does not itself provide a depth-indexed hierarchy whose regularity and expressivity can be compared across levels. A central problem is therefore to construct recursive function spaces that remain mathematically explicit, become genuinely richer with depth, and permit a precise account of the statistical cost of choosing a representation.

In this work, we introduce \emph{Brownian kernel ladders} (BKLs), a hierarchy of recursive integral RKHSs driven by the Brownian kernel
\begin{align*}
\kB\left(x,x'\right)
=
\frac{|x|+|x'|-|x-x'|}{2},
\qquad
x,x'\in\R.
\end{align*}
The first layer consists of linear functionals. At every subsequent layer, functions from the preceding RKHS index Brownian pullback kernels, which are integrated against a probability measure to define the next kernel and hence the next RKHS. Depth is therefore encoded directly in the function-space construction rather than through a prescribed finite-dimensional network parameterization. The Brownian kernel is especially well suited to this recursion: its nonnegative $1$-homogeneity supports kernel-preserving normalization of the layer measures, while its square-root feature geometry propagates explicit H\"older regularity and diagonal bounds through the hierarchy.

The BKL construction naturally gives rise to two complementary mathematical objects. Allowing the complete sequence of canonical ladder measures to vary produces the \emph{full adaptive canonical envelope}; the complexity of a function is the infimum of its top-layer RKHS norm over all admissible ladder realizations. This envelope is the appropriate object for understanding the global geometry, expressivity, and variational structure of the representation family. Statistical estimation raises a different question: once a ladder, or a controlled collection of candidate ladders, has been specified independently of the estimation sample, how large is the resulting model class? This distinction is substantive rather than notational. A supremum over representations placed inside a Gaussian expectation describes unrestricted adaptive selection, whereas a fixed ladder or a finite dictionary incurs a separately identifiable selection cost. BKLs make this representation--estimation separation explicit within one recursive kernel framework.

For the full adaptive envelope, we obtain a structural theory across arbitrary fixed depth. Every depth-$L$ function satisfies a $2^{-(L-1)}$-H\"older estimate and a matching pointwise bound controlled by the infimal BKL complexity. The envelope equipped with this complexity is a quasi-Banach space. The hierarchy is nested, and under a geometric trace condition each additional layer strictly enlarges the function class; moreover, the separating witness remains at positive uniform distance from every fixed-radius ball of the preceding level. We also prove variational well-posedness of regularized empirical-risk minimization over the full envelope for continuous losses that are uniformly bounded below. If the loss is strictly convex in its prediction argument, any population minimizers agree $\nu_X$-almost everywhere, and they agree pointwise when $\nu_X$ has full support.

For statistical estimation, let $\BKLdict$ be a dictionary of $M$ canonical ladders fixed independently of the estimation sample, and let $\BLdictr$ be the union of their radius-$r$ top-layer RKHS balls. We prove
\begin{align*}
\G\left(\BLdictr\right)
\le
\frac{r\left(1+\sqrt{2\ln M}\right)}{\sqrt n}
\left(
\DE\left[
\max_{i\in[n]}\|X_i\|_2^2
\right]
\right)^{2^{-L}}.
\end{align*}
The estimate contains no explicit ambient-dimension factor. It has the standard $n^{-1/2}$ dependence for one realized ladder and charges only the logarithmic factor $1+\sqrt{2\ln M}$ for selection among $M$ candidates. Depth enters through the explicit input-radius factor rather than through a product of layerwise complexity terms. Building on this estimate, we establish existence of radius-constrained controlled estimators, a high-probability penalized oracle inequality, and an excess-risk bound. Polynomial-size dictionaries retain a near-parametric $\widetilde{\m O}(n^{-1/2})$ rate, while uniform statements over a sequence of depths require the corresponding dictionary sizes to remain controlled.

\medskip\noindent\emph{Relation to Hierarchical Kernel and Function-Space Models.}\ 
Kernel methods and deep learning have been connected through deep kernels, convolutional kernels, neural tangent kernels, and other hierarchical constructions \cite{cho2009kernel,wilson2016b,jacot2018neural,mairal2014convolutional,hashimoto2023deep}. Recursive and compositional transformations of kernels have also been studied directly \cite{huang2023hierarchical}. These approaches generally build depth through successive transformations of kernels or features. BKLs instead form a sequence of integral RKHSs in which the indexing variables at one layer are themselves functions from the preceding layer, and the same scalar Brownian kernel supplies the pullback at every step. This makes the effect of depth accessible through the geometry of one explicit base kernel.

Related function-space descriptions of deep models arise in reproducing kernel Banach spaces (RKBSs), chain RKBSs, Barron spaces, and recursive Banach spaces. Deep neural RKBS constructions \cite{bartolucci2024neural} use vector-valued Radon measures and integral representations of the generic form
\begin{align*}
f(x)
=
\int_{\Theta}
\rho(x,\theta)
\,\d\mu(\theta),
\end{align*}
then compose such Banach-space layers to obtain variational formulations and representer theorems for deep architectures. The chain RKBS framework of Heeringa, Spek, and Brune~\cite{heeringa2025deep} organizes depth through recursive kernel chains and establishes correspondences with finite neural networks on finite data sets. Classical Barron-space theory \cite{barron_universal_1993} represents infinitely wide two-layer networks by measures over ridge features, while the recursive Banach spaces of E and Wojtowytsch~\cite{e2020banach} extend the measure-based viewpoint to multilayer ReLU constructions. These theories provide powerful descriptions of neural representations, approximation, and variational structure. The present work addresses a different combination of questions: each realized BKL remains an RKHS, the union over canonical realizations forms a quasi-Banach adaptive envelope, the hierarchy can be proved to grow strictly with depth, and the cost of selecting among realized representations is kept explicit in the statistical analysis.

The shallow BKL kernel connects directly to Brownian kernel-enhanced random neural networks (BKERNNs) \cite{follain2024enhanced}. When $\operatorname{span}(\mX)=\R^d$, the first-layer RKHS is isometrically identified with $\R^d$; for $L=2$ and a measure $\hat\mu$ on its Euclidean unit sphere,
\begin{align*}
k\left[\S^{d-1},\hat\mu\right]\left(\b x,\b x'\right)
=
\int_{\S^{d-1}}
\kB\left(\bomega^\top\b x,\bomega^\top\b x'\right)
\,\d\hat\mu(\bomega),
\end{align*}
which is precisely the Brownian projection kernel used in BKERNNs. This identity is structural. Statistical comparisons additionally depend on whether the projection measure is fixed, selected from a controlled family, or optimized over an unrestricted adaptive class. Our Gaussian-complexity results concern the first two regimes and therefore do not identify their statistical cost with that of unrestricted measure adaptation.

BKLs are also closely related to neural Hilbert ladders (NHLs) \cite{chen2024neural}, which recursively construct integral RKHSs through activation-product kernels of the form
\begin{align*}
k^{(l+1)}(\b x,\b x')
=
\int_{\mH^{(l)}}
\sigma\left(u(\b x)\right)
\sigma\left(u(\b x')\right)
\,\d\mu^{(l)}(u).
\end{align*}
In BKLs, the corresponding layer kernel is
\begin{align*}
k^{(l)}(\b x,\b x')
=
\int_{\mS^{(l)}}
\kB\left(u(\b x),u(\b x')\right)
\,\d\mu^{(l)}(u).
\end{align*}
The Brownian pullback drives both the recursive H\"older structure used in the depth-separation theorem and the diagonal estimates used in the controlled-family complexity bound. More broadly, statistical analyses of recursive models often accumulate depth dependence through covering numbers, pseudo-dimension, norm propagation, or products of layerwise complexity terms \cite{bartlett2002rademacher,wei2019datadependent,anthony1999neural,bartlett2019nearlytight,neyshabur2015norm,bartlett2017spectrally,golowich2018size,zhang2023nearlytight}. The controlled BKL bound avoids such a layerwise product while retaining the explicit cost of choosing among candidate ladders.

\medskip\noindent\emph{Contributions.}\ 
Our main contributions are as follows.
\begin{enumerate}
\item We introduce Brownian kernel ladders, a recursive hierarchy of integral RKHSs generated from linear functionals through Brownian pullback kernels, and prove a kernel-preserving canonical spherical representation of the layer measures.

\item We characterize the full adaptive canonical BKL envelope through depth-dependent H\"older and pointwise bounds, quasi-Banach structure, nestedness, and strict depth growth with ballwise separation under a geometric trace condition.

\item We prove variational well-posedness of regularized empirical-risk minimization over the full adaptive envelope and give the precise almost-everywhere and full-support uniqueness conclusions for population minimizers.

\item We connect the adaptive envelope to realized ladders and finite controlled dictionaries, thereby isolating the order of representation selection and Gaussian averaging within the BKL formalism.

\item We establish Gaussian-complexity bounds for controlled radius classes with no explicit ambient-dimension factor and with the explicit model-selection factor $1+\sqrt{2\ln M}$. The case $M=1$ recovers the standard $n^{-1/2}$ upper-bound dependence for one realized ladder without a selection factor.

\item We derive high-probability penalized oracle and excess-risk bounds for radius-constrained controlled RERM. Polynomial-size dictionaries retain a near-parametric $\widetilde{\m O}(n^{-1/2})$ rate with an explicit representation-selection cost.
\end{enumerate}

The remainder of the paper is organized as follows. \Cref{sec:notations} introduces notation and integral-RKHS preliminaries. \Cref{sec:newNN-architecture} defines individual BKLs, the full adaptive canonical envelope, and controlled dictionaries of realized ladders. \Cref{sec:characterization} establishes the analytical hierarchy. \Cref{sec:stat-analyses} presents adaptive variational well-posedness and the controlled-family Gaussian-complexity, oracle, and excess-risk results. All proofs are deferred to the appendices.

%\section{Related Work}\label{sec:related}
%\input{./parts/related}

%\section{Preliminaries}\label{sec:Prelim}
%In this section we introduce our notations %(Section~\ref{sec:notations}).

Our main theoretical results are gathered in Table~\ref{caption:main-results}; their principal dependencies are shown in Fig.~\ref{fig:dependency_graph}.

\begin{table}
\centering
\begin{tabular}{lll} \toprule
Result & Content & Page\\ \midrule
\Cref{thm:charachteriz_funspace}&  analytical properties of \BK{} spaces & page~\pageref{thm:charachteriz_funspace}\\ 
\Cref{thm:existence_\BK{}}& adaptive RERM existence and population uniqueness & page~\pageref{thm:existence_\BK{}}\\
\Cref{thm:peeling_all_depth_\BK{}}& controlled-family Gaussian complexity & page~\pageref{thm:peeling_all_depth_\BK{}}\\
\Cref{thm:excess-risk-\BK{}}& controlled-family oracle and excess-risk bounds & page~\pageref{thm:excess-risk-\BK{}}\\
\bottomrule
\end{tabular}
\caption{Main theoretical results; for their relation, see Fig.~\ref{fig:dependency_graph}. For our auxiliary results, see Table~\ref{caption:auxiliary-results}.} 
\label{caption:main-results}
\end{table}

\section{Notation} \label{sec:notations}
Below we introduce a few notations: $\N$, $\Zp$, $\Rnn$, $[n]$, $\prod_{i\in I} A_i$, $A^n$, $\I_E$, $\m{O}(\cdot)$, $\tilde{\m{O}}(\cdot)$,  $\mP(\mU)$, $\E$, $\E_\P[X]$, $\P^n$, $\mN(0,1)$, $\delta_x$, $L^p(\mU, \mu)$, $\|\cdot\|_{L^{p}(\mU, \mu)}$, $\mP_1(\mS)$, $\|\mu\|_{\mS,1}$, $\mP_{1,1}(\mS)$,  $\|\b x\|_p$, $(\cdot)^\T$, $\B_r(\mZ)$, $\S_r(\mZ)$, $\S^{d-1}$,  $\mC$, $\left\|f\right\|_{\infty}$, $\mH_k$, $k_{\b x}$, $\Span(\cdot)$, $\overline{S}$, $\kB$, $\mH_{\oplus}$, $\langle \cdot, \cdot \rangle_{\mH_{\oplus}}$. 
 
The set of natural numbers is denoted by $\N\coloneqq\{0,1,2,\ldots\}$; $\Zp \coloneqq \{1,2,\ldots\}$ stands for positive integers; the set of non-negative reals is $\Rnn$. For $n\in \Zp$, let $[n] \coloneqq \{1,\ldots, n\}$. The Cartesian product of sets $(A_i)_{i\in I}$ is written as $\prod_{i\in I} A_i$; when the cardinality of the index set $I$ is finite ($n\coloneq |I|$) and $A:=A_1=\ldots=A_n$, we use the shorthand $A^n$. For a set $E$, its indicator function is $\I_{E}$:
$\I_{E}(x) = 1$ if $x\in E$ and $\I_{E}(x) = 0$ otherwise. For positive sequences $(a_n)_{n=1}^{\infty}$ and $(b_n)_{n=1}^{\infty}$, we write $a_n = \O(b_n)$ if there exist $C>0$ and $n_0\in \Zp$ such that for all $n\ge n_0$ one has $a_n \le Cb_n$; $\tilde{\O}(a_n)$ coincides with  $\O(a_n)$ apart from discarding polylogarithmic factors of $n$, in other words $\tilde{\O}(a_n)=\O(a_n \mathrm{polylog}(n))$. Let $\mU$ be a topological space enriched with the Borel sigma-field; we denote the space of probability measures on $\mU$ by $\mP(\mU)$. 
Expectation is denoted by $\E$, and $\E_\P[X] \ =\int_{\mX} x\d \P(x)$, where $X \sim \P \in \mP(\mX)$.
The $n$-fold product of a probability measure $\P\in \mP(\mU)$ is denoted by $\P^n$. The standard normal distribution on $\R$ with mean $0$ and variance $1$
is denoted by $\mN(0,1)$. The Dirac measure concentrated at point $x$ is denoted by $\delta_x$. For a probability measure $\mu\in \mP(\mU)$, $p \in [1,\infty)$, the Banach space of real-valued $L^p$-integrable functions on $\mU$ w.r.t.\ $\mu$ is denoted by $L^p(\mU, \mu)$, and for $g \in L^p(\mU, \mu)$, we write $\|g\|_{L^p(\mU, \mu)} \coloneqq \left[\int_{\mU} |g(u)|^p\d \mu(u)\right]^{1/p}$. Let $\mS$ be a measurable subset of a Hilbert space $\mH$. We denote the class of Borel probability measures on $\mS$ with finite first moment by $\mP_1(\mS)
\coloneqq
\left\{
\mu\in\mP(\mS): \|\mu\|_{\mS,1}<\infty
\right\}$, where
\begin{align}
\|\mu\|_{\mS,1} &\coloneqq \int_{\mS}\|u\|_{\mH}\d\mu(u).\label{eq:p-moment of mu over subset}
\end{align}
The set of all Borel probability measures on
$\mS$ with first moment equal to one is denoted by $\mP_{1,1}(\mS)
\coloneqq
\left\{
\mu\in\mP(\mS): \|\mu\|_{\mS,1} = 1
\right\}$. For a vector $\b x \in \R^d$ and $p\in [1,\infty)$, let $\|\b x\|_p \coloneqq \left(\sum_{i\in [d]} |x_i|^p\right)^{\frac{1}{p}}$; $\| \b x\|_\infty \coloneqq \max_{i \in [d]}|x_i|=\lim_{p\rightarrow \infty} \|\b x\|_{p}$. The transpose of a vector $\b x\in \R^d$ is written as $\b x^\T\in\R^{1\times d}$. A functional $\|\cdot\|_\mU : \mU \to [0, \infty)$ is referred to as quasi-norm on a vector space $\mU$ with a constant $K \geq 1$ if it satisfies the following 3 properties. (i) Positive definiteness: for all $f \in \mU$, $\|f\|_\mU > 0$ whenever $f \neq 0$. (ii) Absolute homogeneity: for all $f \in \mU$ and all $c \in \R$, $\|c f\|_\mU = |c| \|f\|_\mU$. (iii) Modified triangle inequality: for all $f, g \in \mU$,
\begin{align}
\|f + g\|_\mU \leq K (\|f\|_\mU + \|g\|_\mU). \label{eq:modified-triangle-ineq}
\end{align}
In other words, the quasi-norm is similar to a norm, but it relaxes the triangle inequality (where one would have $K=1$).\footnote{In \eqref{eq:modified-triangle-ineq}, $K$ is referred to as the modulus of concavity of $\mU$.} A vector space enriched with a quasi-norm is called quasi-normed space, and a quasi-Banach space is a quasi-normed space, where Cauchy sequences are convergent (meant in the sense of the quasi-norm of the space).\footnote{For instance, for $0<p\le1$ the space $\ell^p := \{x=(x_n)_{n\ge1} %\subset \R 
\in \R^{\Zp}: \|x\|_p \coloneq\left(\sum_{n=1}^\infty |x_n|^p\right)^{\frac{1}{p}} < \infty\}$ 
with quasi-norm  $\|\cdot\|_p$ is a quasi-Banach space with constant $K_p = 2^{\frac{1}{p}-1}$; for $p=1$, $\ell^p$ is a Banach space. \label{footnote:lp-as-quasi-B-space}} Let $\mZ$ be a quasi-Banach space. The ball and sphere in $\mZ$ with center zero and radius $r$ are defined by $\B_{r}(\mZ) = \{z \in \mZ \,:\, \|z\|_{\mZ} \leq r\}$ and $\S_{r}(\mZ) = \{z \in \mZ \,:\, \|z\|_{\mZ} = r\}$, respectively. We use the shorthand $\S^{d-1} \coloneqq\S_1\left(\left(\R^d,\|\cdot\|_2\right)\right)$.  A function $f :  \R^{2} \to \R$ is called non-negatively $1$-homogeneous if for all $x,y  \in \R$ and for all $a\in \Rnn$ it holds that $f(ax,ay) = a f(x, y)$. Throughout the paper, we assume that our input domain is a compact set $\mX \subset \R^d$, and it is fixed. Let $\mC=\mC(\mX)$ be the space of continuous real-valued functions on $\mX$ equipped with the supremum norm $\|f\|_{\infty}\coloneqq\sup_{\b x\in \mX}|f(\b x)|$ and the associated Borel sigma-algebra. 

A kernel is a symmetric function $k : \mX \times \mX \to \R$ such that for all $n \in \Zp$,  all $(\b x_1, \ldots, \b x_n) \in \mX^n$, and all $(\alpha_1, \ldots, \alpha_n) \in \R^n$, $\sum_{i,j \in [n]} \alpha_i k(\b x_i,\b x_j ) \alpha_j \ge 0$ holds. There is a one-to-one correspondence between kernels and reproducing kernel Hilbert spaces (RKHSs), where the  latter is a Hilbert space of functions $\mH_k \subseteq \R^{\mX}$ 
such that $k_\b x \coloneqq k(\cdot,\b x) \in \mH_k$ for all $\b x \in \mX$ and $h(\b x) = \langle h, k_\b x\rangle_{\mH_k}$
for all $(h, \b x) \in \mH_k \times \mX$; the first property describes the building blocks of the space, the second is called the reproducing property.\footnote{The function $k(\cdot,\b x)$ stands for $\b x' \in \mX \mapsto k(\b x',\b x) \in \R$ while keeping $\b x \in \mX$ fixed.} It is known that $\mH_k = \overline{\Span(k_\b x \,:\, x\in \mX)}$, where $\Span(\cdot)$ stands for the linear hull of its argument and $\overline{\cdot}$ denotes closure. The Brownian kernel is
\begin{align}
\kB(x,x') = \frac{|x| + |x'| - |x - x'|}{2} \quad (x,x'\in \R); \label{def:BK-kernel}
\end{align}
it is non-negatively 1-homogeneous.\footnote{Indeed, for any $x,x'\in \R$ and $a\in \Rnn$ one has $\kB(ax,ax') = \frac{|ax|+|ax'|-|ax-ax'|}{2} = |a| \frac{|x|+|x'|-|x-x'|}{2}=a \frac{|x|+|x'|-|x-x'|}{2}=a \kB(x,x')$.} Let us be given (i) a set $\mD$, (ii) an index set $\Omega$ which is a topological space enriched with a Borel probability measure $\mu\in \mP(\Omega)$, (iii) for each $\omega \in \Omega$ a kernel $k^{(\omega)}$: $\mD \times \mD \to \R$ with its associated RKHS  $\mH_{k^{(\omega)}}$,  and assume that
\begin{align} 
\int_{\Omega}k^{(\omega)}(x,x) \d \mu(\omega) &<\infty \quad \text{for all}\,\, x \in \mD. \label{eq:int-kernel-condition}
\end{align}
By denoting 
\begin{align}
\mH_{\oplus} &= \left\{ \left(f_{\omega}\right)_{\omega \in \Omega} \in \prod_{\omega \in \Omega} \mH_{k^{(\omega)}}\, :\,  \int_{\Omega} \|f_{\omega}\|_{\mH_{k^{(\omega)}}}^2 \d\mu(\omega) < \infty \right\}, \label{eq:Hoplus}
\end{align}
enriched with the inner product $\langle f, g \rangle_{\mH_{\oplus}} = \int_{\Omega} \langle f_{\omega}, g_{\omega} \rangle_{\mH_{k^{(\omega)}}} \d\mu(\omega)$, one can show \cite[Theorem 3.1]{Hotz2012RepresentationBI} that (i) $\mH_{\oplus}$ is a Hilbert space, (ii) 
\begin{align}
\mH_k &= \left\{f:\mD \to \R \,:\, f(x) = \int_{\Omega}f_{\omega}(x)\d \mu(\omega) \,\, \forall x \in \mD,  \t{ with } \left(f_{\omega}\right)_{\omega \in \Omega} \in    \mH_{\oplus} \right\}, \label{eq:Hk-elements}
\end{align}
is an RKHS (called integral RKHS) with kernel
%\begin{align}
$k(x,y) = \int_{\Omega} k^{(\omega)} (x, y) \d \mu (\omega)$ $(x,y\in \mD)$
%\end{align}
and (iii) the norm in $\mH_{k}$ is given by
\begin{align}
\|f\|_{\mH_{k}}^2 &= \inf_{g \in \mH_{\oplus} : f(x) = \int_{\Omega}g_{\omega}(x)\d \mu(\omega),\, \forall x \in \mD} \int_{\Omega} \|g_{\omega}\|_{\mH_{k^{(\omega)}}}^2 \d\mu(\omega). \label{eq:intRKHSnorm}
\end{align}

\section{Brownian Kernel Ladders and Their Complexity} \label{sec:newNN-architecture}

\begin{table}
\centering
\begin{tabular}{ccll} \toprule
Symbol & Name & Defined in\\ \midrule
$\LHmu$ & $L$-level \BK{} associated to $\sSmu$ and $\lmu$ & Def.~\ref{def:\BK{}def}\\ 
$\Hlmu$ & $l$-th layer of $\LHmu$ & Def.~\ref{def:\BK{}def}\\
$\LHmuh$ & $L$-level \BK{} associated to $\left(\SHlmu\right)_{l=1}^{L-1}$ and $\lmuh$ & \eqref{eq:H_{hat-mu}^{(l)}}\\
$\Hlhmul$ &  $l$-th layer of $\LHmuh$ &\eqref{eq:H_{hat-mu}^{(l)}}\\ \midrule
$\LHh$ & full family of $L$-level canonical \BK{}s &\eqref{def:unitsphere\BK{}}\\
$\CL$ & adaptive $L$-complexity & \eqref{\BK{}:unitsphereladder-complexity}\\
$\HL$ & adaptive canonical top-layer envelope & \eqref{def:\BK{}-unitsphereladder-space} \\ \midrule
$\BKLdict$ & controlled dictionary of canonical \BK{}s & Def.~\ref{def:controlled-BKL-family}\\
$\HLdict$ & controlled top-layer class & \eqref{eq:controlled-BKL-top-layer}\\
$\CLdict$ & controlled $L$-complexity & \eqref{eq:controlled-BKL-complexity}\\
$\BLdictr$ & controlled radius-$r$ class & \eqref{eq:controlled-BKL-ball}\\ \bottomrule
\end{tabular}
\caption{Notation for the BKL function spaces and complexity functionals.}
\label{caption:defs}
\end{table}

We introduce in \Cref{sec:\BK{}} our novel kernel-based hierarchical function class called an $L$-level Brownian kernel ladder ($\LHmu$), and then use canonical ladders to define both the full adaptive envelope $\HL$ and controlled families of realized ladders in \Cref{sec:complexity-meaures}. The full envelope is the object of our analytical and variational results, whereas the controlled families provide the statistical model classes used in the generalization analysis. The definitions of this section are summarized in \Cref{caption:defs} for the reader's convenience.

\subsection{Brownian Kernel Ladder} \label{sec:\BK{}}
Next we present our proposed kernel-based deep composition of functions. We start with a construction to generate a specific integral RKHS given
(i) a measurable subset $\mS$ of an RKHS $\mH \subseteq \R^{\mX}$ defined 
on the compact input domain $\mX\subset \R^d$, (ii) a Borel probability measure with finite first moment $\mu \in \mP_1(\mS)$, and (iii) the Brownian kernel $\kB$ [as defined in \eqref{def:BK-kernel}]. The $L$-level application of this procedure (initialized with the linear kernel) will give rise to an $L$-level \BK{} $\LHmu$. This is the design  we detail next.

Let the domain $\mX$, measurable subset $\mS$ of  RKHS $\mH$, probability measure $\mu$ and Brownian kernel $\kB$ as specified above, and let a function $k\left[\mS,\mu\right] \colon \mX \times \mX \to \R$ associated to the pair $(\mS,\mu)$ [in the tuplet $\left(\mX,\mS,\mu, \kB\right)$, $\mX$ and $\kB$ are kept fixed] be defined as 
\begin{align}
    k\left[\mS,\mu\right](\b x,\b x') = \int_{\mS} \underbrace{\kB\left(u(\b x),u\left(\b x'\right)\right)}_{=:\ku\left(\b x,\b x'\right)}\d \mu(u). \label{eq:BK_def}
\end{align}
By the choice $\left(\Omega, \omega, k^{(\omega)}(x,y), \mD, \mu\right) \leftarrow \left(\mS,u,\ku(\b x,\b y), \mX, \mu\right)$ in the integral kernel construction of Section~\ref{sec:notations}, one can show (\Cref{lem:BK}) that condition \eqref{eq:int-kernel-condition} reading in our case as 
\begin{align}
\int_{\mS}\ku(\b x, \b x) \d \mu(u) &<\infty,\,\,\t{for all}\,\,\b x\in \mX, \label{eq:integral-RKHS-condition}
\end{align}
holds, hence $k\left[\mS,\mu\right]$ is a well-defined integral kernel. In our paper, we will use this construction with $\mS = \mH$ and $\mS = \S_1(\mH)$.

Having specified the kernel $k\left[\mS,\mu\right]$, we continue with the definition of the proposed $L$-level \BK{} $\LHmu$ associated to a sequence of measurable subsets of RKHSs $\sSmu:= \left(\Slmu\right)_{l=1}^{L-1}$ and probability measures $\lmu  :=\left(\mu^{(l)}\right)_{l=1}^{L-1}$.

\begin{definition}[$L$-level Brownian kernel ladder]
\label{def:\BK{}def}
Let the level $L\ge 2$ be fixed.
Define the linear kernel
\begin{align}
k^{(0)}(\b x,\b x')
\coloneqq
\b x^\top\b x',
\qquad
\b x,\b x'\in\mX,
\end{align}
and set
\begin{align}
\Honemu
\coloneqq
\mH_{k^{(0)}}.
\end{align}
Equivalently, $\Honemu$ consists of the restrictions to $\mX$ of the linear functionals
$\b x\mapsto\bomega^\top\b x$ with $\bomega\in\R^d$, and it is equipped with the RKHS norm induced by $k^{(0)}$.
For each $l \in [L-1]$, consider a measurable subset $\Slmu \subseteq \Hlmu$ and 
let $\mu^{(l)} \in \mP_1\left(\Slmu\right)$ be a probability measure on $\Slmu$  with finite first moment. 
Form the sequence of integral RKHSs defined recursively as
\begin{align}
k^{(l)} \coloneqq k\left[\Slmu,\mu^{(l)}\right]
\quad \text{defined according to } \eqref{eq:BK_def},
\qquad 
\mH_{\mu}^{(l+1)} \coloneqq \mH_{k^{(l)}},
\quad l\in[L-1].
\end{align}
We refer to the resulting sequence 
\begin{align}
\LHmu \coloneqq \left(\Hlmu\right)_{l=1}^L
\end{align}
as the $L$-level Brownian kernel ladder (\BK) associated to $\sSmu$ and $\lmu$, 
with elements $\lf\coloneqq\left(f^{(l)}\right)_{l=1}^L \in \LHmu$, 
and we call $\Hlmu$ the $l$-th layer of $\LHmu$.\footnote{Notice that 
$\Hlmu$ depends on $\slSmu\coloneq \left(\Simu\right)_{i=1}^{l-1}$ and $\mu^{(1:l-1)}\coloneq \left(\mu^{(i)}\right)_{i=1}^{l-1}$; we use the shorthand $\Hlmu := \Hl_{\slSmu,\mu^{(1:l-1)}}$ to keep the notation light.}
\end{definition}

\begin{assumption}[normalization] \label{assumption:normalizing}
For all $l \in [L-1]$, $\Slmu = \Hlmu$ and $\mu^{(l)} \in \Psph\left(\Hlmu\right)$.
\end{assumption}

\begin{remark}~\label{rem:BKL_def}
\begin{enumerate}
        \item Design of $\LHmu$: 
            The definition of \BK{} can be interpreted as follows. Starting at the first level with linear functionals ($\Honemu$), at each level $l\in[L-1]$ one selects a measurable subset $\Slmu \subseteq \Hlmu$ and a probability measure on $\Slmu$ with finite first moment $\mu^{(l)}\in \mP_1\left(\Slmu\right)$. This subset-measure pair $\left(\Slmu,\mu^{(l)}\right)$ is used to construct the integral kernel $k^{(l)}$ via \eqref{eq:BK_def} by applying the substitution 
            \begin{align}
            (\mH,\mS,\mu)\ \leftarrow\ \left(\Hlmu,\,\Slmu,\,\mu^{(l)}\right),
            \end{align}
            which determines the next-level integral RKHS $\Hlpmu = \mH_{k^{(l)}}$.
    
      \item Choice of $\Honemu$: The first layer is the RKHS of the linear kernel $k^{(0)}(\b x,\b x')=\b x^\top\b x'$. This choice supplies the base case in the verification of \eqref{eq:integral-RKHS-condition} (\Cref{lem:BK}) and fixes the first-layer norm unambiguously. 
      
 \item Canonical form of \BK{} in case of \Cref{assumption:normalizing}: Supposing \Cref{assumption:normalizing}, one can show (Lemma~\ref{lem:measure_constraint_sphere})   that there exists a probability measure on the unit sphere of $\Hlmu$ with first moment equal to one
      $\hat\mu^{(l)}\in \Psph\left(\SHlmu\right)$ for $l \in [L-1]$
      such that for the associated integral kernel 
      \begin{align}
      \hat k^{(l)} = k\left[\SHlmu,\,\hat\mu^{(l)}\right], \label{eq:hat-k-l}
      \end{align}
      one has $\hat k^{(l)} = k^{(l)}$ for $l \in [L-1]$. Specifically this means that 
      \begin{align}
      \LHmu = \LHmuh \eqcolon \left(\Hlhmul\right)_{l=1}^L, \label{eq:H_{hat-mu}^{(l)}}
      \end{align}
      where $\Hlhmu=\mH_{\hat{k}^{(l)}}$ for $l \in [L-1]$ and $\mH_{\hat{\mu}}^{(1)}\coloneq \mH_{\mu}^{(1)}$. In other words, under \Cref{assumption:normalizing} one can assume without loss of generality that $\Slmu=\SHlmu$ for $l\in [L-1]$, while  preserving the associated kernels and integral RKHSs in \BK. 
      \label{rem:BKL_def_canonical_form}
      
      \item \emph{Connections to BKERNNs and NHLs.}
      The Brownian kernel-enhanced random neural network (BKERNN) construction of Follain and Bach~\cite{follain2024enhanced} is closely related to BKLs. The same is true of the neural Hilbert ladder (NHL) framework of Chen~\cite{chen2024neural}. In the specific case $L=2$, suppose that $\operatorname{span}(\mX)=\R^d$. Then the feature map
\begin{align}
\bomega
\longmapsto
\left(\b x\mapsto\bomega^\top\b x\right)
\end{align}
identifies $\Honemu$ isometrically with $\R^d$. Choosing the first-layer support to be the corresponding unit sphere, the integral kernel in the \BK{} construction [see \eqref{eq:BK_def}] takes the form
\begin{align}
k\left[\S^{d-1},\hat\mu\right]\left(\b x,\b x'\right)
&=
\int_{\S^{d-1}}
\kB\left(
\bomega^\top\b x,
\bomega^\top\b x'
\right)
\,\d\hat\mu\left(\bomega\right),
\end{align}
where $\hat\mu\in\Psph\left(\S^{d-1}\right)$. This is the same kernel as the one used in BKERNNs. $L$-level ($L\ge 2$) NHLs and BKLs share a similar hierarchical construction, with the former applying the integral kernel
\begin{align}
        k\left[\mH,\mu\right](\b x,\b x') &= \int_{\mH}\sigma(u(\b x))\sigma(u(\b x'))\d \mu(u), \label{eq:NHL}
\end{align}
instead of \eqref{eq:BK_def}, with $\sigma$ representing a Lipschitz activation function.
The function-space consequences of these connections depend on how the ladder measures are treated statistically. A ladder may be fixed in advance, selected from a controlled family, or optimized over the full adaptive canonical envelope. We distinguish these regimes explicitly in \Cref{sec:adaptive-controlled-BKL-models}; statistical comparisons are made only between classes with matching ladder-selection rules.\label{item:BKERNN-NHL-connections}
     
\end{enumerate}
\end{remark} 
Using the hierarchical structure of $\LHmu$, we now distinguish the full adaptive canonical envelope from controlled families of realized ladders. This separation retains the original measure-adaptive construction for the structural theory while making the statistical cost of ladder selection explicit.

\subsection{Adaptive Canonical \BK{} Envelope and Complexity} \label{sec:complexity-meaures}
We first define the full adaptive class $\HL$ and its infimal complexity $\CL$. Realized ladders and controlled finite dictionaries, which will serve as the statistical model classes, are introduced in \Cref{sec:adaptive-controlled-BKL-models}. The definitions below follow the order in \Cref{caption:defs}, while noting that the first four lines of the table were introduced in \Cref{sec:\BK{}}.

Let $\mX$, $L$, and $d$ be fixed, and let \Cref{assumption:normalizing} hold. A sequence
$\lmuh=\left(\hat\mu^{(l)}\right)_{l=1}^{L-1}$ is called \emph{recursively admissible} if, starting from
$\Honemuh\coloneqq\Honemu$, one has for every $l\in[L-1]$,
\begin{align}
\hat\mu^{(l)}
&\in
\mP_{1,1}\left(
\S_1\left(
\Hlmuh
\right)
\right),\\
\hat k^{(l)}
&\coloneqq
k\left[
\S_1\left(
\Hlmuh
\right),
\hat\mu^{(l)}
\right],\\
\Hlhmu
&\coloneqq
\mH_{\hat k^{(l)}}.
\end{align}
Varying over all recursively admissible sequences gives the family
\begin{align}
\LHh
\coloneqq
\left\{
\LHmuh:
\lmuh\text{ is recursively admissible}
\right\}.
\label{def:unitsphere\BK{}}
\end{align}
We call \(\LHh\) the family of $L$-level canonical BKLs.

\medskip
\noindent
\emph{Complexity of the top-layer elements of $\LHh$.}\ 
Assume now that $f$ lies in the top layer of some $\LHmuh$, that is, $f\in \HLmuh$ for some $\LHmuh \in \LHh$. We define the $L$-complexity of $f$ as the infimum over all such realizations:
\begin{align}
    \CL(f) \coloneqq \inf_{\LHmuh\in \LHh\,:\, f\in \HLmuh}\|f\|_{\HLmuh} \in [0,\infty).\label{\BK{}:unitsphereladder-complexity}
\end{align}
%where the finiteness of $\CL(f)$ follows from the fact that $f\in \HLmuh$.

\medskip
\noindent
\emph{Full adaptive top-layer envelope $\HL$.}\ 
Finally, we define the space $\HL$ as the collection of $L$-th layers of canonical \BK{}s $\LHmuh$ as follows
\begin{align}
\HL \coloneqq \left\{ f \in \HLmuh : \LHmuh \in \LHh,\,\,\CL(f) < \infty \right\} = \bigcup_{\substack{\LHmuh \in \LHh}} \HLmuh. \label{def:\BK{}-unitsphereladder-space}
\end{align}
We refer to $\HL$ as the \emph{full adaptive canonical \BK{} envelope}: the complete sequence of canonical ladder measures is allowed to vary in both \eqref{def:unitsphere\BK{}} and the infimum in \eqref{\BK{}:unitsphereladder-complexity}. The analytical and variational results below concern this full envelope. For the statistical results, we use the following realized and controlled subfamilies.

\subsection{Realized and Controlled Canonical \BK{} Models}
\label{sec:adaptive-controlled-BKL-models}

\begin{definition}[realized and controlled canonical \BK{} models]
\label{def:controlled-BKL-family}
Let $M\in\Zp$. For each $m\in[M]$, let
$\hat\mu_m^{(1:L-1)}=\left(\hat\mu_m^{(l)}\right)_{l=1}^{L-1}$ generate a canonical ladder
$\mH_{\hat\mu_m}^{(1:L)}\in\LHh$. We call
\begin{align}
\BKLdict
\coloneqq
\left\{
\mH_{\hat\mu_m}^{(1:L)}:m\in[M]
\right\}
\subseteq
\LHh
\label{eq:controlled-BKL-dictionary}
\end{align}
a controlled dictionary of canonical \BK{}s and define its top-layer class by
\begin{align}
\HLdict
\coloneqq
\bigcup_{m=1}^{M}
\mH_{\hat\mu_m}^{(L)}.
\label{eq:controlled-BKL-top-layer}
\end{align}
For $f\in\HLdict$, define the controlled $L$-complexity
\begin{align}
\CLdict(f)
\coloneqq
\min_{\substack{m\in[M]\\ f\in\mH_{\hat\mu_m}^{(L)}}}
\left\|f\right\|_{\mH_{\hat\mu_m}^{(L)}}.
\label{eq:controlled-BKL-complexity}
\end{align}
For $r>0$, the associated controlled radius-$r$ class is
\begin{align}
\BLdictr
&\coloneqq
\left\{
 f\in\HLdict:
 \CLdict(f)\le r
\right\}
\nonumber\\
&=
\bigcup_{m=1}^{M}
\B_r\left(\mH_{\hat\mu_m}^{(L)}\right).
\label{eq:controlled-BKL-ball}
\end{align}
When $M=1$, these definitions reduce to one fixed, realized canonical \BK{} and its top-layer RKHS ball.
\end{definition}

Since the dictionary is finite, the minimum in \eqref{eq:controlled-BKL-complexity} is attained. Unlike $\HL$, the class $\HLdict$ need not be a vector space, and $\CLdict$ is not asserted to be a quasi-norm; it is a model-selection complexity on a controlled statistical class.

\begin{prop}[relation between the adaptive and controlled models]
\label{prop:controlled-BKL-bridge}
\begin{align}
\mH_{\hat\mu_m}^{(L)}
\subseteq
\HLdict
\subseteq
\HL,
\qquad m\in[M].
\label{eq:controlled-BKL-inclusions}
\end{align}
Moreover,
\begin{align}
\CL(f)
&\le
\CLdict(f),
&& f\in\HLdict,
\label{eq:controlled-BKL-complexity-lower}\\
\CLdict(f)
&\le
\left\|f\right\|_{\mH_{\hat\mu_m}^{(L)}},
&& f\in\mH_{\hat\mu_m}^{(L)},\quad m\in[M].
\label{eq:controlled-BKL-complexity-upper}
\end{align}
Consequently, for every $r>0$ and $m\in[M]$,
\begin{align}
\B_r\left(\mH_{\hat\mu_m}^{(L)}\right)
\subseteq
\BLdictr
\subseteq
\left\{
 f\in\HL:
 \CL(f)\le r
\right\}.
\label{eq:controlled-BKL-ball-inclusions}
\end{align}
\end{prop}

\begin{proof}
The first inclusion in \eqref{eq:controlled-BKL-inclusions} follows from the union in \eqref{eq:controlled-BKL-top-layer}. The second follows from $\BKLdict\subseteq\LHh$ and the definition of $\HL$ in \eqref{def:\BK{}-unitsphereladder-space}. Since $\CL$ takes the infimum over all canonical ladders whereas $\CLdict$ minimizes over the finite subfamily $\BKLdict$, \eqref{eq:controlled-BKL-complexity-lower} follows. Evaluating that minimum at any admissible index $m$ gives \eqref{eq:controlled-BKL-complexity-upper}. The ball inclusions are immediate consequences.
\end{proof}

\begin{remark}[order of ladder selection and Gaussian averaging]
\label{rem:ladder-selection-order}
In the statistical results, $\BKLdict$ is fixed independently of the estimation sample and of the auxiliary Gaussian or Rademacher variables. It may be deterministic, or it may be constructed from an independent pilot sample, in which case the statistical statements are applied conditionally on that pilot sample. For a fixed sample $\left(\b x_i\right)_{i=1}^{n}$, define
\begin{align}
Z_{\hat\mu}(\epsilon)
\coloneqq
\sup_{f\in\B_r\left(\mH_{\hat\mu}^{(L)}\right)}
\frac{1}{n}
\sum_{i=1}^{n}
\epsilon_i f(\b x_i).
\end{align}
For a fixed ladder, the relevant quantity is $\E_{\epsilon}Z_{\hat\mu}(\epsilon)$. For a controlled dictionary, it is $\E_{\epsilon}\left[\max_{m\in[M]}Z_{\hat\mu_m}(\epsilon)\right]$, which incurs a finite model-selection cost. By contrast, unrestricted adaptive selection places the supremum over all canonical ladders inside the Gaussian expectation. In particular,
\begin{align}
\sup_{\LHmuh\in\LHh}
\E_{\epsilon}Z_{\hat\mu}(\epsilon)
\qquad\text{and}\qquad
\E_{\epsilon}
\left[
\sup_{\LHmuh\in\LHh}
Z_{\hat\mu}(\epsilon)
\right]
\end{align}
are different quantities in general. This quantifier distinction allows the analytical theory to remain on the full envelope $\HL$, while the near-parametric statistical guarantees are formulated for realized or controlled families of ladders.
\end{remark}

We continue with analytical and variational results for the full adaptive envelope $\HL$, followed by statistical results for its controlled subfamilies.

\section{Results: Analytical, Variational, and Statistical Properties of \BK{} Models} \label{sec:results}
We first study the analytical properties of the full adaptive envelope $\HL$ in \Cref{sec:characterization}. We then establish variational well-posedness for regularization over that envelope and develop statistical guarantees for realized and controlled canonical \BK{} models in \Cref{sec:stat-analyses}.

\subsection{Analytical Properties of the Full Adaptive BKL Envelope} \label{sec:characterization}
In this subsection, we establish different analytical properties of $\HL$. 

\begin{theorem}[adaptive-envelope analytical properties]\label{thm:charachteriz_funspace}
Let Assumption \ref{assumption:normalizing} hold and let $L \ge 2$. Then the following statements hold.
\begin{enumerate}[(i)]
 \item Hölder property, pointwise evaluation: Each element $f \in \HL$ is $\frac{1}{2^{L-1}}$-H\"{o}lder continuous with constant $\CL(f)$: 
 \begin{align}\label{eq:HL-holder}
 \left|f\left(\b x\right) - f\left(\b x'\right)\right| &\leq \sqrt[\leftroot{-2} \uproot{0} 2^{L-1}]{\|\b x - \b x'\|_2}\CL(f) \t{ for all }\b x,\b x' \in \mX.
 \end{align}

 Moreover, a similar property holds for pointwise evaluation
  \begin{align}\label{eq:HL-pointwise}
|f(\b x)|
\le
\sqrt[\leftroot{-2}\uproot{0}2^{L-1}]{\|\b x\|_2}\CL(f),
\quad \forall \b x\in\mX .
\end{align}\label{thm:HLp-analytical}
 \item Quasi-Banach structure: The pair $\left(\HL,\CL\right)$ forms a quasi-Banach space. \label{thm:charachteriz_funspace-3}
 
 \item (Strict) monotonicity:
Let $\HL$ and $\HLplus$ be defined according to
\eqref{def:\BK{}-unitsphereladder-space}.
Then, for every $L\ge 2$, one has $\HL \subseteq \HLplus$.
Assume that there exist $\b x_0 \in \mX$, $\rho>0$,
constants $0<c_1\le C_1<\infty$, a Lipschitz map
$\gamma:[0,\rho]\to\mX$ with $\gamma(0)=\b x_0$, and a function
$u_1 \in \Honemu$ such that
\begin{align}
u_1(\b x_0) = 0,
\qquad
c_1 t \le u_1(\gamma(t)) \le C_1 t,
\qquad t\in[0,\rho].
\end{align}
Then there exists a function $f_\star \in \HLplus \setminus \HL$
such that for every $r>0$,
\begin{align}\label{eq:BLr-nonapprox}
\inf_{g \in \BL_r}
\left\| f_\star - g \right\|_\infty
> 0,
\qquad
\BL_r \coloneqq \B_r\left(\left(\HL,\CL\right)\right).
\end{align}\label{thm:charachteriz_funspace-5}
Consequently, whenever the stated trace condition holds,
$\HL \subsetneq \HLplus$.
\label{thm:charachteriz_funspace-4}
\end{enumerate}
\end{theorem}

\begin{remark}~\label{rem:analytical_properties_thm}
    \begin{enumerate}    
      \item Metric on $\HL$ [\Cref{thm:charachteriz_funspace}\ref{thm:charachteriz_funspace-3}]: The space $\left(\HL,\CL\right)$ is a quasi-norm space, hence in particular it is a linear space
        equipped with a quasi-norm $\CL$ satisfying
        \begin{align} \label{eq:quasiineq}
        \CL(f+g) \le K\left[\CL(f)+\CL(g)\right], \qquad f,g\in\HL, 
        \end{align} 
        with $K=K(L)=  \frac{2}{\sqrt[\leftroot{-2} \uproot{0} 2^{L}]{2}} \in \left[2^{\frac{3}{4}},2\right)$ where the value of $K$ is shown in the proof of \Cref{thm:charachteriz_funspace}\ref{thm:charachteriz_funspace-3}.\footnote{The function $L \in \{2,3,4,\ldots\} \mapsto K(L) \in \R$ is non-decreasing and converging to 2 as $L \to \infty$.} Therefore, by the Aoki-Rolewicz theorem \cite[Theorem~$1.3$]{ionescu2008metrization}
        there exists $\alpha=\alpha(K)\in(0,1]$  and a quasi-norm $\left(\CL\right)^*$ on $\HL$ equivalent to $\CL$ for which the metric 
        \begin{align} 
        \dL(f,g) \coloneqq \left[\left(\CL\right)^*(f-g)\right]^{\alpha}, \qquad f,g\in\HL
        \end{align}
        induces the same topology as the quasi-norm $\CL$. This metric allows one to work with (classical) Cauchy sequences.
      \item Expressivity of $\HL$ [\Cref{thm:charachteriz_funspace}\ref{thm:charachteriz_funspace-4}]: The expressivity of BKL increases as a function of the depth $L$. \eqref{eq:BLr-nonapprox} shows that one can find some element $f_\star$ in $\HLplus$ but not in $\HL$, whose distance from arbitrary ball of $\HL$ is positive.
      \item Assumptions of strict increase [\Cref{thm:charachteriz_funspace}\ref{thm:charachteriz_funspace-4}]: A simple sufficient condition is the presence of a \emph{transverse line segment}. More precisely, by \Cref{lem:line-segment-implies-trace}, the assumptions hold whenever there exist $\b x_0,\b v\in\R^d$ and $\rho>0$ such that
      \begin{align}
      \left\{\b x_0+t\b v:t\in[0,\rho]\right\}
      \subseteq
      \m X,
      \qquad
      \b v\notin\operatorname{span}\left\{\b x_0\right\}.
      \end{align}
      In particular, every non-degenerate line segment in $\m X$ having the origin as one endpoint satisfies this condition.
    \end{enumerate}
\end{remark} 

\subsection{Variational and Statistical Properties of \BK{} Models}
\label{sec:stat-analyses}
In this section we first study regularized empirical risk minimization over the full adaptive class $\HL$, and then turn to Gaussian-complexity and excess-risk guarantees for the controlled classes introduced in \Cref{sec:adaptive-controlled-BKL-models}. Let $\m{D}_{XY,n}=\{(x_i,y_i)\}_{i=1}^n\sim\nu_{XY}^n$, where $\nu_{XY}\in\mP(\mX\times\mY)$ is the joint distribution and $\nu_X$ is its input marginal.\footnote{The notation $\nu_X$ will also be used in the Gaussian-complexity analysis of \Cref{sec:gauss-complex}.} For a predictor $f:\mX\to\R$ and a Borel-measurable loss $\ell:\R\times\mY\to\R$, we use the convention that the prediction is the first argument of the loss. The expected risk is
\begin{align}
\mR(f)
&\coloneqq
\int_{\mX\times\mY}
\ell\left(f(x),y\right)
\,\d\nu_{XY}(x,y),
\label{eq:expected-risk}
\end{align}
and the empirical risk is
\begin{align}
\mR_n(f)
&\coloneqq
\int_{\mX\times\mY}
\ell\left(f(x),y\right)
\,\d\nu_{XY,n}(x,y)
=
\frac1n\sum_{i\in[n]}
\ell\left(f(x_i),y_i\right),
\label{eq:empirical-risk}
\end{align}
where $\nu_{XY,n}=n^{-1}\sum_{i=1}^n\delta_{(x_i,y_i)}$. More generally, regularized empirical risk minimization takes the form
\begin{align}
\hat f_\lambda
&\in
\argmin_{f\in\F}
\left\{
\mR_n(f)+\lambda\C(f)
\right\},
\qquad
\lambda>0.
\label{eq:f-hat-lambda}
\end{align}
Here $\F$ is the hypothesis class, $\C$ is its regularizer, and $\lambda$ controls the tradeoff between empirical fit and regularity. The performance difference from the optimal population risk is quantified by the excess risk
\begin{align}
\mE(f) &= \mR(f)-\inf_{g \in \F} \mR(g); \label{eq:def:excess-risk}
\end{align} 
the excess risk is non-negative by definition, and it is equal to zero for any minimizer ($f^* \in \argmin_{f \in \F} \mR(f)$; if it exists). For the adaptive variational result below, we choose a compact input space $\mX\subset \R^d$, the hypothesis class $\F = \HL$, and the regularizer $\C=\CL$, and hence work with
\begin{align}
\hat{f}_\lambda &\in \argmin_{f \in \HL} J_n(f), & J_n(f) & := \m{R}_n(f) + \lambda \CL(f), \quad\lambda > 0,\label{eq:fhatlambda}\\
f^* &\in \argmin_{f \in \HL} \mR(f). \label{eq:f*-def}
\end{align}
We first establish existence of the adaptive RERM solution in \eqref{eq:fhatlambda} and identify the precise uniqueness statement for population minimizers. The controlled-family statistical guarantees are developed subsequently.

\subsubsection{Existence of Solutions} \label{sec:RERM-solution}
The next result establishes variational well-posedness of the fully adaptive RERM problem and gives the correct almost-everywhere uniqueness statement for population minimizers.

\begin{theorem}[adaptive RERM well-posedness]
\label{thm:existence_\BK{}}
Let Assumption~\ref{assumption:normalizing} hold, let $\mX\subset\R^d$ be compact, and let $\mY\subset\R$.
\begin{enumerate}[(i)]
\item\label{thm:existence_\BK{}-1} \emph{Existence of $\hat f_\lambda$.}
Let $\mD_{XY,n}=\{(\b x_i,y_i)\}_{i=1}^n\subset\mX\times\mY$ be fixed, let $\lambda>0$, and let
\begin{align}
\ell:\R\times\mY\longrightarrow\R
\end{align}
be Borel measurable and such that $\ell(\cdot,y)$ is continuous for every $y\in\mY$. Assume that the loss is bounded from below uniformly in both arguments: there exists $b_\ell\in\R$ such that
\begin{align}
\ell(u,y)
\ge
b_\ell,
\qquad
u\in\R,
\quad
 y\in\mY.
\label{eq:loss-global-lower-bound}
\end{align}
Then the objective $J_n$ in \eqref{eq:fhatlambda} attains its minimum on $\HL$; equivalently,
\begin{align}
\argmin_{f\in\HL}J_n(f)
\neq
\varnothing.
\end{align}
\item\label{thm:existence_\BK{}-2} \emph{Uniqueness of population minimizers.}
Assume that $\ell$ is Borel measurable, satisfies \eqref{eq:loss-global-lower-bound}, and that $\ell(\cdot,y)$ is continuous and strictly convex for every $y\in\mY$. Assume also that $\mR(0)<\infty$. Then any two minimizers $f^*,g^*\in\argmin_{f\in\HL}\mR(f)$ satisfy
\begin{align}
f^*
=
 g^*
\qquad
\nu_X\text{-almost everywhere}.
\label{eq:population-minimizer-ae-uniqueness}
\end{align}
If, moreover,
\begin{align}
\operatorname{supp}(\nu_X)
=
\mX,
\end{align}
then $f^*=g^*$ pointwise on $\mX$; hence the population minimizer, whenever it exists, is unique as an element of $\mC(\mX)$.
\end{enumerate}
\end{theorem}

\begin{remark}
\label{rem:nonuniqueness}
The lower-bound assumption \eqref{eq:loss-global-lower-bound} is satisfied with $b_\ell=0$ by the squared, absolute, Huber, and $\varepsilon$-insensitive losses considered later. Since $\CL$ is only a quasi-norm, the empirical minimizer in \eqref{eq:fhatlambda} need not be unique. By contrast, strict convexity yields uniqueness of population predictions up to $\nu_X$-almost-everywhere equality, and full support of $\nu_X$ upgrades this conclusion to pointwise uniqueness because every element of $\HL$ is continuous.
\end{remark}

\subsubsection{Gaussian Complexity}\label{sec:gauss-complex}
We now bound the Gaussian complexity of the controlled radius-$r$ class $\BLdictr$ from \eqref{eq:controlled-BKL-ball}. The dictionary $\BKLdict$ is fixed independently of the estimation sample and of the auxiliary Gaussian variables, as specified in \Cref{rem:ladder-selection-order}.

\begin{definition}[(empirical) Gaussian complexity]
\label{def:gaussian-complexity}
Let $\F$ be a real-valued function class over $\mX$ and $\bm{\varepsilon}=(\varepsilon_i)_{i=1}^n\sim\mN^n(0,1)$. Then, for a fixed dataset $\m D_n=\{\b x_i\}_{i=1}^{n}\subset\mX$, the empirical Gaussian complexity of $\F$ is defined by
\begin{align}
\eG(\F)
\coloneqq
\eE\left[
\sup_{f\in\F}
\frac{1}{n}
\sum_{i=1}^{n}
\varepsilon_i f(\b x_i)
\right].
\label{eq:empiricalGn-bound}
\end{align}
Its expectation with respect to $(X_i)_{i=1}^{n}\sim\nu_X^n$,
\begin{align}
\G(\F)
\coloneqq
\DE\left[\eG(\F)\right]
=
\eDE\left[
\sup_{f\in\F}
\frac{1}{n}
\sum_{i=1}^{n}
\varepsilon_i f(X_i)
\right],
\label{eq:Gn-bound}
\end{align}
is called the Gaussian complexity of $\F$.
\end{definition}
For $M\in\Zp$, set
\begin{align}
\kappa_M
\coloneqq
1+\sqrt{2\ln M}.
\label{eq:controlled-dictionary-factor}
\end{align}

\begin{theorem}[controlled-family Gaussian complexity]
\label{thm:peeling_all_depth_\BK{}}
Let Assumption~\ref{assumption:normalizing} hold. Let
\begin{align}
\BKLdict
=
\left\{
\mH_{\hat\mu_m}^{(1:L)}:m\in[M]
\right\}
\end{align}
be a controlled dictionary from \Cref{def:controlled-BKL-family}, fixed independently of the estimation sample, where $M\in\Zp$. Let $r>0$ and $n\in\Zp$. Then, for every fixed sample $\{\b x_i\}_{i=1}^{n}\subset\mX$,
\begin{align}
\eG\left(\BLdictr\right)
\le
\frac{r\kappa_M}{\sqrt n}
\left(
\max_{i\in[n]}
\left\|\b x_i\right\|_2^2
\right)^{2^{-L}}.
\label{eq:controlled-empirical-gaussian-bound}
\end{align}
Consequently,
\begin{align}
\G\left(\BLdictr\right)
\le
\frac{r\kappa_M}{\sqrt n}
\left(
\DE\left[
\max_{i\in[n]}
\left\|X_i\right\|_2^2
\right]
\right)^{2^{-L}}.
\label{eq:gaussian_contraction_\BK{}}
\end{align}
In particular, the bound contains no explicit factor depending on the ambient dimension $d$.
\end{theorem}

\begin{corollary}[one fixed realized canonical \BK{}]
\label{cor:fixed-ladder-gaussian-complexity}
Let $\mH_{\hat\mu}^{(1:L)}\in\LHh$ be fixed independently of the sample. Let $n\in\Zp$, $r>0$, and fix $\{\b x_i\}_{i=1}^{n}\subset\mX$. Then
\begin{align}
\eG\left(
\B_r\left(\mH_{\hat\mu}^{(L)}\right)
\right)
&\le
\frac{r}{\sqrt n}
\left(
\max_{i\in[n]}
\left\|\b x_i\right\|_2^2
\right)^{2^{-L}},
\label{eq:fixed-ladder-empirical-gaussian-bound}\\
\G\left(
\B_r\left(\mH_{\hat\mu}^{(L)}\right)
\right)
&\le
\frac{r}{\sqrt n}
\left(
\DE\left[
\max_{i\in[n]}
\left\|X_i\right\|_2^2
\right]
\right)^{2^{-L}}.
\label{eq:fixed-ladder-expected-gaussian-bound}
\end{align}
\end{corollary}

\begin{proof}
Apply \Cref{thm:peeling_all_depth_\BK{}} with $M=1$, for which $\kappa_1=1$ and $\BLdictr=\B_r\left(\mH_{\hat\mu}^{(L)}\right)$.
\end{proof}

\begin{remark}~\label{rem:Gaussian_complexity_thm}
\begin{enumerate}
\item \emph{Cost of ladder selection.}
For one fixed ladder, \Cref{cor:fixed-ladder-gaussian-complexity} has the standard $n^{-1/2}$ upper-bound dependence and no logarithmic model-selection term. For a dictionary of $M$ ladders, adaptation costs only the factor $\kappa_M=1+\sqrt{2\ln M}$. In particular, if $M\le n^q$ for a constant $q>0$, then
\begin{align}
\G\left(\BLdictr\right)
=
\widetilde{\m O}\left(n^{-1/2}\right).
\end{align}

\item \emph{Dependence on depth and input radius.}
Let
\begin{align}
B
\coloneqq
\sup_{\b x\in\mX}
\left\|\b x\right\|_2^2
<\infty.
\label{eq:B-def}
\end{align}
Since $L\ge2$,
\begin{align}
B^{2^{-L}}
\le
\max\left\{1,B^{1/4}\right\}.
\end{align}
Therefore,
\begin{align}
\G\left(\BLdictr\right)
\le
\frac{r\kappa_M}{\sqrt n}
\max\left\{1,B^{1/4}\right\}.
\label{eq:Gaussian-complexity-uniform}
\end{align}
If $\mX$ is normalized so that $B\le1$, the radius factor in \eqref{eq:Gaussian-complexity-uniform} is at most one. Uniformity over a sequence of depths additionally requires the corresponding dictionary sizes $M=M_L$ to be controlled.

\item \emph{Scope of the theorem.}
The bounds above concern one realized ladder or a finite controlled dictionary fixed before the estimation sample. They do not assert an $n^{-1/2}$ bound for the unrestricted adaptive ball
\begin{align}
\left\{
 f\in\HL:\CL(f)\le r
\right\},
\end{align}
for which the supremum over all ladder measures lies inside the Gaussian expectation. This is precisely the distinction recorded in \Cref{rem:ladder-selection-order}.
\end{enumerate}
\end{remark}

\subsubsection{Controlled-Family Oracle and Excess-Risk Bounds}
\label{sec:excess-risk-guarantees}
The Gaussian-complexity estimate in \Cref{thm:peeling_all_depth_\BK{}} is a radius-dependent statement for the controlled class $\BLdictr$. Accordingly, we formulate the statistical estimator on the same fixed radius class. For $r>0$ and $\lambda\ge0$, define the radius-constrained controlled RERM estimator by
\begin{align}
\hat f_{\lambda,r,\BKLdict}
\in
\argmin_{f\in\BLdictr}
\left\{
\mR_n(f)
+
\lambda\CLdict(f)
\right\}.
\label{eq:controlled-RERM}
\end{align}
The case $\lambda=0$ is empirical risk minimization over the controlled radius-$r$ class, while $\lambda>0$ additionally favors smaller realized-ladder norms inside that class. We measure performance relative to the best predictor in the same statistical model by
\begin{align}
\mE_{\BKLdict,r}(f)
\coloneqq
\mR(f)
-
\inf_{g\in\BLdictr}
\mR(g).
\label{eq:controlled-excess-risk}
\end{align}
No population-risk minimizer is assumed to exist.

\begin{theorem}[controlled-family excess risk]
\label{thm:excess-risk-\BK{}}
Let Assumption~\ref{assumption:normalizing} hold. Let
\begin{align}
\BKLdict
=
\left\{
\mH_{\hat\mu_m}^{(1:L)}:m\in[M]
\right\}
\end{align}
be a controlled dictionary from \Cref{def:controlled-BKL-family}, fixed independently of the estimation sample, where $M\in\Zp$. Let $n\in\Zp$, $r>0$, and $\lambda\ge0$. Let $B$ be defined in \eqref{eq:B-def}, let $A\ge0$ satisfy
\begin{align}
\mY
\subseteq
[-A,A],
\qquad
rB^{2^{-L}}
\le
A,
\label{eq:controlled-loss-range}
\end{align}
and let $\ell:\R\times\mY\to\Rnn$ satisfy, for some constants $M_\ell,L_\ell<\infty$,
\begin{align}
0
\le
\ell(u,y)
&\le
M_\ell,
\label{eq:controlled-loss-boundedness}\\
\left|
\ell(u_1,y)-\ell(u_2,y)
\right|
&\le
L_\ell
\left|u_1-u_2\right|,
\label{eq:controlled-loss-lipschitz}
\end{align}
for all $y\in\mY$ and all $u,u_1,u_2\in[-A,A]$. Then the solution set in \eqref{eq:controlled-RERM} is nonempty.

Moreover, there exists a universal constant $C>0$ such that, for every $\delta\in(0,1)$, with probability at least $1-\delta$ over the estimation sample,
\begin{align}
\mR\left(\hat f_{\lambda,r,\BKLdict}\right)
-
\inf_{f\in\BLdictr}
\left\{
\mR(f)+\lambda\CLdict(f)
\right\}
\le
r\Lambda_{n,M,L,\delta},
\label{eq:controlled-penalized-oracle}
\end{align}
where
\begin{align}
\Lambda_{n,M,L,\delta}
\coloneqq
C L_\ell B^{2^{-L}}
\left[
\frac{\kappa_M}{\sqrt n}
+
\sqrt{
\frac{\ln(1/\delta)}{n}
}
\right].
\label{eq:controlled-statistical-scale}
\end{align}
Consequently,
\begin{align}
\mE_{\BKLdict,r}
\left(\hat f_{\lambda,r,\BKLdict}\right)
\le
r
\left(
\Lambda_{n,M,L,\delta}
+
\lambda
\right).
\label{eq:exessriskbound}
\end{align}
In particular, if $0\le\lambda\le\Lambda_{n,M,L,\delta}$, then
\begin{align}
\mE_{\BKLdict,r}
\left(\hat f_{\lambda,r,\BKLdict}\right)
\le
2r\Lambda_{n,M,L,\delta},
\label{eq:controlled-near-parametric-excess}
\end{align}
while the choice $\lambda=0$ improves the factor $2$ in \eqref{eq:controlled-near-parametric-excess} to $1$.
\end{theorem}

\begin{remark}~\label{rem:excess-risk-\BK{}}
\begin{enumerate}
\item \emph{Near-parametric rate and model-selection cost.}
For a fixed depth and radius, \eqref{eq:controlled-near-parametric-excess} contains no explicit ambient-dimension factor and has the rate
\begin{align}
\m O\left(
 rL_\ell B^{2^{-L}}
 \sqrt{\frac{1+\ln M+\ln(1/\delta)}{n}}
\right).
\end{align}
Thus a polynomial-size dictionary retains the near-parametric $\widetilde{\m O}(n^{-1/2})$ rate. The dependence on the number of candidate ladders is logarithmic through $\kappa_M=1+\sqrt{2\ln M}$.

\item \emph{One realized ladder.}
When $M=1$, one has $\kappa_1=1$, and \Cref{thm:excess-risk-\BK{}} gives the corresponding fixed-ladder oracle and excess-risk guarantees without a model-selection factor.

\item \emph{Why the radius is explicit.}
Theorem~\ref{thm:existence_\BK{}} establishes variational well-posedness for the unrestricted adaptive RERM problem on $\HL$. The present theorem addresses a different question: statistical estimation over a realized or controlled model whose ladder-selection cost is explicit. The radius constraint matches the class controlled by \Cref{thm:peeling_all_depth_\BK{}} and prevents an uncharged supremum over arbitrary ladder measures and arbitrary RKHS norms.

\item \emph{Independent pilot construction.}
If $\BKLdict$ is constructed from an independent pilot sample, the theorem applies conditionally on that pilot sample and hence also unconditionally after averaging over the pilot stage.
\end{enumerate}
\end{remark}

\begin{table}
\centering
\scriptsize
\caption{Valid constants $M_\ell$ and $L_\ell$ for common losses when predictions and outputs lie in $[-A,A]$; see \Cref{lem:loss-bounds} with $M\leftarrow A$.}
\label{tab:loss-examples-final}
\renewcommand{\arraystretch}{1.1}
\begin{tabular}{l l c c}
\hline
Loss & $\ell(u,y)$ & \textbf{$M_\ell$} & \textbf{$L_\ell$} \\
\hline
squared & $(u-y)^2$ & $4A^2$ & $4A$ \\

absolute & $|u-y|$ & $2A$ & $1$ \\

Huber ($\delta>0$) & $\begin{cases}0.5(u-y)^2,&\t{ if }|u-y|\le\delta,\\ \delta|u-y|-0.5\delta^2,&\t{ if }|u-y|>\delta\end{cases}$ & $\begin{cases}2A^2,&\t{ if }\delta\ge 2A,\\ 2\delta A-0.5\delta^2,&\t{ if }\delta<2A\end{cases}$ & $\delta$ \\

$\epsilon$–insensitive ($\epsilon>0$) & $\max\{0,|u-y|-\epsilon\}$ & $\max\{0,2A-\epsilon\}$ & $1$ \\
\hline
\end{tabular}
\end{table}

% ===========================
\begin{figure}[H]
\centering
\scriptsize

\begin{tikzpicture}
\def\delta{1}
\begin{axis}[
  width=0.43\linewidth,     % smaller → allows two figs in one line
  height=0.32\linewidth,    % proportional scaling
  xlabel={$u$}, ylabel={$\ell(u,0)$},
  xmin=-2, xmax=2,
  ymin=0, ymax=4,
  restrict y to domain=0:4, 
  grid=major,
  grid style={gray!15},
  legend style={
    at={(0.5,1.05)},  
    anchor=south,
    legend columns=-1,
    column sep=6pt,
    font=\scriptsize,
    draw=none,
    fill=none
  },
  ticklabel style={font=\tiny},
]

% -------------------------
% Squared loss
% -------------------------
\addplot[black, thick, domain=-2:2, samples=400]
  {(x)^2};

% -------------------------
% Huber loss
% -------------------------
\addplot[blue!70!black, thick, dotted, domain=-2:2, samples=800]
  {
    (abs(x) <= \delta) * (0.5*(x)^2)
  + (abs(x) > \delta) * (\delta*(abs(x) - 0.5*\delta))
  };

% -------------------------
% Transition markers ±δ
% ---------------
\addplot[gray!50, dashed, semithick] coordinates {(-1,0) (-1,4)};
\addplot[gray!50, dashed, semithick] coordinates {(1,0) (1,4)};

\legend{Squared, Huber}

\end{axis}
\end{tikzpicture}

\caption{
Illustration of the squared loss and the Huber loss ($\delta=1$) for $y=0$. 
The Huber loss is quadratic on $[-\delta,\delta]$ and linear outside of this region. 
At the transition points $\pm\delta$, the two pieces join smoothly: both the function 
value and the first derivative coincide, so the Huber loss is continuously differentiable 
at these knots. Dashed lines indicate the knot locations. 
See Fig.~\ref{fig:losses-abs-eps} for the corresponding absolute and 
$\epsilon$–insensitive losses.
}
\label{fig:losses-sq-huber}
\end{figure}
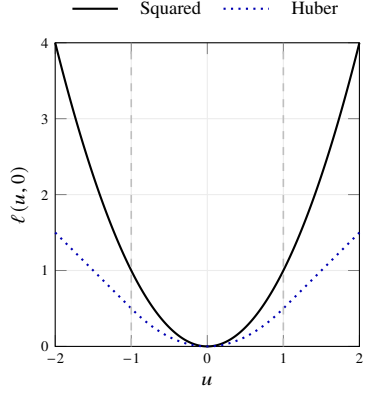
% ===========================
% Figure B: Absolute + ε–insensitive
% ===========================
% =====================
% Figure B: Absolute + ε–insensitive (rescaled to match Figure A)
% ===========
\begin{figure}[H]
\centering
\scriptsize

\begin{tikzpicture}
\def\eps{1.0}
\begin{axis}[
  width=0.43\linewidth,     
  height=0.32\linewidth,
  xlabel={$u$}, ylabel={$\ell(u,0)$},
  xmin=-2, xmax=2,
  ymin=0, ymax=2,
  restrict y to domain=0:2,
  grid=major,
  grid style={gray!15},
  legend style={
    at={(0.5,1.05)},
    anchor=south,
    legend columns=-1,
    column sep=6pt,
    font=\scriptsize,
    draw=none,
    fill=none
  },
  ticklabel style={font=\tiny},
]

% -------------------------
% Absolute loss
% -------------------------
\addplot[red, very thick, dashed, domain=-2:2, samples=300]
  {abs(x)};

% -------------------------
% ε–insensitive loss
% -------------------------
\addplot[green!40!black, thick, dashdotdotted, domain=-2:2, samples=300]
  {
    (abs(x) <= \eps) * 0
  + (abs(x) > \eps) * (abs(x) - \eps)
  };

% -------------------------
% Region markers ±ε
% -------------------------
\addplot[gray!60, dashed, very thin] coordinates {(-1,0) (-1,2)};
\addplot[gray!60, dashed, very thin] coordinates {(1,0) (1,2)};

% -------------------------
% Legend
% -------------------------
\legend{Absolute, $\varepsilon$–insensitive}

\end{axis}
\end{tikzpicture}
\caption{
Illustration of the absolute and the $\varepsilon$–insensitive loss ($\varepsilon=1$). 
The $\varepsilon$–insensitive loss is identically zero on $[-\varepsilon,\varepsilon]$, it is $u\mapsto u-\varepsilon$ on $[\varepsilon,\infty)$, and it equals $u\mapsto -u-\varepsilon$ on $(-\infty,-\varepsilon]$. See Fig.~\ref{fig:losses-sq-huber} for an illustration of the squared and the Huber loss.
}
\label{fig:losses-abs-eps}
\end{figure}
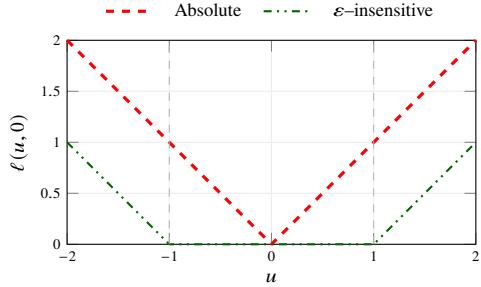

\section{Conclusion} \label{sec:conclu}
Brownian kernel ladders provide a recursive function-space construction in which depth is encoded by successive Brownian integral RKHSs rather than by a prescribed finite-dimensional architecture. Starting from linear functionals, the hierarchy repeatedly integrates Brownian pullback kernels over functions from the preceding layer. The nonnegative $1$-homogeneity of the Brownian kernel permits a kernel-preserving canonical spherical normalization and makes the regularity inherited across depth explicit.

The resulting theory has two complementary levels. At the representational level, varying all canonical ladder measures produces a full adaptive envelope equipped with an infimal complexity. We proved recursive H\"older and pointwise estimates, quasi-Banach structure, nestedness, and, under a geometric trace condition, strict growth with ballwise separation. We also established existence of regularized empirical-risk minimizers on this adaptive envelope and the precise uniqueness conclusions for population predictions under strict convexity.

At the estimation level, one realized ladder or a finite dictionary fixed independently of the estimation sample defines a controlled statistical model. Its Gaussian complexity has no explicit ambient-dimension factor and has the standard $n^{-1/2}$ upper-bound dependence, with the factor $1+\sqrt{2\ln M}$ measuring selection among $M$ candidate ladders. The associated radius-constrained estimators satisfy high-probability penalized oracle and excess-risk bounds, and polynomial-size dictionaries retain a near-parametric rate. Thus the paper does not treat representation choice as statistically free: it separates the geometry of the full adaptive hierarchy from the explicit cost of estimating within, or selecting among, realized representations.

Taken together, these results establish BKLs as a mathematically explicit framework for studying hierarchical kernel representations at both the structural and statistical levels. They also point to several next questions. On the analytical side, universality, approximation rates, and depth-dependent representation efficiency remain to be determined. On the statistical side, complexity penalties, entropy conditions, and data-dependent procedures may extend the controlled-dictionary theory to richer or infinite ladder families. On the computational side, finite-dimensional approximations, sparse dictionaries, independent pilot-sample construction, and scalable optimization are natural routes toward practical BKL learning.

\appendix

\begin{figure}[!t]
\centering
\resizebox{\textwidth}{!}{%
  \input{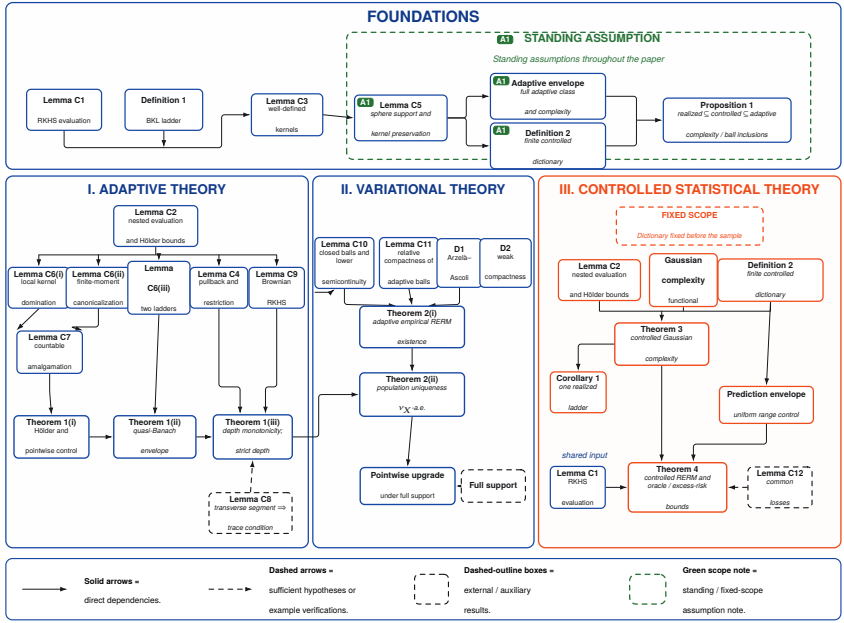}%
}
\caption{Clickable dependency graph of the paper. The footer of the graph
summarizes the meaning of solid and dashed arrows, auxiliary boxes, and scope
notes. In the electronic PDF, clicking a card jumps to the corresponding
statement or supporting result.}
\label{fig:dependency_graph}
\end{figure}

\section{Additional Notation}\label{app:Additional_notations}
The additional notation used in the appendix includes
$|S|$, $|x|$, $\circ$, $\mu\ll\nu$, $T_\#\mu$, $\|\mu\|_{TV}$,
$\im(f)$, $\mm{Rad}$, $\supp(f)$, and $C_c^\infty(\R)$.

For a finite set $S$, $|S|$ denotes its cardinality, while $|x|$ denotes the
absolute value of $x\in\R$. The composition of functions $f$ and $g$ is denoted
by $f\circ g$. Let $\mu,\nu\in\mP(\mU)$. We write $\mu\ll\nu$ when $\mu$ is
absolutely continuous with respect to $\nu$. If $(\mX,\F_\mX)$ and
$(\mY,\F_\mY)$ are measurable spaces, $\mu\in\mP(\mX)$, and
$T:\mX\to\mY$ is measurable, then the pushforward measure $T_\#\mu$ is defined
by
\begin{align}
\left(T_\#\mu\right)(B)
\coloneqq
\mu\left(T^{-1}(B)\right),
\qquad
B\in\F_\mY.
\end{align}
For every $q\in L^1(\mY,T_\#\mu)$,
\begin{align}
\int_\mY q(y)\d(T_\#\mu)(y)
=
\int_\mX q(T(x))\d\mu(x).
\end{align}
For a finite signed measure $\mu$ on $(\mH,\mA)$, its total variation norm is
\begin{align}
\|\mu\|_{TV}
\coloneqq
\sup\left\{
\sum_{i=1}^{\infty}|\mu(A_i)|:
(A_i)_{i=1}^{\infty}\text{ is a measurable partition of }\mH
\right\}.
\end{align}
The phrase ``convergence in total variation'' means convergence in this norm.
For a function $f:\mX\to\mY$, its image is
\begin{align}
\im(f)
\coloneqq
\left\{f(x):x\in\mX\right\}.
\end{align}

The Rademacher distribution on $\{-1,+1\}$ is denoted by $\mm{Rad}$. A family $\F\subseteq\mC(\mX)$ is equicontinuous if, for every
$\varepsilon>0$, there exists $\delta>0$ such that
\begin{align}
\|\b x-\b y\|_2<\delta
\quad\Longrightarrow\quad
|f(\b x)-f(\b y)|<\varepsilon
\end{align}
for every $f\in\F$ and every $\b x,\b y\in\mX$. A functional
$J:\mC(\mX)\to[-\infty,+\infty]$ is lower semicontinuous with respect to
$\|\cdot\|_\infty$ if
\begin{align}
f_n\to f\text{ in }\|\cdot\|_\infty
\quad\Longrightarrow\quad
J(f)
\le
\liminf_{n\to\infty}J(f_n).
\end{align}
Finally,
\begin{align}
\supp(f)
\coloneqq
\overline{\left\{x\in\R:f(x)\neq0\right\}},
\end{align}
and $C_c^\infty(\R)$ denotes the space of smooth real-valued functions on
$\R$ with compact support.

\Cref{app} is dedicated to the proofs of our main results using internal and external lemmas presented in \Cref{app:auxi_lems} and  \Cref{app:Ext_Stats}, respectively. Our auxiliary theoretical results are summarized in Table~\ref{caption:auxiliary-results}. The principal dependencies among the active results are summarized in \Cref{fig:dependency_graph}.

\section{Proofs}\label{app}
This section includes the proofs of our main  results presented in \Cref{sec:results}:  \Cref{thm:charachteriz_funspace}, \Cref{thm:existence_\BK{}}, \Cref{thm:peeling_all_depth_\BK{}} and \Cref{thm:excess-risk-\BK{}}. 

\subsection{Proof of \texorpdfstring{\Cref{thm:charachteriz_funspace}\ref{thm:HLp-analytical}}{Theorem 1(i)}}

\begin{noheadproof}
Let $f \in \HL$ be arbitrary. By definition of $\HL$, there exists a canonical ladder $\LHmuh \in \LHh$ such that $f \in \HLmuh$. By canonical representation, under Assumption~\ref{assumption:normalizing}, for each $l \in [L-1]$ there exists $\hat\mu^{(l)} \in \Psph\left(\SHlmu\right)$ such that
\begin{align}
\hat k^{(l)} = k\left[\SHlmu,\hat\mu^{(l)}\right],
\qquad
\Hlpmuh = \mH_{\hat k^{(l)}}.
\end{align}
Hence, without loss of generality, we work with this canonical representation. Applying Lemma~\ref{lem:RKHSfunctions_regularity} at the top layer $L$, one has
\begin{align}\label{eq:bd_1}
|f(\b x)|
&\stackrel{(a)}{\le}
\|f\|_{\HLmuh}
\left[\hat k^{(L-1)}(\b x,\b x)\right]^{1/2}
\stackrel{(b)}{\le}
\|f\|_{\HLmuh}
\sqrt[\leftroot{-2}\uproot{0}2^{L-1}]{\|\b x\|_2},
\quad \forall \b x \in \mX,
\end{align}
where (a) follows from \eqref{eq:|f(x)|-bound} and (b) uses Lemma~\ref{lem:kernel_boud_induction}. Similarly,
\begin{align}\label{eq:bd_2}
|f(\b x) - f(\b x')|
&\stackrel{(a)}{\le}
\|f\|_{\HLmuh}
\, d_{\hat k^{(L-1)}}(\b x,\b x')
\stackrel{(b)}{\le}
\|f\|_{\HLmuh}
\sqrt[\leftroot{-2}\uproot{0}2^{L-1}]{\|\b x - \b x'\|_2},
\end{align}
for all $\b x,\b x' \in \mX$, where (a) follows from \eqref{eq:|f(x)-f(x)|-bound} and (b) uses Lemma~\ref{lem:kernel_boud_induction}. By the definition of $\CL(f)$,
\begin{align}
\CL(f)
=
\inf_{\LHmuh \in \LHh \,:\, f \in \HLmuh}
\|f\|_{\HLmuh}.
\end{align}
Taking the infimum over all admissible representations in the bounds of \eqref{eq:bd_1} and \eqref{eq:bd_2} yields
\begin{align}
|f(\b x)|
&\le
\CL(f)
\sqrt[\leftroot{-2}\uproot{0}2^{L-1}]{\|\b x\|_2},\\
|f(\b x) - f(\b x')|
&\le
\CL(f)
\sqrt[\leftroot{-2}\uproot{0}2^{L-1}]{\|\b x - \b x'\|_2},
\end{align}
which proves \Cref{thm:charachteriz_funspace}\ref{thm:HLp-analytical}.
\end{noheadproof}
\hfill \BlackBox

\subsection{Proof of \texorpdfstring{\Cref{thm:charachteriz_funspace}\ref{thm:charachteriz_funspace-3}}{Theorem 1(ii)}}
\begin{noheadproof}
Define
\begin{align}
\kappa_L
\coloneqq
\frac{2}{\sqrt[\leftroot{-2}\uproot{0}2^L]{2}}
=
2^{1-2^{-L}}.
\end{align}
We prove successively that $\HL$ is a vector space, that $\CL$ is a quasi-norm with modulus of concavity at most $\kappa_L$, and that the resulting quasi-normed space is complete.

\medskip
\noindent
\textbf{Step 1} (vector-space structure):

\smallskip
\noindent
\emph{Scalar multiplication.}
Let $f\in\HL$ and $c\in\R$. The case $c=0$ is immediate because the zero function belongs to every top-layer RKHS. Suppose that $c\neq0$. For every canonical ladder $\LHmuh\in\LHh$ satisfying $f\in\HLmuh$, linearity of the RKHS gives
\begin{align}
cf\in\HLmuh,
\qquad
\|cf\|_{\HLmuh}
=
|c|\,\|f\|_{\HLmuh}.
\end{align}
Taking the infimum over all canonical realizations of $f$ yields
\begin{align}
\CL(cf)
\le
|c|\,\CL(f).
\label{eq:abs-homog-forward-repaired}
\end{align}
Applying \eqref{eq:abs-homog-forward-repaired} to $c^{-1}(cf)$ gives the reverse inequality. Consequently,
\begin{align}
\CL(cf)
=
|c|\,\CL(f),
\qquad
f\in\HL,
\quad
c\in\R.
\label{eq:absolute-homogeneity}
\end{align}
In particular, $\HL$ is closed under scalar multiplication.

\smallskip
\noindent
\emph{Addition.}
Let $f_b,f_c\in\HL$. Choose canonical ladders $\LH_{\hat\mu_b},\LH_{\hat\mu_c}\in\LHh$ such that
\begin{align}
f_b\in\mH_{\hat\mu_b}^{(L)},
\qquad
f_c\in\mH_{\hat\mu_c}^{(L)}.
\end{align}
By \Cref{lem:measures_inequality}\ref{lem:two-ladder-amalgamation}, there exists a canonical ladder $\LH_{\hat\mu_a}\in\LHh$ for which
\begin{align}
\mH_{\hat\mu_b}^{(L)}
\cup
\mH_{\hat\mu_c}^{(L)}
\subseteq
\mH_{\hat\mu_a}^{(L)}.
\end{align}
Therefore $f_b+f_c\in\mH_{\hat\mu_a}^{(L)}\subseteq\HL$, and $\HL$ is a vector space.

\medskip
\noindent
\textbf{Step 2} (quasi-norm properties):

\smallskip
\noindent
\emph{Positive definiteness.}
Let $f\in\HL$ satisfy $\CL(f)=0$. For every $\varepsilon>0$, the definition of the infimum provides a canonical ladder $\LHmuh\in\LHh$ such that
\begin{align}
f\in\HLmuh,
\qquad
\|f\|_{\HLmuh}<\varepsilon.
\end{align}
Using the pointwise estimate \eqref{eq:HL-pointwise}, for every $\b x\in\mX$ one obtains
\begin{align}
|f(\b x)|
\le
\varepsilon
\sqrt[\leftroot{-2}\uproot{0}2^{L-1}]{\|\b x\|_2}.
\end{align}
Letting $\varepsilon\downarrow0$ gives $f(\b x)=0$ for every $\b x\in\mX$. Hence $f=0$.

\smallskip
\noindent
\emph{Absolute homogeneity.}
This is exactly \eqref{eq:absolute-homogeneity}.

\smallskip
\noindent
\emph{Quasi-triangle inequality.}
Fix $f_b,f_c\in\HL$ and $\varepsilon>0$. Choose a canonical ladder $\LH_{\hat\mu_b}\in\LHh$ for $f_b$ and a canonical ladder $\LH_{\hat\mu_c}\in\LHh$ for $f_c$ such that
\begin{align}
\|f_b\|_{\mH_{\hat\mu_b}^{(L)}}
&\le
\CL(f_b)+\varepsilon,
&
\|f_c\|_{\mH_{\hat\mu_c}^{(L)}}
&\le
\CL(f_c)+\varepsilon.
\label{eq:epsilon-realizations-repaired}
\end{align}
By \Cref{lem:measures_inequality}\ref{lem:two-ladder-amalgamation}, there exists a canonical ladder $\LH_{\hat\mu_a}\in\LHh$ such that
\begin{align}
\|f_b+f_c\|_{\mH_{\hat\mu_a}^{(L)}}
\le
\kappa_L
\left(
\|f_b\|_{\mH_{\hat\mu_b}^{(L)}}
+
\|f_c\|_{\mH_{\hat\mu_c}^{(L)}}
\right).
\end{align}
Since $\CL$ is the infimum over all canonical realizations, \eqref{eq:epsilon-realizations-repaired} implies
\begin{align}
\CL(f_b+f_c)
&\le
\kappa_L
\left(
\CL(f_b)+\CL(f_c)+2\varepsilon
\right).
\end{align}
Letting $\varepsilon\downarrow0$ yields
\begin{align}
\CL(f_b+f_c)
\le
\kappa_L
\left(
\CL(f_b)+\CL(f_c)
\right).
\label{eq:quasi-triangle-repaired}
\end{align}
Thus $\CL$ is a quasi-norm on $\HL$ with modulus of concavity at most $\kappa_L<2$.

\medskip
\noindent
\textbf{Step 3} (completeness):
Let $(f_n)_{n\ge1}$ be a $\CL$-Cauchy sequence. By \eqref{eq:HL-pointwise} and compactness of $\mX$,
\begin{align}
\|f_n-f_m\|_\infty
\le
R_X^{2^{-(L-1)}}\CL(f_n-f_m),
\qquad
R_X\coloneqq\sup_{\b x\in\mX}\|\b x\|_2<\infty.
\end{align}
Hence $(f_n)_{n\ge1}$ is Cauchy in $\left(C(\mX),\|\cdot\|_\infty\right)$ and converges uniformly to some $f_\infty\in C(\mX)$. Choose a strictly increasing subsequence $(f_{n_j})_{j\ge0}$ such that
\begin{align}
\CL\left(f_{n_{j+1}}-f_{n_j}\right)
\le
2^{-2(j+2)L},
\qquad
j\ge0.
\label{eq:rapid-subsequence-repaired}
\end{align}
Set
\begin{align}
h_0\coloneqq f_{n_0},
\qquad
h_j\coloneqq f_{n_j}-f_{n_{j-1}},
\quad j\ge1.
\end{align}
For $h_0$, choose any canonical realization with finite top-layer norm. For each $j\ge1$, choose a canonical ladder $\LH_{\hat\mu_j}\in\LHh$ such that
\begin{align}
h_j\in\mH_{\hat\mu_j}^{(L)},
\qquad
\|h_j\|_{\mH_{\hat\mu_j}^{(L)}}
\le
\CL(h_j)+2^{-2(j+1)L}
\le
2^{1-2(j+1)L}.
\label{eq:rapid-realizations-repaired}
\end{align}
The final inequality follows from \eqref{eq:rapid-subsequence-repaired} with index $j-1$. Apply \Cref{lemma:TVconvergence} to the sequence of canonical ladders realizing $(h_j)_{j\ge0}$. There exist a canonical ladder $\LH_{\hat\mu_*}\in\LHh$ and a constant $A_L<\infty$, depending only on $L$, such that
\begin{align}
\|h_j\|_{\mH_{\hat\mu_*}^{(L)}}
\le
A_L
2^{(j+1)(L-1)/2}
\|h_j\|_{\mH_{\hat\mu_j}^{(L)}},
\qquad
j\ge0.
\label{eq:countable-embedding-used}
\end{align}
The term $j=0$ has finite norm. For $j\ge1$, combining \eqref{eq:rapid-realizations-repaired} and \eqref{eq:countable-embedding-used} gives
\begin{align}
\|h_j\|_{\mH_{\hat\mu_*}^{(L)}}
&\le
2A_L
2^{(j+1)(L-1)/2}
2^{-2(j+1)L}\\
&=
2A_L
2^{-(j+1)(3L+1)/2}.
\end{align}
The series on the right is summable. Therefore,
\begin{align}
\sum_{j=0}^{\infty}
\|h_j\|_{\mH_{\hat\mu_*}^{(L)}}
<\infty.
\end{align}
Since $\mH_{\hat\mu_*}^{(L)}$ is a Hilbert space, the series $\sum_{j=0}^{\infty}h_j$ converges in $\mH_{\hat\mu_*}^{(L)}$ to some $h$. Its $m$-th partial sum equals $f_{n_m}$. RKHS-norm convergence implies pointwise convergence, whereas $f_{n_m}\to f_\infty$ uniformly. Hence $h=f_\infty$ as functions on $\mX$, so
\begin{align}
f_\infty\in\mH_{\hat\mu_*}^{(L)}\subseteq\HL.
\end{align}
Moreover,
\begin{align}
\CL\left(f_{n_m}-f_\infty\right)
&\le
\left\|f_{n_m}-f_\infty\right\|_{\mH_{\hat\mu_*}^{(L)}}\\
&\le
\sum_{j=m+1}^{\infty}
\|h_j\|_{\mH_{\hat\mu_*}^{(L)}}
\longrightarrow
0.
\end{align}
Because $(f_n)_{n\ge1}$ is $\CL$-Cauchy and the subsequence $(f_{n_m})_{m\ge0}$ converges to $f_\infty$ in $\CL$, the full sequence also converges to $f_\infty$ in $\CL$. Thus $\left(\HL,\CL\right)$ is complete.
\end{noheadproof}
\hfill \BlackBox

\subsection{Proof of \texorpdfstring{\Cref{thm:charachteriz_funspace}\ref{thm:charachteriz_funspace-4}}{Theorem 1(iii)}}
\begin{noheadproof}
\begin{itemize}
\item Inclusion $\HL\subseteq\HLplus$:
By \Cref{thm:charachteriz_funspace}\ref{thm:charachteriz_funspace-3}
applied at depth $L+1$, $\HLplus$ is a vector space and therefore contains the zero
function. It remains to consider a nonzero $f\in\HL$.
By the definition of the adaptive envelope, there exists a canonical $L$-level ladder
$\LHmuh\in\LHh$ such that
\begin{align}
\LHmuh\in\LHh,
\qquad
f\in\HLmuh.
\label{eq:eps-realization}
\end{align}
Set
\begin{align}
q_f
\coloneqq
\left\|f\right\|_{\HLmuh}
>0,
\qquad
\widetilde u_f
\coloneqq
\frac{f}{q_f}
\in
\S_1\left(\HLmuh\right),
\end{align}
and append the canonical measure
\begin{align}
\hat\mu^{(L)}
\coloneqq
\delta_{\widetilde u_f}
\in
\Psph\left(\S_1\left(\HLmuh\right)\right).
\end{align}
This keeps the first $L$ layers unchanged and defines a canonical $(L+1)$-level ladder.
Since $\widetilde u_f\in\mC(\mX)$ and $\mX$ is compact,
\begin{align}
R_f
\coloneqq
\sup_{\b x\in\mX}
\left|\widetilde u_f(\b x)\right|
<\infty.
\label{eq:R-def}
\end{align}
Choose $\chi_f\in C_c^\infty(\R)$ satisfying
$\chi_f(t)=1$ for every $|t|\le R_f$, and define
\begin{align}
\phi_f(t)
\coloneqq
 t\chi_f(t),
\qquad
h_f
\coloneqq
q_f\phi_f.
\label{eq:h-def}
\end{align}
The function $h_f$ is absolutely continuous, satisfies $h_f(0)=0$, and has a
compactly supported derivative in $L^2(\R)$. Hence
$h_f\in\mH_{\kB}$ by \Cref{lem:brownian-rkhs}. Let $Z_f\coloneqq\widetilde u_f(\mX)$. The restriction of $h_f$ to $Z_f$ belongs to
the RKHS of the Brownian kernel restricted to $Z_f\times Z_f$. Applying
\Cref{lem:kernel_transformation} to the surjection
$\widetilde u_f:\mX\to Z_f$ shows that
$h_f\circ\widetilde u_f$ belongs to the last-layer pullback RKHS of the appended
ladder. Moreover, for every $\b x\in\mX$,
\begin{align}
h_f\left(\widetilde u_f(\b x)\right)
&\stackrel{(a)}{=}
q_f\phi_f\left(\widetilde u_f(\b x)\right)
\stackrel{(b)}{=}
q_f\widetilde u_f(\b x)
\stackrel{(c)}{=}
f(\b x),
\qquad
f\in\mH_{\hat\mu}^{(L+1)}.
\label{eq:IRKHS-last-layer}
\end{align}
Here (a) uses \eqref{eq:h-def}, (b) follows from
$\widetilde u_f(\mX)\subseteq[-R_f,R_f]$ and the definition of $\phi_f$, and
(c) uses the definition of $\widetilde u_f$. Thus $f\in\HLplus$, proving
$\HL\subseteq\HLplus$.

\item Strict inclusion $\HL\subsetneq\HLplus$:
Assume the trace condition in the statement.

\medskip
\noindent
\textbf{Step 1} (choice of the Brownian nonlinearity):
Let
\begin{align}
\alpha_L
\coloneqq
2^{-(L-1)}.
\end{align}
Since
\begin{align}
2^{-(L-1)/L}
=
2^{-1+1/L}
>
\frac12,
\end{align}
the interval below is nonempty. Fix
\begin{align}
a
\in
\left(
\frac12,
2^{-(L-1)/L}
\right),
\end{align}
and choose $\eta\in C_c^\infty(\R)$ such that
\begin{align}
0\le\eta\le1,
\qquad
\eta(t)=1
\quad\text{for every }|t|\le\frac12.
\end{align}
Define
\begin{align}
t_+
\coloneqq
\max\{t,0\},
\qquad
g_a(t)
\coloneqq
\eta(t)(t_+)^a.
\end{align}
Then $g_a(0)=0$, $g_a$ is absolutely continuous, and for almost every $t\in\R$,
\begin{align}
g_a'(t)
=
\eta'(t)(t_+)^a
+
a\eta(t)t^{a-1}\I_{(0,\infty)}(t).
\end{align}
The first term belongs to $L^2(\R)$ because it is bounded and compactly supported.
For the second term, the only possible singularity is at zero, and
\begin{align}
\int_0^1 t^{2a-2}\d t
=
\frac{1}{2a-1}
<\infty
\end{align}
because $a>\frac12$. The compact support of $\eta$ controls the remaining part.
Consequently $g_a'\in L^2(\R)$, and
\begin{align}
g_a\in\mH_{\kB}
\end{align}
by \Cref{lem:brownian-rkhs}.

\medskip
\noindent
\textbf{Step 2} (recursive trace and local power bounds):
Starting from the function $u_1$ in the theorem statement, define pointwise
\begin{align}
u_{l+1}
\coloneqq
g_a\circ u_l,
\qquad
l=1,\ldots,L.
\end{align}
We prove that there exist $\delta>0$ and constants $c_l,C_l>0$ such that
\begin{align}
c_l t^{a^{l-1}}
\le
u_l(\gamma(t))
\le
C_l t^{a^{l-1}},
\qquad
0\le t\le\delta,
\qquad
l=1,\ldots,L+1.
\label{eq:um-bounds-trace}
\end{align}
For $l=1$, the claim holds on $[0,\delta_1]$ with
$\delta_1\coloneqq\rho$ and the constants $c_1,C_1$ from the assumption.
Assume that, for some $l\le L$, the claimed bounds hold on
$[0,\delta_l]$. Since $C_l t^{a^{l-1}}\to0$ as $t\downarrow0$, choose
$\delta_{l+1}\in(0,\delta_l]$ such that
\begin{align}
C_l\delta_{l+1}^{a^{l-1}}
\le
\frac12.
\end{align}
Then, for every $t\in[0,\delta_{l+1}]$,
\begin{align}
0
\le
u_l(\gamma(t))
\le
\frac12,
\end{align}
so $\eta\left(u_l(\gamma(t))\right)=1$ and
\begin{align}
u_{l+1}(\gamma(t))
&\stackrel{(a)}{=}
g_a\left(u_l(\gamma(t))\right)
\stackrel{(b)}{=}
u_l(\gamma(t))^a.
\end{align}
Here (a) is the recursive definition and (b) uses the definition of $g_a$ on
$[0,\frac12]$. Raising the induction bounds to the positive power $a$ gives
\begin{align}
c_l^a t^{a^l}
\le
u_{l+1}(\gamma(t))
\le
C_l^a t^{a^l}.
\end{align}
Thus the induction continues with
$c_{l+1}=c_l^a$ and $C_{l+1}=C_l^a$.
After the finitely many steps, set $\delta\coloneqq\delta_{L+1}$.
Since $\delta_{L+1}\le\delta_l$ for every $l$, all the preceding bounds hold
simultaneously on $[0,\delta]$. This proves \eqref{eq:um-bounds-trace}. In particular,
\begin{align}
u_{L+1}(\gamma(t))
\ge
c_{L+1}t^{a^L},
\qquad
0\le t\le\delta.
\label{eq:lower-final-trace}
\end{align}
Moreover, $u_1(\b x_0)=0$ and $g_a(0)=0$ imply recursively that
\begin{align}
u_l(\b x_0)=0,
\qquad
l=1,\ldots,L+1.
\label{eq:zero-trace}
\end{align}
For every $l\in[L]$, the lower bound in \eqref{eq:um-bounds-trace} is strictly
positive when $t\in(0,\delta]$, so $u_l$ is not the zero function.

\medskip
\noindent
\textbf{Step 3} (construction of a canonical $(L+1)$-level ladder):
Set $\mH_{\hat\mu}^{(1)}\coloneqq\Honemu$, so that
$u_1\in\mH_{\hat\mu}^{(1)}$. Suppose recursively that the first $l-1$ canonical
measures have been chosen and
$u_l\in\mH_{\hat\mu}^{(l)}$ for some $l\in[L]$. Define
\begin{align}
q_l
\coloneqq
\left\|u_l\right\|_{\mH_{\hat\mu}^{(l)}}
>0,
\qquad
\widetilde u_l
\coloneqq
\frac{u_l}{q_l}
\in
\S_1\left(\mH_{\hat\mu}^{(l)}\right),
\end{align}
and choose
\begin{align}
\hat\mu^{(l)}
\coloneqq
\delta_{\widetilde u_l}
\in
\Psph\left(
\S_1\left(\mH_{\hat\mu}^{(l)}\right)
\right).
\end{align}
The resulting next-layer kernel is
\begin{align}
\hat k^{(l)}(\b x,\b x')
=
\kB\left(
\widetilde u_l(\b x),
\widetilde u_l(\b x')
\right).
\end{align}
Define
\begin{align}
h_l(s)
\coloneqq
g_a(q_l s).
\end{align}
Then $h_l(0)=0$, and a change of variables gives
\begin{align}
\int_{\R}|h_l'(s)|^2\d s
&\stackrel{(a)}{=}
\int_{\R}q_l^2\left|g_a'(q_l s)\right|^2\d s
\stackrel{(b)}{=}
q_l\int_{\R}|g_a'(t)|^2\d t
<\infty.
\end{align}
Here (a) follows from the chain rule and (b) uses $t=q_ls$.
Thus $h_l\in\mH_{\kB}$ by \Cref{lem:brownian-rkhs}.
Restricting $h_l$ to
$Z_l\coloneqq\widetilde u_l(\mX)$ and applying
\Cref{lem:kernel_transformation} to
$\widetilde u_l:\mX\to Z_l$ yields
\begin{align}
u_{l+1}(\b x)
&\stackrel{(a)}{=}
g_a\left(u_l(\b x)\right)
\stackrel{(b)}{=}
g_a\left(q_l\widetilde u_l(\b x)\right)
\stackrel{(c)}{=}
h_l\left(\widetilde u_l(\b x)\right)
\in
\mH_{\hat\mu}^{(l+1)}.
\end{align}
Here (a) is the recursive definition, (b) uses
$u_l=q_l\widetilde u_l$, and (c) uses the definition of $h_l$.
Induction over $l=1,\ldots,L$ constructs a canonical $(L+1)$-level ladder and gives
\begin{align}
u_{L+1}
\in
\mH_{\hat\mu}^{(L+1)}
\subseteq
\HLplus.
\label{eq:uLplus-trace}
\end{align}

\medskip
\noindent
\textbf{Step 4} (Hölder contradiction):
Suppose, toward a contradiction, that $u_{L+1}\in\HL$. Let $L_\gamma>0$ be a
Lipschitz constant of $\gamma$. Applying \eqref{eq:HL-holder} and
\eqref{eq:zero-trace}, for every $t\in[0,\delta]$ we obtain
\begin{align}
u_{L+1}(\gamma(t))
&\stackrel{(a)}{=}
\left|
u_{L+1}(\gamma(t))-u_{L+1}(\b x_0)
\right|\\
&\stackrel{(b)}{\le}
\CL\left(u_{L+1}\right)
\left\|\gamma(t)-\gamma(0)\right\|_2^{\alpha_L}
\stackrel{(c)}{\le}
\CL\left(u_{L+1}\right)L_\gamma^{\alpha_L}t^{\alpha_L}.
\end{align}
Here (a) uses \eqref{eq:zero-trace}, (b) applies the Hölder estimate of
\Cref{thm:charachteriz_funspace}\ref{thm:HLp-analytical}, and (c) uses the
Lipschitz property of $\gamma$. Combining this upper bound with
\eqref{eq:lower-final-trace} yields
\begin{align}
c_{L+1}t^{a^L}
\le
\CL\left(u_{L+1}\right)L_\gamma^{\alpha_L}t^{\alpha_L}.
\end{align}
By the choice of $a$,
\begin{align}
a^L
<
\alpha_L,
\end{align}
and therefore
\begin{align}
t^{a^L-\alpha_L}
\longrightarrow
\infty
\qquad
(t\downarrow0),
\end{align}
which contradicts the preceding inequality. Hence
\begin{align}
u_{L+1}
\notin
\HL.
\label{eq:notin-trace}
\end{align}
Together with \eqref{eq:uLplus-trace}, this proves
$u_{L+1}\in\HLplus\setminus\HL$ and hence strict inclusion.

\medskip
\noindent
\textbf{Step 5} (ballwise separation):
Set
\begin{align}
f_\star
\coloneqq
u_{L+1},
\qquad
\beta_L
\coloneqq
a^L,
\qquad
\rho_\star
\coloneqq
\delta,
\qquad
c_\star
\coloneqq
c_{L+1},
\qquad
K_\gamma
\coloneqq
L_\gamma^{\alpha_L}.
\end{align}
Then
\begin{align}
f_\star(\b x_0)=0,
\qquad
f_\star(\gamma(t))
\ge
c_\star t^{\beta_L},
\qquad
0\le t\le\rho_\star,
\end{align}
and $\beta_L<\alpha_L$.
Fix $r>0$ and define
\begin{align}
t_r
\coloneqq
\min\left\{
\rho_\star,
\left(
\frac{c_\star}{4rK_\gamma}
\right)^{\frac{1}{\alpha_L-\beta_L}}
\right\}
>0.
\end{align}
Since $t_r$ does not exceed either term in this minimum,
\begin{align}
rK_\gamma t_r^{\alpha_L}
\le
\frac{c_\star}{4}t_r^{\beta_L}.
\label{eq:choose-tr}
\end{align}
Let $g\in\BL_r$ be arbitrary. By \eqref{eq:HL-holder},
\begin{align}
\left|g(\gamma(t_r))-g(\b x_0)\right|
&\stackrel{(a)}{\le}
\CL(g)\left\|\gamma(t_r)-\gamma(0)\right\|_2^{\alpha_L}\\
&\stackrel{(b)}{\le}
rK_\gamma t_r^{\alpha_L}
\stackrel{(c)}{\le}
\frac{c_\star}{4}t_r^{\beta_L}.
\label{eq:g-holder-ball}
\end{align}
Here (a) applies \Cref{thm:charachteriz_funspace}\ref{thm:HLp-analytical},
(b) uses $\CL(g)\le r$ and the Lipschitz property of $\gamma$, and (c) uses
\eqref{eq:choose-tr}. If

\begin{align}
|g(\b x_0)|
\ge
\frac{c_\star}{4}t_r^{\beta_L},
\end{align}
then, since $f_\star(\b x_0)=0$,
\begin{align}
\left\|f_\star-g\right\|_\infty
\ge
|f_\star(\b x_0)-g(\b x_0)|
\ge
\frac{c_\star}{4}t_r^{\beta_L}.
\end{align}
Otherwise,
\begin{align}
|g(\b x_0)|
<
\frac{c_\star}{4}t_r^{\beta_L},
\end{align}
and \eqref{eq:g-holder-ball} gives
\begin{align}
|g(\gamma(t_r))|
&\le
|g(\b x_0)|
+
|g(\gamma(t_r))-g(\b x_0)|
<
\frac{c_\star}{2}t_r^{\beta_L}.
\end{align}
Since $f_\star(\gamma(t_r))\ge c_\star t_r^{\beta_L}$,
\begin{align}
\left\|f_\star-g\right\|_\infty
&\ge
\left|f_\star(\gamma(t_r))-g(\gamma(t_r))\right|\\
&\ge
c_\star t_r^{\beta_L}
-
|g(\gamma(t_r))|
>
\frac{c_\star}{2}t_r^{\beta_L}.
\end{align}
Thus in both cases,
\begin{align}
\left\|f_\star-g\right\|_\infty
\ge
\frac{c_\star}{4}t_r^{\beta_L}.
\end{align}
Taking the infimum over $g\in\BL_r$ yields
\begin{align}
\inf_{g\in\BL_r}
\left\|f_\star-g\right\|_\infty
\ge
\frac{c_\star}{4}t_r^{\beta_L}
>0.
\end{align}
This proves \eqref{eq:BLr-nonapprox}.
\end{itemize}
\end{noheadproof}\hfill \BlackBox\\

\subsection{Proof of \texorpdfstring{\Cref{thm:existence_\BK{}}\ref{thm:existence_\BK{}-1}}{Theorem 2(i)}}
\label{thm-existence_\BK{}:proof-1}
\begin{noheadproof}
Set
\begin{align}
j_*
\coloneqq
\inf_{f\in\HL}J_n(f).
\end{align}
Since $0\in\HL$, $\CL(0)=0$, and the loss is real-valued, one has
\begin{align}
j_*
\le
J_n(0)
=
\frac1n\sum_{i=1}^n\ell(0,y_i)
<
\infty.
\end{align}
On the other hand, \eqref{eq:loss-global-lower-bound} and the nonnegativity of the quasi-norm give, for every $f\in\HL$,
\begin{align}
J_n(f)
&\stackrel{(a)}{=}
\frac1n\sum_{i=1}^n\ell\left(f(\b x_i),y_i\right)
+
\lambda\CL(f)\\
&\stackrel{(b)}{\ge}
b_\ell
+
\lambda\CL(f)
\stackrel{(c)}{\ge}
b_\ell.
\end{align}
Here (a) uses the definitions of $J_n$ and $\mR_n$, (b) applies \eqref{eq:loss-global-lower-bound} to every summand, and (c) uses $\CL(f)\ge0$. Consequently,
\begin{align}
-\infty
<
b_\ell
\le
j_*
<
\infty.
\end{align}
Choose a minimizing sequence $(f_m)_{m\ge1}\subset\HL$ satisfying
\begin{align}
J_n(f_m)
\le
j_*+\frac1m,
\qquad
m\ge1.
\label{eq:RERM-near-minimizing-sequence}
\end{align}
Combining the preceding lower bound with \eqref{eq:RERM-near-minimizing-sequence} yields
\begin{align}
\lambda\CL(f_m)
&\stackrel{(a)}{\le}
J_n(f_m)-b_\ell
\stackrel{(b)}{\le}
j_*+1-b_\ell,
\qquad
m\ge1,
\end{align}
where (a) follows from $J_n(f_m)\ge b_\ell+\lambda\CL(f_m)$ and (b) uses $m^{-1}\le1$. Therefore,
\begin{align}
\CL(f_m)
\le
r_*
\coloneqq
\frac{j_*+1-b_\ell}{\lambda}
<
\infty,
\qquad
m\ge1.
\label{eq:RERM-coercive}
\end{align}
Thus $(f_m)_{m\ge1}\subset\BL_{r_*}$.\par\noindent By \Cref{lem:HL_precompact}, the ball $\BL_{r_*}$ is relatively compact in $\left(\mC(\mX),\|\cdot\|_\infty\right)$. Hence there exist a subsequence, not relabeled, and a function $f\in\mC(\mX)$ such that
\begin{align}
f_m
\longrightarrow
f
\qquad
\text{uniformly on }\mX.
\label{eq:RERM-uniform-subsequence}
\end{align}
By the lower semicontinuity established in \Cref{lem:lsc-CL},
\begin{align}
\CL(f)
\le
\liminf_{m\to\infty}\CL(f_m)
\le
r_*,
\label{eq:RERM-limit-in-HL}
\end{align}
so $f\in\HL$. For each $i\in[n]$, \eqref{eq:RERM-uniform-subsequence} gives $f_m(\b x_i)\to f(\b x_i)$. Since $\ell(\cdot,y_i)$ is continuous,
\begin{align}
\ell\left(f_m(\b x_i),y_i\right)
\longrightarrow
\ell\left(f(\b x_i),y_i\right).
\end{align}
Because the empirical risk is a finite average, it follows that
\begin{align}
\mR_n(f_m)
\longrightarrow
\mR_n(f).
\label{eq:RERM-risk-continuity}
\end{align}
Using \eqref{eq:RERM-limit-in-HL} and \eqref{eq:RERM-risk-continuity},
\begin{align}
J_n(f)
&\stackrel{(a)}{=}
\mR_n(f)+\lambda\CL(f)\\
&\stackrel{(b)}{\le}
\lim_{m\to\infty}\mR_n(f_m)
+
\lambda\liminf_{m\to\infty}\CL(f_m)\\
&\stackrel{(c)}{\le}
\liminf_{m\to\infty}
\left(
\mR_n(f_m)+\lambda\CL(f_m)
\right)\\
&\stackrel{(d)}{=}
j_*.
\end{align}
Equality (a) is the definition of $J_n$. Inequality (b) uses \eqref{eq:RERM-limit-in-HL} and \eqref{eq:RERM-risk-continuity}. Inequality (c) uses the elementary fact that if $a_m\to a$, then $a+\liminf_m b_m\le\liminf_m(a_m+b_m)$. Equality (d) follows from \eqref{eq:RERM-near-minimizing-sequence}. Since $j_*$ is the infimum of $J_n$ over $\HL$, the reverse inequality $j_*\le J_n(f)$ holds automatically. Hence
\begin{align}
J_n(f)
=
j_*,
\end{align}
and $f$ is a minimizer of \eqref{eq:fhatlambda}.
\end{noheadproof}\hfill\BlackBox

\subsection{Proof of \texorpdfstring{\Cref{thm:existence_\BK{}}\ref{thm:existence_\BK{}-2}}{Theorem 2(ii)}}
\label{thm-existence_\BK{}:proof}
\begin{noheadproof}
Let $f^*,g^*\in\argmin_{f\in\HL}\mR(f)$ and fix $t\in(0,1)$. Since $\HL$ is a vector space by \Cref{thm:charachteriz_funspace}\ref{thm:charachteriz_funspace-3},
\begin{align}
f_t
\coloneqq
tf^*+(1-t)g^*
\in
\HL.
\end{align}
The assumption $\mR(0)<\infty$ implies that the minimum risk is finite, and therefore
\begin{align}
\mR(f^*)
=
\mR(g^*)
=
\inf_{h\in\HL}\mR(h)
<
\infty.
\end{align}
Define the shifted loss
\begin{align}
\widetilde\ell(u,y)
\coloneqq
\ell(u,y)-b_\ell.
\end{align}
By \eqref{eq:loss-global-lower-bound}, $\widetilde\ell\ge0$, and strict convexity of $\ell(\cdot,y)$ is equivalent to strict convexity of $\widetilde\ell(\cdot,y)$. Convexity gives, for every $(x,y)\in\mX\times\mY$,
\begin{align}
\widetilde\ell\left(f_t(x),y\right)
\le
t\widetilde\ell\left(f^*(x),y\right)
+
(1-t)\widetilde\ell\left(g^*(x),y\right).
\label{eq:population-convex-loss-pointwise}
\end{align}
The right-hand side of \eqref{eq:population-convex-loss-pointwise} is integrable because $f^*$ and $g^*$ have finite risk. Integrating and adding back $b_\ell$ yields
\begin{align}
\mR(f_t)
\le
t\mR(f^*)+(1-t)\mR(g^*)
=
\inf_{h\in\HL}\mR(h).
\end{align}
By the definition of the infimum, equality must hold throughout, so $f_t$ is also a population minimizer. Set
\begin{align}
A
\coloneqq
\left\{
 x\in\mX:
 f^*(x)\neq g^*(x)
\right\}.
\end{align}
Suppose, for contradiction, that $\nu_X(A)>0$. Define
\begin{align}
\Delta_t(x,y)
&\coloneqq
t\widetilde\ell\left(f^*(x),y\right)
+
(1-t)\widetilde\ell\left(g^*(x),y\right)\\
&\hspace{1.3cm}
-
\widetilde\ell\left(f_t(x),y\right).
\end{align}
By convexity, $\Delta_t\ge0$ everywhere. By strict convexity,
\begin{align}
\Delta_t(x,y)
>
0,
\qquad
x\in A,
\quad
 y\in\mY.
\end{align}
Moreover,
\begin{align}
0
\le
\Delta_t(x,y)
\le
t\widetilde\ell\left(f^*(x),y\right)
+
(1-t)\widetilde\ell\left(g^*(x),y\right),
\end{align}
so $\Delta_t$ is integrable. Since
\begin{align}
\nu_{XY}\left(A\times\mY\right)
=
\nu_X(A)
>
0,
\end{align}
a nonnegative measurable function that is strictly positive on $A\times\mY$ has strictly positive integral. Hence
\begin{align}
\int_{\mX\times\mY}
\Delta_t(x,y)
\,\d\nu_{XY}(x,y)
>
0.
\end{align}
Equivalently,
\begin{align}
\mR(f_t)
<
t\mR(f^*)+(1-t)\mR(g^*),
\end{align}
which contradicts minimality. Therefore,
\begin{align}
\nu_X(A)
=
0,
\end{align}
and \eqref{eq:population-minimizer-ae-uniqueness} follows. Finally, assume $\operatorname{supp}(\nu_X)=\mX$. If $f^*(x_0)\neq g^*(x_0)$ at some $x_0\in\mX$, continuity of $f^*-g^*$ gives an open neighborhood $U$ of $x_0$ in $\mX$ on which $f^*\neq g^*$. Since $x_0\in\operatorname{supp}(\nu_X)$, one has $\nu_X(U)>0$, contradicting the almost-everywhere equality already proved. Thus $f^*=g^*$ pointwise on $\mX$.
\end{noheadproof}\hfill\BlackBox

\subsection{Proof of \texorpdfstring{\Cref{thm:peeling_all_depth_\BK{}}}{Theorem 3}}
\label{thm:peeling_all_depth_\BK{}:proof}
\begin{noheadproof}
Fix a sample $\{\b x_i\}_{i=1}^{n}\subset\mX$ and abbreviate
\begin{align}
\rho_{L,n}
\coloneqq
\left(
\max_{i\in[n]}
\left\|\b x_i\right\|_2^2
\right)^{2^{-L}}.
\end{align}
For each $m\in[M]$, let
\begin{align}
k_m
\coloneqq
\hat k_m^{(L-1)}
\end{align}
be the kernel of the top-layer RKHS $\mH_{\hat\mu_m}^{(L)}$, and let $\b K_m\in\R^{n\times n}$ be its Gram matrix,
\begin{align}
(\b K_m)_{ij}
\coloneqq
k_m(\b x_i,\b x_j),
\qquad
 i,j\in[n].
\end{align}
For $\bm\varepsilon\in\R^n$, define
\begin{align}
Z_m(\bm\varepsilon)
\coloneqq
\sup_{f\in\B_r\left(\mH_{\hat\mu_m}^{(L)}\right)}
\frac{1}{n}
\sum_{i=1}^{n}
\varepsilon_i f(\b x_i).
\end{align}
Since $\BLdictr$ is the union of these $M$ RKHS balls, one has
\begin{align}
\eG\left(\BLdictr\right)
\stackrel{(a)}{=}
\eE\left[
\max_{m\in[M]}
Z_m(\bm\varepsilon)
\right],
\label{eq:controlled-union-max}
\end{align}
where (a) follows from \eqref{eq:controlled-BKL-ball} and the fact that the supremum of a function over a finite union is the maximum of the individual suprema. We first study one fixed $m\in[M]$. By the reproducing property and Hilbert-space duality,
\begin{align}
Z_m(\bm\varepsilon)
&\stackrel{(a)}{=}
\sup_{\left\|f\right\|_{\mH_{\hat\mu_m}^{(L)}}\le r}
\left\langle
f,
\frac{1}{n}
\sum_{i=1}^{n}
\varepsilon_i
k_m(\cdot,\b x_i)
\right\rangle_{\mH_{\hat\mu_m}^{(L)}}
\nonumber\\
&\stackrel{(b)}{=}
\frac{r}{n}
\left\|
\sum_{i=1}^{n}
\varepsilon_i
k_m(\cdot,\b x_i)
\right\|_{\mH_{\hat\mu_m}^{(L)}}
\nonumber\\
&\stackrel{(c)}{=}
\frac{r}{n}
\sqrt{
\bm\varepsilon^{\top}
\b K_m
\bm\varepsilon
}.
\label{eq:controlled-rkhs-duality}
\end{align}
Here (a) uses
\begin{align}
f(\b x_i)
=
\left\langle
f,k_m(\cdot,\b x_i)
\right\rangle_{\mH_{\hat\mu_m}^{(L)}},
\end{align}
(b) evaluates the support function of the radius-$r$ Hilbert ball, and (c) expands the squared RKHS norm and applies the reproducing property once more. By \eqref{eq:nested:point-eval} in \Cref{lem:kernel_boud_induction}, for every $i\in[n]$,
\begin{align}
\sqrt{k_m(\b x_i,\b x_i)}
&\stackrel{(a)}{\le}
\left\|\b x_i\right\|_2^{2^{-(L-1)}}
\stackrel{(b)}{=}
\left(
\left\|\b x_i\right\|_2^2
\right)^{2^{-L}}
\stackrel{(c)}{\le}
\rho_{L,n}.
\label{eq:controlled-diagonal-bound}
\end{align}
Here (a) applies the nested diagonal estimate at layer $l=L-1$, (b) rewrites the exponent, and (c) uses the definition of $\rho_{L,n}$. Squaring \eqref{eq:controlled-diagonal-bound} and summing over $i$ gives
\begin{align}
\operatorname{tr}(\b K_m)
=
\sum_{i=1}^{n}
k_m(\b x_i,\b x_i)
\le
n\rho_{L,n}^2.
\label{eq:controlled-trace-bound}
\end{align}
Consequently, Jensen's inequality and the identity
$\eE\left[\bm\varepsilon^{\top}\b K_m\bm\varepsilon\right]=\operatorname{tr}(\b K_m)$ yield
\begin{align}
\eE\left[Z_m(\bm\varepsilon)\right]
&\stackrel{(a)}{\le}
\frac{r}{n}
\sqrt{
\eE\left[
\bm\varepsilon^{\top}
\b K_m
\bm\varepsilon
\right]
}
\stackrel{(b)}{=}
\frac{r}{n}
\sqrt{
\operatorname{tr}(\b K_m)
}
\stackrel{(c)}{\le}
\frac{r\rho_{L,n}}{\sqrt n}.
\label{eq:controlled-single-mean}
\end{align}
Here (a) applies Jensen's inequality to the concave square-root function in \eqref{eq:controlled-rkhs-duality}, (b) uses that the coordinates of $\bm\varepsilon$ are independent standard Gaussian variables, and (c) applies \eqref{eq:controlled-trace-bound}. We next verify a common Gaussian Lipschitz constant. For arbitrary
$\bm\varepsilon,\bm\varepsilon'\in\R^n$, the triangle inequality for the seminorm induced by the positive-semidefinite matrix $\b K_m$ gives
\begin{align}
\left|
Z_m(\bm\varepsilon)
-
Z_m(\bm\varepsilon')
\right|
&\stackrel{(a)}{\le}
\frac{r}{n}
\sqrt{
(\bm\varepsilon-\bm\varepsilon')^{\top}
\b K_m
(\bm\varepsilon-\bm\varepsilon')
}
\nonumber\\
&\stackrel{(b)}{\le}
\frac{r}{n}
\sqrt{\lambda_{\max}(\b K_m)}
\left\|\bm\varepsilon-\bm\varepsilon'\right\|_2
\nonumber\\
&\stackrel{(c)}{\le}
\frac{r\rho_{L,n}}{\sqrt n}
\left\|\bm\varepsilon-\bm\varepsilon'\right\|_2.
\label{eq:controlled-lipschitz-bound}
\end{align}
Here (a) uses \eqref{eq:controlled-rkhs-duality}, (b) uses the spectral bound for a positive-semidefinite matrix, and (c) uses
\begin{align}
\lambda_{\max}(\b K_m)
\le
\operatorname{tr}(\b K_m)
\le
n\rho_{L,n}^2.
\end{align}
Set
\begin{align}
\sigma
\coloneqq
\frac{r\rho_{L,n}}{\sqrt n}.
\end{align}
If $\sigma=0$, then \eqref{eq:controlled-trace-bound} gives $\b K_m=0$ for every $m\in[M]$, and consequently \eqref{eq:controlled-rkhs-duality} gives $Z_m\equiv0$. In this case \eqref{eq:controlled-empirical-gaussian-bound} is immediate. Hence, for the remainder of the proof, assume that $\sigma>0$. By the Gaussian concentration inequality for Lipschitz functions of a standard Gaussian vector \cite[Theorem~2.26]{wainwright2019highdimensional}, \eqref{eq:controlled-lipschitz-bound} implies that, for every $t\in\R$,
\begin{align}
\eE\left[
\exp\left(
 t\left(
 Z_m(\bm\varepsilon)
-
\eE[Z_m(\bm\varepsilon)]
 \right)
\right)
\right]
\le
\exp\left(
\frac{t^2\sigma^2}{2}
\right).
\label{eq:controlled-subgaussian-mgf}
\end{align}
If $M=1$, then
\begin{align}
\eE\left[
\max_{m\in[M]}
\left(
Z_m-\eE Z_m
\right)
\right]
=0.
\end{align}
Assume now that $M\ge2$. For every $t>0$,
\begin{align}
\MoveEqLeft
\eE\left[
\max_{m\in[M]}
\left(
Z_m-\eE Z_m
\right)
\right]
\nonumber\\
&\stackrel{(a)}{\le}
\frac{1}{t}
\eE\left[
\ln\left(
\sum_{m=1}^{M}
\exp\left(
 t\left(
 Z_m-\eE Z_m
 \right)
\right)
\right)
\right]
\nonumber\\
&\stackrel{(b)}{\le}
\frac{1}{t}
\ln\left(
\sum_{m=1}^{M}
\eE\left[
\exp\left(
 t\left(
 Z_m-\eE Z_m
 \right)
\right)
\right]
\right)
\nonumber\\
&\stackrel{(c)}{\le}
\frac{\ln M}{t}
+
\frac{t\sigma^2}{2}.
\label{eq:controlled-log-sum-exp}
\end{align}
Here (a) uses the log-sum-exp upper bound for a finite maximum, (b) applies Jensen's inequality to the concave logarithm, and (c) applies \eqref{eq:controlled-subgaussian-mgf}. Choosing
\begin{align}
t
=
\frac{\sqrt{2\ln M}}{\sigma}
\end{align}
in \eqref{eq:controlled-log-sum-exp} gives
\begin{align}
\eE\left[
\max_{m\in[M]}
\left(
Z_m-\eE Z_m
\right)
\right]
\le
\sigma\sqrt{2\ln M}.
\label{eq:controlled-centered-maximum}
\end{align}
The same inequality is valid for $M=1$ because both sides are then zero. Combining
\eqref{eq:controlled-union-max}, \eqref{eq:controlled-single-mean}, and \eqref{eq:controlled-centered-maximum}, we obtain
\begin{align}
\eG\left(\BLdictr\right)
&\stackrel{(a)}{\le}
\max_{m\in[M]}
\eE\left[Z_m\right]
+
\eE\left[
\max_{m\in[M]}
\left(
Z_m-\eE Z_m
\right)
\right]
\nonumber\\
&\stackrel{(b)}{\le}
\frac{r\rho_{L,n}}{\sqrt n}
+
\frac{r\rho_{L,n}}{\sqrt n}
\sqrt{2\ln M}
\nonumber\\
&\stackrel{(c)}{=}
\frac{r\rho_{L,n}}{\sqrt n}
\left(1+\sqrt{2\ln M}\right).
\end{align}
Here (a) adds and subtracts the individual means and uses
$\max_m(a_m+b_m)\le\max_m a_m+\max_m b_m$, (b) applies \eqref{eq:controlled-single-mean} and \eqref{eq:controlled-centered-maximum}, and (c) factors the common term. This is \eqref{eq:controlled-empirical-gaussian-bound}. Finally, take expectation with respect to $(X_i)_{i=1}^{n}\sim\nu_X^n$. By \eqref{eq:controlled-empirical-gaussian-bound},
\begin{align}
\G\left(\BLdictr\right)
&\stackrel{(a)}{\le}
\frac{r\kappa_M}{\sqrt n}
\DE\left[
\left(
\max_{i\in[n]}
\left\|X_i\right\|_2^2
\right)^{2^{-L}}
\right]
\nonumber\\
&\stackrel{(b)}{\le}
\frac{r\kappa_M}{\sqrt n}
\left(
\DE\left[
\max_{i\in[n]}
\left\|X_i\right\|_2^2
\right]
\right)^{2^{-L}}.
\end{align}
Here (a) uses the definition of $\rho_{L,n}$ and (b) applies Jensen's inequality because $2^{-L}\in(0,1]$ and the map $s\mapsto s^{2^{-L}}$ is concave on $[0,\infty)$. This proves \eqref{eq:gaussian_contraction_\BK{}}.
\end{noheadproof}\hfill\BlackBox\\

\subsection{Proof of \texorpdfstring{\Cref{thm:excess-risk-\BK{}}}{Theorem 4}}
\label{thm-excess-risk-\BK{}:proof}

\begin{noheadproof}
Set
\begin{align}
\rho_L
\coloneqq
B^{2^{-L}}.
\end{align}
We first record the deterministic envelope of the controlled radius-$r$ class. Let $f\in\BLdictr$. By \eqref{eq:controlled-BKL-ball}, there exists $m\in[M]$ such that
\begin{align}
f\in\mH_{\hat\mu_m}^{(L)},
\qquad
\left\|f\right\|_{\mH_{\hat\mu_m}^{(L)}}
\le
r.
\end{align}
Writing $k_m\coloneqq\hat k_m^{(L-1)}$ for the top-layer kernel of that ladder, \Cref{lem:RKHSfunctions_regularity} and \Cref{lem:kernel_boud_induction} yield, for every $\b x\in\mX$,
\begin{align}
\left|f(\b x)\right|
&\le
\left\|f\right\|_{\mH_{\hat\mu_m}^{(L)}}
\sqrt{k_m(\b x,\b x)}
\nonumber\\
&\le
r
\left\|\b x\right\|_2^{2^{-(L-1)}}
\le
rB^{2^{-L}}
=
r\rho_L.
\label{eq:controlled-class-envelope}
\end{align}
Thus \eqref{eq:controlled-loss-range} guarantees that every prediction occurring below lies in the interval on which \eqref{eq:controlled-loss-boundedness}--\eqref{eq:controlled-loss-lipschitz} hold.

\medskip
\noindent
\textbf{Step 1} (existence of the controlled RERM solution):
For each $m\in[M]$, consider
\begin{align}
\inf_{\substack{f\in\mH_{\hat\mu_m}^{(L)}\\
\left\|f\right\|_{\mH_{\hat\mu_m}^{(L)}}\le r}}
\left\{
\mR_n(f)
+
\lambda
\left\|f\right\|_{\mH_{\hat\mu_m}^{(L)}}
\right\}.
\end{align}
The closed radius-$r$ ball of the Hilbert space $\mH_{\hat\mu_m}^{(L)}$ is weakly compact. Point evaluation is a bounded linear functional and is therefore weakly continuous. Since the sample is finite and $\ell(\cdot,y_i)$ is continuous on the relevant prediction interval, $\mR_n$ is weakly continuous on this ball, while the RKHS norm is weakly lower semicontinuous. Hence the displayed problem attains its minimum for every $m$. Taking the minimum over the finite set $[M]$ produces a minimizing pair $(\hat m,\hat f)$. By the attainment of the finite minimum in \eqref{eq:controlled-BKL-complexity}, minimizing over such pairs is equivalent to minimizing
\begin{align}
\mR_n(f)+\lambda\CLdict(f)
\end{align}
over $f\in\BLdictr$. Therefore $\hat f$ is a solution of \eqref{eq:controlled-RERM}, proving nonemptiness of its solution set.

\medskip
\noindent
\textbf{Step 2} (uniform pairwise deviation on the controlled ball):
For a sample
\begin{align}
s
=
\left\{
(\b x_i,y_i)
\right\}_{i=1}^{n}
\in
(\mX\times\mY)^n,
\end{align}
write
\begin{align}
\mR_{n,s}(f)
\coloneqq
\frac1n
\sum_{i=1}^{n}
\ell\left(f(\b x_i),y_i\right),
\end{align}
and define
\begin{align}
D_s(f)
&\coloneqq
\mR(f)-\mR_{n,s}(f),\\
\Phi_r(s)
&\coloneqq
\sup_{f,g\in\BLdictr}
\left\{
D_s(f)-D_s(g)
\right\}.
\end{align}
Equivalently,
\begin{align}
\Phi_r(s)
=
\sup_{f\in\BLdictr}D_s(f)
-
\inf_{g\in\BLdictr}D_s(g).
\end{align}
Let $S=\{(X_i,Y_i)\}_{i=1}^{n}\sim\nu_{XY}^{n}$ and let $\bm\sigma=(\sigma_i)_{i=1}^{n}\sim\mm{Rad}^{n}$ be independent. Standard symmetrization \cite[Lemma~26.2]{shalev2014understanding} gives
\begin{align}
\E\left[\Phi_r(S)\right]
&\le
2\,
\E
\left[
\sup_{f,g\in\BLdictr}
\frac1n
\sum_{i=1}^{n}
\sigma_i
\left(
\ell(f(X_i),Y_i)
-
\ell(g(X_i),Y_i)
\right)
\right].
\label{eq:controlled-pairwise-symmetrization}
\end{align}
For fixed data, define
\begin{align}
\varphi_i(u)
\coloneqq
\ell(u,Y_i)-\ell(0,Y_i).
\end{align}
Then $\varphi_i(0)=0$ and $\varphi_i$ is $L_\ell$-Lipschitz on the relevant interval. The constant terms cancel in the pairwise difference, and
\begin{align}
&\sup_{f,g\in\BLdictr}
\frac1n
\sum_{i=1}^{n}
\sigma_i
\left(
\ell(f(X_i),Y_i)-\ell(g(X_i),Y_i)
\right)
\\
&\qquad\le
2
\sup_{f\in\BLdictr}
\left|
\frac1n
\sum_{i=1}^{n}
\sigma_i\varphi_i(f(X_i))
\right|.
\end{align}
The contraction inequality \cite[Eq.~(4.20)]{Ledoux1991}, followed by the Gaussian--Rademacher comparison inequality \cite[Exercise~(5.5)]{wainwright2019highdimensional}, therefore implies
\begin{align}
\E\left[\Phi_r(S)\right]
&\le
C_1L_\ell
\G\left(\BLdictr\right)
\nonumber\\
&\le
C_1L_\ell
\frac{r\kappa_M}{\sqrt n}
\left(
\E
\left[
\max_{i\in[n]}
\left\|X_i\right\|_2^2
\right]
\right)^{2^{-L}}
\nonumber\\
&\le
C_1L_\ell
\frac{r\kappa_M\rho_L}{\sqrt n},
\label{eq:controlled-expected-pairwise-deviation}
\end{align}
where $C_1>0$ is universal, the second inequality uses \Cref{thm:peeling_all_depth_\BK{}}, and the last one uses $\|X_i\|_2^2\le B$. We also used that $\BLdictr$ is symmetric, since it is a union of centered RKHS balls; hence the absolute Gaussian supremum coincides with the Gaussian complexity in \eqref{eq:Gn-bound}. We next apply bounded differences. Let $s^{(i)}$ be obtained from $s$ by replacing $(\b x_i,y_i)$ with $(\b x_i',y_i')$. The inequality
\begin{align}
\left|
\sup_u A(u)-\sup_u B(u)
\right|
\le
\sup_u
\left|A(u)-B(u)\right|
\end{align}
and the definition of $\Phi_r$ imply
\begin{align}
\left|
\Phi_r(s)-\Phi_r(s^{(i)})
\right|
&\le
\frac1n
\sup_{f,g\in\BLdictr}
\left|
\ell(f(\b x_i),y_i)-\ell(g(\b x_i),y_i)
\right|
\\
&\quad+
\frac1n
\sup_{f,g\in\BLdictr}
\left|
\ell(f(\b x_i'),y_i')-\ell(g(\b x_i'),y_i')
\right|.
\end{align}
By \eqref{eq:controlled-loss-lipschitz} and \eqref{eq:controlled-class-envelope},
\begin{align}
\left|
\ell(f(\b x),y)-\ell(g(\b x),y)
\right|
&\le
L_\ell
\left|f(\b x)-g(\b x)\right|
\\
&\le
2L_\ell r\rho_L.
\end{align}
Consequently,
\begin{align}
\left|
\Phi_r(s)-\Phi_r(s^{(i)})
\right|
\le
\frac{4L_\ell r\rho_L}{n}.
\label{eq:controlled-pairwise-bounded-difference}
\end{align}
McDiarmid's inequality \cite{mcdiarmid1989method}, together with \eqref{eq:controlled-expected-pairwise-deviation}, now gives a universal constant $C_2>0$ such that, with probability at least $1-\delta$,
\begin{align}
\Phi_r(S)
&\le
C_2L_\ell r\rho_L
\left[
\frac{\kappa_M}{\sqrt n}
+
\sqrt{
\frac{\ln(1/\delta)}{n}
}
\right]
\nonumber\\
&\le
r\Lambda_{n,M,L,\delta},
\label{eq:controlled-pairwise-high-probability}
\end{align}
after choosing the universal constant $C$ in \eqref{eq:controlled-statistical-scale} sufficiently large.

\medskip
\noindent
\textbf{Step 3} (oracle and excess-risk inequalities):
Let $\hat f_{\lambda,r,\BKLdict}$ be any solution of \eqref{eq:controlled-RERM}, and let $f\in\BLdictr$ be arbitrary. By optimality,
\begin{align}
\mR_n\left(\hat f_{\lambda,r,\BKLdict}\right)
+
\lambda\CLdict\left(\hat f_{\lambda,r,\BKLdict}\right)
\le
\mR_n(f)
+
\lambda\CLdict(f).
\label{eq:controlled-RERM-optimality}
\end{align}
Therefore,
\begin{align}
\mR\left(\hat f_{\lambda,r,\BKLdict}\right)-\mR(f)
&=
D_S\left(\hat f_{\lambda,r,\BKLdict}\right)-D_S(f)
\\
&\quad+
\mR_n\left(\hat f_{\lambda,r,\BKLdict}\right)-\mR_n(f)
\\
&\le
\Phi_r(S)
+
\lambda\CLdict(f)
-
\lambda\CLdict\left(\hat f_{\lambda,r,\BKLdict}\right)
\\
&\le
r\Lambda_{n,M,L,\delta}
+
\lambda\CLdict(f),
\end{align}
on the event in \eqref{eq:controlled-pairwise-high-probability}. Rearranging and taking the infimum over $f\in\BLdictr$ proves \eqref{eq:controlled-penalized-oracle}. Finally, every $f\in\BLdictr$ satisfies $\CLdict(f)\le r$. Choose an $\eta$-optimal sequence for $\inf_{f\in\BLdictr}\mR(f)$. Then
\begin{align}
\inf_{f\in\BLdictr}
\left\{
\mR(f)+\lambda\CLdict(f)
\right\}
\le
\inf_{f\in\BLdictr}
\mR(f)
+
\lambda r.
\end{align}
Combining this inequality with \eqref{eq:controlled-penalized-oracle} proves \eqref{eq:exessriskbound}. The special cases \eqref{eq:controlled-near-parametric-excess} and $\lambda=0$ follow immediately.
\end{noheadproof}\hfill \BlackBox\\

\section{Internal Lemmas}\label{app:auxi_lems}
\setcounter{lemma}{0}
\renewcommand{\thelemma}{C\arabic{lemma}}

\begin{table}
\centering
\small
\begin{tabular}{@{}p{0.18\textwidth}p{0.61\textwidth}p{0.14\textwidth}@{}}
\toprule
Result & Content & Page\\
\midrule
\Cref{lem:RKHSfunctions_regularity}& pointwise evaluation, Lipschitz continuity & page~\pageref{lem:RKHSfunctions_regularity}\\
\Cref{lem:kernel_boud_induction}& nested nested pointwise evaluation and Lipschitz continuity & page~\pageref{lem:kernel_boud_induction}\\
\Cref{lem:BK}& well-defined Brownian integral kernels & page~\pageref{lem:BK}\\
\Cref{lem:kernel_transformation}& kernel pullback and restriction & page~\pageref{lem:kernel_transformation}\\
\Cref{lem:measure_constraint_sphere}& measure support on the sphere and kernel preservation & page~\pageref{lem:measure_constraint_sphere}\\
\Cref{lem:measures_inequality}& two-ladder amalgamation & page~\pageref{lem:measures_inequality}\\
\Cref{lemma:TVconvergence}& countable canonical amalgamation & page~\pageref{lemma:TVconvergence}\\
\Cref{lem:line-segment-implies-trace}& transverse-segment sufficient condition for strict depth & page~\pageref{lem:line-segment-implies-trace}\\
\Cref{lem:brownian-rkhs}& Brownian RKHS characterization & page~\pageref{lem:brownian-rkhs}\\
\Cref{lem:lsc-CL}& closed adaptive balls and lower semicontinuity of $\CL$ & page~\pageref{lem:lsc-CL}\\
\Cref{lem:HL_precompact}& relative compactness of adaptive quasi-norm balls & page~\pageref{lem:HL_precompact}\\
\Cref{lem:loss-bounds}& boundedness and common loss bounds & page~\pageref{lem:loss-bounds}\\
\bottomrule
\end{tabular}
\caption{Auxiliary results in the dependency chain; see Fig.~\ref{fig:dependency_graph}. For the main results, see Table~\ref{caption:main-results}.}
\label{caption:auxiliary-results}
\end{table}

In this section, we present our own lemmas used in \Cref{app}.

\begin{lemma}[pointwise evaluation, Lipschitz continuity]\label{lem:RKHSfunctions_regularity}
Let $\mH_k$ be an RKHS on $\mX$ with kernel function $k$. Then, for each $f \in \mH_k$  
 \begin{align}
     |f(\b x)| \leq \|f\|_{\mH_k}[k(\b x,\b x)]^{1/2}, \quad \forall \b x \in \mX, \label{eq:|f(x)|-bound}\\
          |f(\b x) - f(\b x')| \leq \|f\|_{\mH_k}d_{k}(\b x,\b x'), \quad \forall \b x,\b x' \in \mX,\label{eq:|f(x)-f(x)|-bound}
 \end{align}
where   
\begin{align}
d_{k}(\b x,\b x') &\coloneqq \left[k(\b x,\b x) -2k(\b x,\b x') + k(\b x',\b x')\right]^{1/2}. \label{eq:def:d_k}
\end{align}
\end{lemma}

\begin{proof}~

\begin{itemize}
    \item \tb{Part-1 (pointwise evaluation)}: \eqref{eq:|f(x)|-bound} is a known result \cite[Lemma~4.23]{steinwart08support}.
    \item \tb{Part-2 (Lipschitz continuity)}: One has
        \begin{align}
        \left|f(\b x) - f(\b x')\right| & \stackrel{(a)}{=} \left|\langle f, k(\b x, \cdot) - k(\b x', \cdot) \rangle_{\mH_k}\right| \stackrel{(b)}{\le} \|f\|_{\mH_k} \|k(\b x, \cdot) - k(\b x', \cdot)\|_{\mH_k}\\
        &\stackrel{(c)}{=} \|f\|_{\mH_k} \underbrace{\sqrt{k(\b x, \b x) - 2k(\b x, \b x') + k(\b x', \b x')}}_{=d_k(\b x,\b x')}.
        \end{align}
        (a) follows from the reproducing property and the linearity of the inner product, (b) is implied by the CBS inequality, (c) comes from the fact that in a Hilbert space the norm is induced by the inner product, the linearity of the inner product, the reproducing property, the fact that $k$ is a symmetric function, and the definition of $d_k$. 
 \end{itemize}
\end{proof}

\begin{lemma}[nested pointwise evaluation and Lipschitz continuity]\label{lem:kernel_boud_induction}
Let $L\ge 2$ and $\mX \subset \R^d$ compact. Suppose that $\sSmu$ and $\lmu$ (in \Cref{def:\BK{}def}) satisfy  Assumption~\ref{assumption:normalizing}; let the corresponding canonical sets and measures (according to \Cref{rem:BKL_def}) be $\sSmuh \coloneq \left(S_{\hat{\mu}}^{(l)}\right)_{l=1}^{L-1} \coloneq \left(\SHlmu\right)_{l=1}^{L-1}$ and $\lmuh\coloneq \left(\hat{\mu}^{(l)}\right)_{l=1}^{L-1}\in \times_{l=1}^{L-1}\mP_{1,1}\left(\SHlmu\right)$, respectively. Let 
 \begin{align}
 \hat k^{(l)} &= k\left[\SHlmu,\hat \mu^{(l)}\right],\\
 d_{\hat k^{(l)}}(\b x,\b x') &\coloneq  \left[\hat{k}^{(l)}(\b x,\b x) -2\hat{k}^{(l)}(\b x,\b x') + \hat{k}^{(l)}(\b x',\b x')\right]^{1/2},
 \end{align}
 with $l \in [L-1]$ and $\b x, \b x' \in \mX$. Let $k^{(0)}(\b x,\b x') = \b x^\T \b x'$. Then, 
 \begin{align}
     \left[\hat k^{(l)}(\b x,\b x)\right]^{1/2} &\leq 
     \sqrt[\leftroot{-2} \uproot{0} 2^{l}]{ \|\b x\|_2}, \quad \forall \b x \in \mX, \label{eq:nested:point-eval}\\
    d_{\hat k^{(l)}}(\b x,\b x') &\leq  \sqrt[\leftroot{-2} \uproot{0} 2^{l}]{ \|\b x-\b x'\|_2}, \quad \forall \b x,\b x' \in \mX, \label{eq:nested:Lipschitz}
 \end{align}
 where $l\in \{0,1,\ldots, L-1\}$.
\end{lemma}

\begin{proof}\label{lem:kernel_boud_inductionproof}
Let
\begin{align}
\LHmuh
&\coloneqq
\left(\Hlhmul\right)_{l=1}^{L},
\qquad
\Hlhmu
\coloneqq
\mH_{\hat k^{(l)}},
\quad l\in[L-1],
\end{align}
with $\mH_{\hat\mu}^{(1)}\coloneqq\mH_{\mu}^{(1)}$.
We prove both estimates simultaneously by induction on $l$.

\smallskip
\noindent
\emph{Base case $l=0$.}
Since $k^{(0)}(\b x,\b x')=\b x^\top\b x'$, one has
\begin{align}
\left[k^{(0)}(\b x,\b x)\right]^{1/2}
&=
\|\b x\|_2,\\
d_{k^{(0)}}(\b x,\b x')
&=
\left(
\|\b x\|_2^2
-2\b x^\top\b x'
+\|\b x'\|_2^2
\right)^{1/2}
=
\|\b x-\b x'\|_2.
\end{align}
Thus \eqref{eq:nested:point-eval} and \eqref{eq:nested:Lipschitz} hold for $l=0$.

\smallskip
\noindent
\emph{Induction step.}
Fix $l\in\{0,\ldots,L-2\}$ and assume that both estimates hold at level $l$.
For every $\b x\in\mX$, the canonical construction and the identity
$\kB(a,a)=|a|$ give
\begin{align}
\hat k^{(l+1)}(\b x,\b x)
&\stackrel{(a)}{=}
\int_{\SHlpmu}
\kB\left(u(\b x),u(\b x)\right)
\,\d\hat\mu^{(l+1)}(u)\\
&\stackrel{(b)}{=}
\int_{\SHlpmu}
|u(\b x)|
\,\d\hat\mu^{(l+1)}(u)\\
&\stackrel{(c)}{\le}
\int_{\SHlpmu}
\|u\|_{\Hlpmuh}
\left[\hat k^{(l)}(\b x,\b x)\right]^{1/2}
\,\d\hat\mu^{(l+1)}(u)\\
&\stackrel{(d)}{=}
\left[\hat k^{(l)}(\b x,\b x)\right]^{1/2}
\int_{\SHlpmu}
\|u\|_{\Hlpmuh}
\,\d\hat\mu^{(l+1)}(u)\\
&\stackrel{(e)}{=}
\left[\hat k^{(l)}(\b x,\b x)\right]^{1/2}\\
&\stackrel{(f)}{\le}
\|\b x\|_2^{2^{-l}}.
\end{align}
Here (a) is the definition of the next canonical integral kernel, (b) uses
$\kB(a,a)=|a|$, (c) applies the RKHS pointwise estimate
\eqref{eq:|f(x)|-bound} in $\Hlpmuh=\mH_{\hat k^{(l)}}$, (d) factors out a term independent of $u$, (e) uses
$\hat\mu^{(l+1)}\in\mP_{1,1}(\SHlpmu)$, and (f) is the induction hypothesis.
Taking square roots of the nonnegative quantities yields
\begin{align}
\left[\hat k^{(l+1)}(\b x,\b x)\right]^{1/2}
&\le
\|\b x\|_2^{2^{-(l+1)}}.
\end{align}
For the canonical kernel metric, linearity of integration and the Brownian identity
\begin{align}
\kB(a,a)-2\kB(a,b)+\kB(b,b)
&=
|a-b|
\end{align}
give, for every $\b x,\b x'\in\mX$,
\begin{align}
\left[d_{\hat k^{(l+1)}}(\b x,\b x')\right]^2
&\stackrel{(a)}{=}
\int_{\SHlpmu}
|u(\b x)-u(\b x')|
\,\d\hat\mu^{(l+1)}(u)\\
&\stackrel{(b)}{\le}
\int_{\SHlpmu}
\|u\|_{\Hlpmuh}
 d_{\hat k^{(l)}}(\b x,\b x')
\,\d\hat\mu^{(l+1)}(u)\\
&\stackrel{(c)}{=}
 d_{\hat k^{(l)}}(\b x,\b x')\\
&\stackrel{(d)}{\le}
\|\b x-\b x'\|_2^{2^{-l}}.
\end{align}
Here (a) follows from the definition of the kernel metric and the preceding Brownian identity, (b) applies the RKHS difference estimate
\eqref{eq:|f(x)-f(x)|-bound} in $\Hlpmuh$, (c) again uses the unit first moment of $\hat\mu^{(l+1)}$, and (d) is the induction hypothesis.
Taking square roots gives
\begin{align}
 d_{\hat k^{(l+1)}}(\b x,\b x')
&\le
\|\b x-\b x'\|_2^{2^{-(l+1)}}.
\end{align}
This proves both assertions at level $l+1$ and completes the induction.
\end{proof}

\begin{lemma}[well-defined Brownian integral kernels]
\label{lem:BK}
Assume that $\mX\subset\R^d$ is compact, let $k^{(0)}(\b x,\b x')=\b x^\top\b x'$, and let $\Honemu=\mH_{k^{(0)}}$ as in \Cref{def:\BK{}def}. For $l\ge1$, suppose that $\Hlmu$ has already been defined, let $\Slmu\subseteq\Hlmu$ be measurable, and let $\mu^{(l)}\in\mP_1(\Slmu)$. For $u\in\Slmu$, define
\begin{align}
k_u(\b x,\b x')
\coloneqq
\kB\left(u(\b x),u(\b x')\right).
\end{align}
Then, for every $\b x\in\mX$,
\begin{align}
\int_{\Slmu}
k_u(\b x,\b x)
\,\d\mu^{(l)}(u)
<
\infty.
\end{align}
Consequently, $k^{(l)}=k[\Slmu,\mu^{(l)}]$ and $\Hlpmu=\mH_{k^{(l)}}$ are well-defined.
\end{lemma}

\begin{proof}\label{lem:BK_proof}
Since $\kB(a,a)=|a|$, it is enough to prove
\begin{align}
\int_{\Slmu}|u(\b x)|\,\d\mu^{(l)}(u)
<
\infty,
\qquad
\b x\in\mX.
\label{eq:i-kernel-reduction-to-u}
\end{align}
We proceed recursively.

\medskip
\noindent
\emph{Base layer $l=1$.}
For $u\in\Honemu$, the RKHS pointwise estimate and the identity $k^{(0)}(\b x,\b x)=\|\b x\|_2^2$ give
\begin{align}
|u(\b x)|
&\stackrel{(a)}{\le}
\|u\|_{\Honemu}
\left[k^{(0)}(\b x,\b x)\right]^{1/2}
\stackrel{(b)}{=}
\|u\|_{\Honemu}\|\b x\|_2.
\end{align}
Here (a) is \eqref{eq:|f(x)|-bound}, and (b) is the definition of the linear kernel. Therefore,
\begin{align}
\int_{\Sonemu}|u(\b x)|\,\d\mu^{(1)}(u)
&\le
\|\b x\|_2
\int_{\Sonemu}
\|u\|_{\Honemu}
\,\d\mu^{(1)}(u)
<
\infty,
\end{align}
because $\mu^{(1)}\in\mP_1(\Sonemu)$.

\medskip
\noindent
\emph{Recursive step.}
Assume that $k^{(l)}$ and $\Hlpmu=\mH_{k^{(l)}}$ are well-defined. For $u\in\Slpmu$, \eqref{eq:|f(x)|-bound} gives
\begin{align}
|u(\b x)|
\le
\|u\|_{\Hlpmu}
\left[k^{(l)}(\b x,\b x)\right]^{1/2}.
\end{align}
Hence
\begin{align}
\int_{\Slpmu}|u(\b x)|\,\d\mu^{(l+1)}(u)
&\le
\left[k^{(l)}(\b x,\b x)\right]^{1/2}
\int_{\Slpmu}
\|u\|_{\Hlpmu}
\,\d\mu^{(l+1)}(u)
<
\infty.
\end{align}
The first factor is finite because $k^{(l)}$ is a kernel, and the second is finite because $\mu^{(l+1)}\in\mP_1(\Slpmu)$. This proves \eqref{eq:i-kernel-reduction-to-u} at the next layer and completes the recursion.
\end{proof}

\begin{lemma}[kernel pullback and restriction]
\label{lem:kernel_transformation}
Let $\mZ_0$ be a set, let $u:\mX\to\mZ_0$, and put $\mZ\coloneqq\im(u)$. Let $k:\mZ_0\times\mZ_0\to\R$ be a kernel, denote its restriction to $\mZ\times\mZ$ by $k_{\mZ}$, and define
\begin{align}
k_u(\b x,\b x')
\coloneqq
k\left(u(\b x),u(\b x')\right),
\qquad
\b x,\b x'\in\mX.
\end{align}
Then the following statements hold.
\begin{enumerate}[(i)]
\item The pullback map
\begin{align}
U_u:\mH_{k_{\mZ}}\to\mH_{k_u},
\qquad
U_u g
\coloneqq
g\circ u,
\end{align}
is a surjective linear isometry. In particular,
\begin{align}
\mH_{k_u}
=
\left\{g\circ u:g\in\mH_{k_{\mZ}}\right\},
\qquad
\|g\circ u\|_{\mH_{k_u}}
=
\|g\|_{\mH_{k_{\mZ}}}.
\end{align}\label{kernel_transformation:1}
\item For every $g\in\mH_k$,
\begin{align}
g\circ u
\in
\mH_{k_u},
\qquad
\|g\circ u\|_{\mH_{k_u}}
\le
\|g\|_{\mH_k}.
\end{align}\label{kernel_transformation:2}
\item For every $h\in\mH_{k_u}$, there exists a minimum-norm extension $g_h\in\mH_k$ satisfying
\begin{align}
h
=
g_h\circ u,
\qquad
\|g_h\|_{\mH_k}
=
\|h\|_{\mH_{k_u}}.
\end{align}\label{kernel_transformation:3}
\item If $k=\kB$ and $k_u(\b x,\b x')=\kB(u(\b x),u(\b x'))$, then, for every $\alpha\ge0$ and every $\b x\in\mX$,
\begin{align}
\alpha k_u(\cdot,\b x)
=
k_{\alpha u}(\cdot,\b x).
\end{align}\label{kernel_transformation:4}
\end{enumerate}
\end{lemma}

\begin{proof}

\textbf{Step 1} (restriction to $\mZ$):
Define the closed subspaces
\begin{align}
\mM_{\mZ}
&\coloneqq
\overline{\Span\left\{k(\cdot,z):z\in\mZ\right\}}^{\mH_k},\\
\mN_{\mZ}
&\coloneqq
\left\{g\in\mH_k:g(z)=0\text{ for all }z\in\mZ\right\}.
\end{align}
We first show that
\begin{align}
\mN_{\mZ}
=
\mM_{\mZ}^{\perp}.
\label{eq:kernel-restriction-orthogonal-decomposition}
\end{align}
Indeed, if $g\in\mN_{\mZ}$, then for every $z\in\mZ$,
\begin{align}
\langle g,k(\cdot,z)\rangle_{\mH_k}
\stackrel{(a)}{=}
g(z)
\stackrel{(b)}{=}0,
\end{align}
where (a) is the reproducing property and (b) is the definition of $\mN_{\mZ}$. Thus $g\perp\mM_{\mZ}$. Conversely, if $g\perp\mM_{\mZ}$, then for every $z\in\mZ$,
\begin{align}
g(z)
=
\langle g,k(\cdot,z)\rangle_{\mH_k}
=
0,
\end{align}
so $g\in\mN_{\mZ}$. This proves \eqref{eq:kernel-restriction-orthogonal-decomposition}. On the algebraic span of the kernel sections with centers in $\mZ$, define the restriction map
\begin{align}
R_{\mZ}
\left(
\sum_{i=1}^{n}\alpha_i k(\cdot,z_i)
\right)
\coloneqq
\sum_{i=1}^{n}\alpha_i k_{\mZ}(\cdot,z_i).
\end{align}
For every such finite combination,
\begin{align}
\left\|
\sum_{i=1}^{n}\alpha_i k(\cdot,z_i)
\right\|_{\mH_k}^2
&\stackrel{(a)}{=}
\sum_{i,j=1}^{n}
\alpha_i\alpha_j k(z_i,z_j)\\
&\stackrel{(b)}{=}
\sum_{i,j=1}^{n}
\alpha_i\alpha_j k_{\mZ}(z_i,z_j)\\
&\stackrel{(c)}{=}
\left\|
\sum_{i=1}^{n}\alpha_i k_{\mZ}(\cdot,z_i)
\right\|_{\mH_{k_{\mZ}}}^2.
\end{align}
Here, (a) and (c) follow from the reproducing property, and (b) follows from the definition of $k_{\mZ}$. Hence $R_{\mZ}$ extends uniquely to an isometry from $\mM_{\mZ}$ into $\mH_{k_{\mZ}}$. Its range contains the dense span of the kernel sections of $\mH_{k_{\mZ}}$ and is closed because $\mM_{\mZ}$ is complete. Therefore,
\begin{align}
R_{\mZ}:\mM_{\mZ}\to\mH_{k_{\mZ}}
\end{align}
is a surjective linear isometry. Let $P_{\mZ}:\mH_k\to\mM_{\mZ}$ be the orthogonal projection. By \eqref{eq:kernel-restriction-orthogonal-decomposition}, $g-P_{\mZ}g\in\mN_{\mZ}$, and therefore
\begin{align}
g|_{\mZ}
=
(P_{\mZ}g)|_{\mZ}.
\end{align}
Consequently,
\begin{align}
\|g|_{\mZ}\|_{\mH_{k_{\mZ}}}
&\stackrel{(a)}{=}
\|P_{\mZ}g\|_{\mH_k}
\stackrel{(b)}{\le}
\|g\|_{\mH_k},
\label{eq:kernel-restriction-contraction}
\end{align}
where (a) uses the isometry $R_{\mZ}$ and (b) uses contractivity of an orthogonal projection. Conversely, for every $q\in\mH_{k_{\mZ}}$, the function
\begin{align}
\widetilde q
\coloneqq
R_{\mZ}^{-1}q
\in
\mM_{\mZ}
\end{align}
is an extension of $q$ and satisfies
\begin{align}
\|\widetilde q\|_{\mH_k}
=
\|q\|_{\mH_{k_{\mZ}}}.
\label{eq:kernel-restriction-minimum-extension}
\end{align}
It is minimum norm because every other extension differs from $\widetilde q$ by an element of $\mN_{\mZ}=\mM_{\mZ}^{\perp}$.

\medskip
\noindent
\textbf{Step 2} (pullback from $\mZ$ to $\mX$):
For a finite combination of restricted kernel sections, define
\begin{align}
U_u
\left(
\sum_{i=1}^{n}\alpha_i k_{\mZ}(\cdot,z_i)
\right)
\coloneqq
\sum_{i=1}^{n}\alpha_i k_u(\cdot,\b x_i),
\end{align}
where $\b x_i\in\mX$ is chosen so that $u(\b x_i)=z_i$, which is possible because $\mZ=\im(u)$. This definition is independent of the chosen preimages because its squared norm is determined by the Gram matrix:
\begin{align}
\left\|
\sum_{i=1}^{n}\alpha_i k_u(\cdot,\b x_i)
\right\|_{\mH_{k_u}}^2
&\stackrel{(a)}{=}
\sum_{i,j=1}^{n}\alpha_i\alpha_j k_u(\b x_i,\b x_j)\\
&\stackrel{(b)}{=}
\sum_{i,j=1}^{n}\alpha_i\alpha_j k_{\mZ}(z_i,z_j)\\
&\stackrel{(c)}{=}
\left\|
\sum_{i=1}^{n}\alpha_i k_{\mZ}(\cdot,z_i)
\right\|_{\mH_{k_{\mZ}}}^2.
\end{align}
Here, (a) and (c) use the reproducing property and (b) uses $z_i=u(\b x_i)$. Thus $U_u$ extends to a linear isometry from $\mH_{k_{\mZ}}$ into $\mH_{k_u}$. Its range contains every kernel section $k_u(\cdot,\b x)$ and is closed, hence it equals $\mH_{k_u}$. Moreover, on the dense span of kernel sections,
\begin{align}
(U_uq)(\b x)
=
q(u(\b x)),
\end{align}
and the identity extends to every $q\in\mH_{k_{\mZ}}$ by continuity of point evaluation. This proves item~\ref{kernel_transformation:1}. Let $g\in\mH_k$. By \eqref{eq:kernel-restriction-contraction}, $q\coloneqq g|_{\mZ}$ belongs to $\mH_{k_{\mZ}}$ and $\|q\|_{\mH_{k_{\mZ}}}\le\|g\|_{\mH_k}$. Applying item~\ref{kernel_transformation:1},
\begin{align}
g\circ u
=
q\circ u
\in
\mH_{k_u},
\qquad
\|g\circ u\|_{\mH_{k_u}}
=
\|q\|_{\mH_{k_{\mZ}}}
\le
\|g\|_{\mH_k},
\end{align}
which proves item~\ref{kernel_transformation:2}. Let $h\in\mH_{k_u}$. By item (i), there exists a unique $q\in\mH_{k_{\mZ}}$ such that $h=q\circ u$ and $\|q\|_{\mH_{k_{\mZ}}}=\|h\|_{\mH_{k_u}}$. Let $g_h\coloneqq R_{\mZ}^{-1}q$. Then \eqref{eq:kernel-restriction-minimum-extension} gives
\begin{align}
h
=
g_h\circ u,
\qquad
\|g_h\|_{\mH_k}
=
\|q\|_{\mH_{k_{\mZ}}}
=
\|h\|_{\mH_{k_u}},
\end{align}
which proves item~\ref{kernel_transformation:3}.

\medskip
\noindent
\textbf{Step 3} (Brownian homogeneity):
Let $k=\kB$, $\alpha\ge0$, and fix $\b x\in\mX$. For every $\b x'\in\mX$,
\begin{align}
\alpha k_u(\b x',\b x)
&\stackrel{(a)}{=}
\alpha\kB\left(u(\b x'),u(\b x)\right)\\
&\stackrel{(b)}{=}
\kB\left(\alpha u(\b x'),\alpha u(\b x)\right)\\
&\stackrel{(c)}{=}
k_{\alpha u}(\b x',\b x),
\end{align}
where (a) is the definition of $k_u$, (b) uses the nonnegative one-homogeneity of the Brownian kernel, and (c) is the definition of $k_{\alpha u}$. Hence $\alpha k_u(\cdot,\b x)=k_{\alpha u}(\cdot,\b x)$, proving item~\ref{kernel_transformation:4}.
\end{proof}

\begin{lemma}[measure support on the sphere and kernel preservation]
\label{lem:measure_constraint_sphere}
Let the $L$-level \BK{} $\LHmu$ be defined according to Definition~\ref{def:\BK{}def} under Assumption~\ref{assumption:normalizing}, yielding
\begin{align}
k^{(l)} = k\left[\Hlmu,\,\mu^{(l)}\right],
\qquad
\Hlpmu = \mH_{k^{(l)}},
\quad \text{for all } l \in [L-1].
\end{align}
Then, for each $l \in [L-1]$, there exists a probability measure
$\hat{\mu}^{(l)} \in \Psph\left(\SHlmu\right)$ such that
\begin{align}
\hat{k}^{(l)} = k\left[\SHlmu,\,\hat{\mu}^{(l)}\right],
\qquad
\mH_{\hat{k}^{(l)}} = \Hlpmu,
\end{align}
and
\begin{align}
\hat{k}^{(l)} = k^{(l)} \quad \text{pointwise on } \mX \times \mX,
\qquad
\Hlpmu = \Hlpmuh.
\end{align}
\end{lemma}

\begin{proof}
Fix $l \in [L-1]$.

\medskip
\noindent
\textbf{Step 1} (weighted radial projection):
Define
\begin{align}
N : \Hlmu \setminus \left\{0\right\}
&\longrightarrow \SHlmu,
&
N(u)
&\coloneqq
\frac{u}{\left\|u\right\|_{\Hlmu}}.
\end{align}
The map $N$ is continuous on $\Hlmu \setminus \left\{0\right\}$ and is therefore Borel measurable. For every Borel set $A \subseteq \SHlmu$, define
\begin{align}
\hat{\mu}^{(l)}(A)
\coloneqq
\int_{\Hlmu \setminus \left\{0\right\}}
\I_A\left(N(u)\right)
\left\|u\right\|_{\Hlmu}
\,\d\mu^{(l)}(u).
\end{align}
Since $u \mapsto \left\|u\right\|_{\Hlmu}$ is nonnegative and integrable under $\mu^{(l)}$, this formula defines a finite Borel measure on $\SHlmu$.

\medskip
\noindent
\textbf{Step 2} (probability and first-moment normalization):
By Assumption~\ref{assumption:normalizing}, one has $\mu^{(l)} \in \mP_{1,1}\left(\Hlmu\right)$. Hence,
\begin{align}
\hat{\mu}^{(l)}\left(\SHlmu\right)
&\stackrel{(a)}{=}
\int_{\Hlmu \setminus \left\{0\right\}}
\left\|u\right\|_{\Hlmu}
\,\d\mu^{(l)}(u)
\stackrel{(b)}{=}
\int_{\Hlmu}
\left\|u\right\|_{\Hlmu}
\,\d\mu^{(l)}(u)
\stackrel{(c)}{=}
1.
\end{align}
Here, (a) follows from the definition of $\hat{\mu}^{(l)}$ with $A=\SHlmu$, (b) uses $\left\|0\right\|_{\Hlmu}=0$, and (c) follows from the unit first moment of $\mu^{(l)}$. Therefore, $\hat{\mu}^{(l)}$ is a probability measure. Moreover, because $\left\|v\right\|_{\Hlmu}=1$ for every $v\in\SHlmu$,
\begin{align}
\left\|\hat{\mu}^{(l)}\right\|_{\SHlmu,1}
&\stackrel{(a)}{=}
\int_{\SHlmu}
\left\|v\right\|_{\Hlmu}
\,\d\hat{\mu}^{(l)}(v)
\stackrel{(b)}{=}
\hat{\mu}^{(l)}\left(\SHlmu\right)
\stackrel{(c)}{=}
1.
\end{align}
In (a), we use the definition of the first moment, (b) uses the fact that $\hat{\mu}^{(l)}$ is supported on $\SHlmu$, and (c) uses the preceding total-mass computation. Consequently,
\begin{align}
\hat{\mu}^{(l)}
\in
\Psph\left(\SHlmu\right).
\end{align}

\medskip
\noindent
\textbf{Step 3} (exact preservation of the kernel):
Fix $\b x,\b x'\in\mX$. Since the integrand vanishes at $u=0$, one has
\begin{align}
k^{(l)}\left(\b x,\b x'\right)
&\stackrel{(a)}{=}
\int_{\Hlmu}
\kB\left(u(\b x),u(\b x')\right)
\,\d\mu^{(l)}(u)\\
&\stackrel{(b)}{=}
\int_{\Hlmu \setminus \left\{0\right\}}
\kB\left(u(\b x),u(\b x')\right)
\,\d\mu^{(l)}(u)\\
&\stackrel{(c)}{=}
\int_{\Hlmu \setminus \left\{0\right\}}
\kB\left(
\left\|u\right\|_{\Hlmu}N(u)(\b x),
\left\|u\right\|_{\Hlmu}N(u)(\b x')
\right)
\,\d\mu^{(l)}(u)\\
&\stackrel{(d)}{=}
\int_{\Hlmu \setminus \left\{0\right\}}
\left\|u\right\|_{\Hlmu}
\kB\left(N(u)(\b x),N(u)(\b x')\right)
\,\d\mu^{(l)}(u)\\
&\stackrel{(e)}{=}
\int_{\SHlmu}
\kB\left(v(\b x),v(\b x')\right)
\,\d\hat{\mu}^{(l)}(v)\\
&\stackrel{(f)}{=}
\hat{k}^{(l)}\left(\b x,\b x'\right).
\end{align}
Equality (a) follows from the definition of $k^{(l)}$. Equality (b) uses $\kB\left(0,0\right)=0$. Equality (c) follows from the definition of $N$, equality (d) uses the nonnegative $1$-homogeneity of the Brownian kernel, equality (e) follows from the definition of $\hat{\mu}^{(l)}$, and equality (f) is the definition of $\hat{k}^{(l)}$. Since $\b x,\b x'\in\mX$ were arbitrary,
\begin{align}
\hat{k}^{(l)} = k^{(l)}
\qquad
\text{pointwise on } \mX\times\mX.
\end{align}

\medskip
\noindent
\textbf{Step 4} (preservation of the RKHS):
Because $k^{(l)}$ and $\hat{k}^{(l)}$ coincide pointwise, their associated RKHSs coincide isometrically. Therefore,
\begin{align}
\mH_{\hat{k}^{(l)}}
=
\mH_{k^{(l)}}
=
\Hlpmu,
\qquad
\Hlpmuh
=
\Hlpmu.
\end{align}
This completes the proof.
\end{proof}

\begin{lemma}[two-ladder amalgamation]
\label{lem:measures_inequality}
Let $L\ge2$.
\begin{enumerate}[(i)]
\item \label{lem:local-kernel-domination}
\emph{Local domination.}
Let $\mH\subseteq\R^{\mX}$ be a Hilbert space of functions on $\mX$, and let $\mu_a,\mu_b\in\mP_1(\mH)$. Define
\begin{align}
k_s(\b x,\b x')
\coloneqq
\int_{\mH}
\kB\left(u(\b x),u(\b x')\right)
\,\d\mu_s(u),
\qquad
s\in\{a,b\}.
\end{align}
Assume that $\mu_b\ll\mu_a$ and that
\begin{align}
0
\le
\frac{\d\mu_b}{\d\mu_a}(u)
\le
M
\qquad
\mu_a\text{-a.e.}
\end{align}
for some $M>0$. Then
\begin{align}
\mH_{k_b}
\subseteq
\mH_{k_a},
\qquad
\|f\|_{\mH_{k_a}}
\le
\sqrt{M}\,\|f\|_{\mH_{k_b}},
\quad
f\in\mH_{k_b}.
\end{align}

\item \label{lem:finite-moment-canonicalization}
\emph{Canonicalization of a finite-moment ladder.}
Let $\LHmu=\left(\Hlmu\right)_{l=1}^L$ be an $L$-level \BK{} as in Definition~\ref{def:\BK{}def}, with layer measures $\mu^{(l)}\in\mP_1\left(\Slmu\right)$. Assume that
\begin{align}
m_l
\coloneqq
\int_{\Slmu}
\|u\|_{\Hlmu}
\,\d\mu^{(l)}(u)
\in
(0,\infty),
\qquad
l\in[L-1].
\end{align}
Define
\begin{align}
q_1\coloneqq1,
\qquad
q_{l+1}\coloneqq\sqrt{q_lm_l},
\quad
l\in[L-1].
\end{align}
Then there exists a canonical ladder $\LHmuh\in\LHh$ such that, for every $l\in[L]$,
\begin{align}
\Hlmuh
&=
\Hlmu
\quad\text{as sets of functions},
&
\|f\|_{\Hlmuh}
&=
q_l\|f\|_{\Hlmu},
\quad
f\in\Hlmu.
\label{eq:finite-moment-canonicalization}
\end{align}
In particular,
\begin{align}
q_L
=
\prod_{l=1}^{L-1}
m_l^{2^{-(L-l)}}.
\end{align}

\item \label{lem:two-ladder-amalgamation}
\emph{Two-ladder canonical amalgamation.}
Let $\LH_{\hat\mu_b},\LH_{\hat\mu_c}\in\LHh$ be canonical $L$-level \BK{}s. Then there exists a canonical ladder $\LH_{\hat\mu_a}\in\LHh$ such that
\begin{align}
\mH_{\hat\mu_b}^{(L)}
\cup
\mH_{\hat\mu_c}^{(L)}
\subseteq
\mH_{\hat\mu_a}^{(L)}.
\end{align}
Moreover, with
\begin{align}
\kappa_L
\coloneqq
\frac{2}{\sqrt[\leftroot{-2}\uproot{0}2^L]{2}},
\end{align}
one has, for all $f_b\in\mH_{\hat\mu_b}^{(L)}$ and $f_c\in\mH_{\hat\mu_c}^{(L)}$,
\begin{align}
\|f_b+f_c\|_{\mH_{\hat\mu_a}^{(L)}}
\le
\kappa_L
\left(
\|f_b\|_{\mH_{\hat\mu_b}^{(L)}}
+
\|f_c\|_{\mH_{\hat\mu_c}^{(L)}}
\right).
\label{eq:two-ladder-amalgamation}
\end{align}
\end{enumerate}
\end{lemma}

\begin{proof}~

\emph{Proof of \ref{lem:local-kernel-domination}.}
For $u\in\mH$, let
\begin{align}
k_u(\b x,\b x')
\coloneqq
\kB\left(u(\b x),u(\b x')\right).
\end{align}
Fix $f\in\mH_{k_b}$ and $\varepsilon>0$. By the integral-RKHS representation \eqref{eq:Hk-elements} and norm identity \eqref{eq:intRKHSnorm}, there exists an admissible measurable field $a=(a_u)_{u\in\mH}$ with $a_u\in\mH_{k_u}$ such that
\begin{align}
f(\b x)
&=
\int_{\mH}
a_u(\b x)
\,\d\mu_b(u),
\qquad
\b x\in\mX,\\
\int_{\mH}
\|a_u\|_{\mH_{k_u}}^2
\,\d\mu_b(u)
&\le
\|f\|_{\mH_{k_b}}^2+\varepsilon.
\end{align}
Let $w=\d\mu_b/\d\mu_a$. Define $h_u=w(u)a_u$ on $\{w>0\}$ and $h_u=0$ on $\{w=0\}$. Then $h_u\in\mH_{k_u}$ for $\mu_a$-almost every $u$, and the change-of-measure identity gives
\begin{align}
f(\b x)
&=
\int_{\mH}
h_u(\b x)
\,\d\mu_a(u).
\end{align}
Thus $f\in\mH_{k_a}$. Moreover,
\begin{align}
\|f\|_{\mH_{k_a}}^2
&\stackrel{(a)}{\le}
\int_{\mH}
\|h_u\|_{\mH_{k_u}}^2
\,\d\mu_a(u)\\
&\stackrel{(b)}{=}
\int_{\mH}
w(u)^2
\|a_u\|_{\mH_{k_u}}^2
\,\d\mu_a(u)\\
&\stackrel{(c)}{\le}
M
\int_{\mH}
w(u)
\|a_u\|_{\mH_{k_u}}^2
\,\d\mu_a(u)\\
&\stackrel{(d)}{=}
M
\int_{\mH}
\|a_u\|_{\mH_{k_u}}^2
\,\d\mu_b(u)\\
&\stackrel{(e)}{\le}
M\left(
\|f\|_{\mH_{k_b}}^2+\varepsilon
\right).
\end{align}
Here (a) evaluates the infimum in \eqref{eq:intRKHSnorm} at the field $h$, (b) uses homogeneity of each fiber norm, (c) uses $w^2\le Mw$, (d) is the change-of-measure formula, and (e) uses the $\varepsilon$-optimal choice of $a$. Letting $\varepsilon\downarrow0$ proves \ref{lem:local-kernel-domination}.

\medskip
\noindent
\emph{Proof of \ref{lem:finite-moment-canonicalization}.}
We proceed recursively. Set $\mH_{\hat\mu}^{(1)}\coloneqq\mH_{\mu}^{(1)}$, so \eqref{eq:finite-moment-canonicalization} holds for $l=1$ with $q_1=1$. Suppose that it holds at layer $l$. For every Borel set $A\subseteq\S_1\left(\mH_{\hat\mu}^{(l)}\right)$, define
\begin{align}
\hat\mu^{(l)}(A)
\coloneqq
\frac{1}{m_l}
\int_{\Slmu\setminus\{0\}}
\I_A\left(
\frac{u}{q_l\|u\|_{\Hlmu}}
\right)
\|u\|_{\Hlmu}
\,\d\mu^{(l)}(u).
\end{align}
By the induction hypothesis,
\begin{align}
\left\|
\frac{u}{q_l\|u\|_{\Hlmu}}
\right\|_{\mH_{\hat\mu}^{(l)}}
=1,
\end{align}
so this is a probability measure on the required unit sphere. Its first moment is one. By nonnegative $1$-homogeneity of the Brownian kernel,
\begin{align}
\hat k^{(l)}(\b x,\b x')
&=
\int
\kB\left(v(\b x),v(\b x')\right)
\,\d\hat\mu^{(l)}(v)\\
&=
\frac{1}{q_lm_l}
k^{(l)}(\b x,\b x').
\end{align}
For every kernel $k$ and every $c>0$, one has
\begin{align}
\mH_{ck}=\mH_k
\quad\text{as sets},
\qquad
\|f\|_{\mH_{ck}}
=c^{-1/2}\|f\|_{\mH_k}.
\end{align}
Applying this identity with $c=(q_lm_l)^{-1}$ gives
\begin{align}
\mH_{\hat\mu}^{(l+1)}
&=
\mH_{\mu}^{(l+1)},
&
\|f\|_{\mH_{\hat\mu}^{(l+1)}}
&=
\sqrt{q_lm_l}\,
\|f\|_{\mH_{\mu}^{(l+1)}}\\
&&&=
q_{l+1}\|f\|_{\mH_{\mu}^{(l+1)}}.
\end{align}
This proves the induction. The product formula for $q_L$ follows directly from the recursion $q_{l+1}=\sqrt{q_lm_l}$.

\medskip
\noindent
\emph{Proof of \ref{lem:two-ladder-amalgamation}.}
We first construct a possibly noncanonical ladder $\mH_{\mu_a}^{(1:L)}$. At the first layer, all three ladders use the same linear RKHS. Suppose recursively that, for $s\in\{b,c\}$, the identity map on functions defines a continuous embedding
\begin{align}
J_s^{(l)}:
\mH_{\hat\mu_s}^{(l)}
\longrightarrow
\mH_{\mu_a}^{(l)}.
\end{align}
Let
\begin{align}
\nu_s^{(l)}
\coloneqq
\left(J_s^{(l)}\right)_{\#}\hat\mu_s^{(l)},
\qquad
\mu_a^{(l)}
\coloneqq
\frac{1}{2}\nu_b^{(l)}
+
\frac{1}{2}\nu_c^{(l)}.
\end{align}
These are Borel probability measures on $\mH_{\mu_a}^{(l)}$. Before constructing the next kernel, define
\begin{align}
m_l
\coloneqq
\int_{\mH_{\mu_a}^{(l)}}
\|u\|_{\mH_{\mu_a}^{(l)}}
\,\d\mu_a^{(l)}(u).
\end{align}
At the first layer, $m_1=1$. For $l\ge2$, the embeddings already obtained at level $l$ and the unit-sphere support of the original canonical measures give
\begin{align}
0<m_l\le\sqrt{2}.
\end{align}
Thus $\mu_a^{(l)}$ has finite first moment and defines a valid next-layer integral kernel. Because each $J_s^{(l)}$ is the identity on the underlying functions, the kernel induced by $\nu_s^{(l)}$ is exactly $\hat k_s^{(l)}$. Hence
\begin{align}
k_a^{(l)}
=
\frac{1}{2}\hat k_b^{(l)}
+
\frac{1}{2}\hat k_c^{(l)}.
\end{align}
Moreover, $\nu_s^{(l)}\ll\mu_a^{(l)}$ with density bounded by $2$. Applying \ref{lem:local-kernel-domination} yields the next embeddings
\begin{align}
\mH_{\hat\mu_s}^{(l+1)}
&\subseteq
\mH_{\mu_a}^{(l+1)},
&
\|f\|_{\mH_{\mu_a}^{(l+1)}}
&\le
\sqrt{2}\,
\|f\|_{\mH_{\hat\mu_s}^{(l+1)}}.
\label{eq:two-ladder-precanonical-embedding}
\end{align}
This completes the recursive construction and simultaneously shows that every layer measure has a strictly positive finite first moment. Apply \ref{lem:finite-moment-canonicalization} to $\mH_{\mu_a}^{(1:L)}$. It produces a canonical ladder $\LH_{\hat\mu_a}$ satisfying
\begin{align}
\|f\|_{\mH_{\hat\mu_a}^{(L)}}
=
q_L\|f\|_{\mH_{\mu_a}^{(L)}},
\qquad
q_L
=
\prod_{l=1}^{L-1}m_l^{2^{-(L-l)}}.
\end{align}
Using $m_l\le\sqrt2$ for every $l$ gives
\begin{align}
q_L
&\le
\left(\sqrt2\right)^{\sum_{l=1}^{L-1}2^{-(L-l)}}
=
2^{1/2-2^{-L}}.
\end{align}
Finally, by \eqref{eq:two-ladder-precanonical-embedding} at the top layer and the triangle inequality in $\mH_{\mu_a}^{(L)}$,
\begin{align}
\|f_b+f_c\|_{\mH_{\hat\mu_a}^{(L)}}
&=
q_L
\|f_b+f_c\|_{\mH_{\mu_a}^{(L)}}\\
&\le
q_L\sqrt2
\left(
\|f_b\|_{\mH_{\hat\mu_b}^{(L)}}
+
\|f_c\|_{\mH_{\hat\mu_c}^{(L)}}
\right)\\
&\le
2^{1-2^{-L}}
\left(
\|f_b\|_{\mH_{\hat\mu_b}^{(L)}}
+
\|f_c\|_{\mH_{\hat\mu_c}^{(L)}}
\right).
\end{align}
This is \eqref{eq:two-ladder-amalgamation}.
\end{proof}

\begin{lemma}[countable canonical amalgamation]
\label{lemma:TVconvergence}
Let $L\ge2$, and let
\begin{align}
\left\{
\LH_{\hat\mu_j}
\right\}_{j\ge0}
\subseteq
\LHh
\end{align}
be a countable family of canonical $L$-level \BK{}s. Then there exist a canonical ladder $\LH_{\hat\mu_*}\in\LHh$ and a finite constant $A_L>0$, depending only on $L$, such that, for every $j\ge0$,
\begin{align}
\mH_{\hat\mu_j}^{(L)}
&\subseteq
\mH_{\hat\mu_*}^{(L)},
&
\|f\|_{\mH_{\hat\mu_*}^{(L)}}
&\le
A_L
2^{(j+1)(L-1)/2}
\|f\|_{\mH_{\hat\mu_j}^{(L)}},
\quad
f\in\mH_{\hat\mu_j}^{(L)}.
\label{eq:countable-canonical-amalgamation}
\end{align}
\end{lemma}

\begin{proof}
For $j\ge0$ and $l\in[L-1]$, define
\begin{align}
p_{j,l}
\coloneqq
\left(2^l-1\right)2^{-(j+1)l}.
\end{align}
The geometric-series identity gives
\begin{align}
\sum_{j=0}^{\infty}p_{j,l}=1.
\end{align}
We recursively construct a possibly noncanonical ladder $\mH_{\mu_*}^{(1:L)}$. At level one, set $\mH_{\mu_*}^{(1)}$ equal to the common linear RKHS. Suppose that continuous identity embeddings
\begin{align}
J_j^{(l)}:
\mH_{\hat\mu_j}^{(l)}
\longrightarrow
\mH_{\mu_*}^{(l)}
\end{align}
have been constructed. Define
\begin{align}
\nu_j^{(l)}
&\coloneqq
\left(J_j^{(l)}\right)_{\#}\hat\mu_j^{(l)},\\
\mu_*^{(l)}
&\coloneqq
\sum_{j=0}^{\infty}
p_{j,l}\nu_j^{(l)}.
\end{align}
The series converges in total variation because each $\nu_j^{(l)}$ is a probability measure and the weights sum to one. Thus $\mu_*^{(l)}$ is a Borel probability measure on $\mH_{\mu_*}^{(l)}$. Define its first moment by
\begin{align}
m_l
\coloneqq
\int_{\mH_{\mu_*}^{(l)}}
\|u\|_{\mH_{\mu_*}^{(l)}}
\,\d\mu_*^{(l)}(u).
\end{align}
At level one, all component measures are supported on the common unit sphere, so $m_1=1$. For $l\ge2$, the embeddings already obtained at level $l$ give
\begin{align}
m_l
&=
\sum_{j=0}^{\infty}
p_{j,l}
\int
\|J_j^{(l)}u\|_{\mH_{\mu_*}^{(l)}}
\,\d\hat\mu_j^{(l)}(u)\\
&\le
\sum_{j=0}^{\infty}
p_{j,l}
p_{j,l-1}^{-1/2}\\
&=
\frac{2^l-1}{\sqrt{2^{l-1}-1}}
\sum_{j=0}^{\infty}
2^{-(j+1)(l+1)/2}
<\infty.
\end{align}
The moment is positive because each embedded unit-sphere element is a nonzero function and hence has positive norm in the common RKHS. Therefore $\mu_*^{(l)}$ has a strictly positive finite first moment and defines a valid next-layer integral kernel. Since $J_j^{(l)}$ is the identity on functions, the kernel induced by $\nu_j^{(l)}$ is $\hat k_j^{(l)}$, and therefore
\begin{align}
k_*^{(l)}
=
\sum_{j=0}^{\infty}
p_{j,l}\hat k_j^{(l)}.
\end{align}
Moreover,
\begin{align}
\nu_j^{(l)}
\ll
\mu_*^{(l)},
\qquad
0
\le
\frac{\d\nu_j^{(l)}}{\d\mu_*^{(l)}}
\le
p_{j,l}^{-1}
\quad
\mu_*^{(l)}\text{-a.e.}
\end{align}
Applying \Cref{lem:measures_inequality}\ref{lem:local-kernel-domination} gives
\begin{align}
\mH_{\hat\mu_j}^{(l+1)}
&\subseteq
\mH_{\mu_*}^{(l+1)},
&
\|f\|_{\mH_{\mu_*}^{(l+1)}}
&\le
p_{j,l}^{-1/2}
\|f\|_{\mH_{\hat\mu_j}^{(l+1)}}.
\label{eq:countable-precanonical-embedding}
\end{align}
This completes the recursion and simultaneously verifies the finite-moment requirement at every layer. For $l\ge2$, introduce the deterministic bound
\begin{align}
\overline m_l
\coloneqq
\frac{2^l-1}{\sqrt{2^{l-1}-1}}
\frac{2^{-(l+1)/2}}{1-2^{-(l+1)/2}},
\end{align}
so that the preceding computation gives $m_l\le\overline m_l$. Define
\begin{align}
Q_L
\coloneqq
\prod_{l=2}^{L-1}
\overline m_l^{2^{-(L-l)}},
\end{align}
with the empty product interpreted as one. This number depends only on $L$. Apply \Cref{lem:measures_inequality}\ref{lem:finite-moment-canonicalization} to $\mH_{\mu_*}^{(1:L)}$. We obtain a canonical ladder $\LH_{\hat\mu_*}\in\LHh$ and
\begin{align}
q_L
=
\prod_{l=1}^{L-1}
m_l^{2^{-(L-l)}}
\le
Q_L
<\infty
\end{align}
such that
\begin{align}
\|f\|_{\mH_{\hat\mu_*}^{(L)}}
=
q_L\|f\|_{\mH_{\mu_*}^{(L)}}.
\end{align}
Using \eqref{eq:countable-precanonical-embedding} at $l=L-1$ yields
\begin{align}
\|f\|_{\mH_{\hat\mu_*}^{(L)}}
&\le
q_Lp_{j,L-1}^{-1/2}
\|f\|_{\mH_{\hat\mu_j}^{(L)}}\\
&\le
\frac{Q_L}{\sqrt{2^{L-1}-1}}
2^{(j+1)(L-1)/2}
\|f\|_{\mH_{\hat\mu_j}^{(L)}}.
\end{align}
Taking
\begin{align}
A_L
\coloneqq
\frac{Q_L}{\sqrt{2^{L-1}-1}}
\end{align}
proves \eqref{eq:countable-canonical-amalgamation}.
\end{proof}

\begin{lemma}[transverse line segment implies
\Cref{thm:charachteriz_funspace}\ref{thm:charachteriz_funspace-4}]
\label{lem:line-segment-implies-trace}
Let $\mX\subset\R^d$ be compact. Assume that there exist $\b x_0,\b v\in\R^d$ and $\rho>0$ such that
\begin{align}
\b v
\notin
\operatorname{span}\left\{\b x_0\right\},
\qquad
\left\{\b x_0+t\b v:t\in[0,\rho]\right\}
\subseteq
\mX.
\end{align}
Then the trace assumption in
\Cref{thm:charachteriz_funspace}\ref{thm:charachteriz_funspace-4}
is satisfied.
\end{lemma}

\begin{proof}
Define
\begin{align}
\gamma(t)
\coloneqq
\b x_0+t\b v,
\qquad
0\le t\le\rho.
\end{align}
The segment assumption gives $\gamma(t)\in\mX$ for every $t\in[0,\rho]$, and
\begin{align}
\gamma(0)=\b x_0.
\end{align}
Moreover, for all $s,t\in[0,\rho]$,
\begin{align}
\|\gamma(t)-\gamma(s)\|_2
&\stackrel{(a)}{=}
\|(t-s)\b v\|_2
\stackrel{(b)}{=}
|t-s|\,\|\b v\|_2,
\end{align}
where (a) uses the definition of $\gamma$ and (b) uses absolute homogeneity of the Euclidean norm. Hence $\gamma$ is Lipschitz. Define the component of $\b v$ orthogonal to $\b x_0$ by
\begin{align}
\b w
\coloneqq
\begin{cases}
\b v,
& \b x_0=\b0,\\[4pt]
\displaystyle
\b v-
\frac{\langle\b v,\b x_0\rangle}{\|\b x_0\|_2^2}
\b x_0,
& \b x_0\neq\b0.
\end{cases}
\end{align}
The transversality assumption $\b v\notin\operatorname{span}\{\b x_0\}$ implies $\b w\neq\b0$. By construction,
\begin{align}
\langle\b w,\b x_0\rangle
=0,
\qquad
\langle\b w,\b v\rangle
=\|\b w\|_2^2.
\label{eq:transverse-segment-projection-identities}
\end{align}
For $\b x_0=\b0$, both identities are immediate because $\b w=\b v$. For $\b x_0\neq\b0$, the first follows by direct substitution, and the decomposition
\begin{align}
\b v
=
\b w+
\frac{\langle\b v,\b x_0\rangle}{\|\b x_0\|_2^2}
\b x_0
\end{align}
together with $\b w\perp\b x_0$ gives the second. Set
\begin{align}
\b\omega
\coloneqq
\frac{\b w}{\|\b w\|_2^2},
\qquad
u_1(\b x)
\coloneqq
\b\omega^\top\b x,
\qquad
\b x\in\mX.
\end{align}
Since $u_1$ is a linear functional, $u_1\in\Honemu$. Using \eqref{eq:transverse-segment-projection-identities},
\begin{align}
u_1(\b x_0)
&\stackrel{(a)}{=}
\frac{\langle\b w,\b x_0\rangle}{\|\b w\|_2^2}
\stackrel{(b)}{=}0,
\end{align}
where (a) uses the definition of $u_1$ and (b) uses the first identity in \eqref{eq:transverse-segment-projection-identities}. Furthermore, for every $t\in[0,\rho]$,
\begin{align}
u_1(\gamma(t))
&\stackrel{(a)}{=}
\frac{\langle\b w,\b x_0+t\b v\rangle}{\|\b w\|_2^2}\\
&\stackrel{(b)}{=}
\frac{\langle\b w,\b x_0\rangle+t\langle\b w,\b v\rangle}{\|\b w\|_2^2}\\
&\stackrel{(c)}{=}
t.
\end{align}
Here, (a) substitutes the definitions of $u_1$ and $\gamma$, (b) uses linearity of the inner product, and (c) uses both identities in \eqref{eq:transverse-segment-projection-identities}. Thus the theorem's trace condition holds with
\begin{align}
c_1=C_1=1.
\end{align}
\end{proof}

\begin{lemma}[Brownian RKHS characterization]
\label{lem:brownian-rkhs}
The RKHS $\mH_{\kB}$ associated with the Brownian kernel
\begin{align}
\kB(x,x') = \frac{|x|+|x'|-|x-x'|}{2},\qquad \text{for all } x,x' \in \R.
\end{align}
consists of all absolutely continuous functions $h:\R\to\R$ such that $h(0)=0$ and
\begin{align}
\int_{\R} |h'(t)|^2 \d t < \infty,
\end{align}
with norm
\begin{align}
\|h\|_{\mH_{\kB}}^2 = \int_{\R} |h'(t)|^2 \d t.
\end{align}
\end{lemma}

\begin{proof}
We prove that the RKHS associated with $\kB$ is exactly the Cameron--Martin space on $\R$.

\medskip
\noindent
\textbf{Step 1} (feature representation of $\kB$):
For $x\in\R$, define the function $\phi_x:\R\to\R$ by
\begin{align}
\phi_x\left(t\right)
\coloneqq
\left\{
\begin{aligned}
&\I_{\left[0,x\right]}\left(t\right), && x\ge 0,\\
&-\I_{\left[x,0\right]}\left(t\right), && x<0.
\end{aligned}
\right.
\end{align}
We claim that
\begin{align}
\kB\left(x,x'\right)
=
\left\langle
\phi_x,\phi_{x'}
\right\rangle_{L^2\left(\R\right)},
\qquad x,x'\in\R.
\label{eq:brownian-feature-representation}
\end{align}
To verify this, we distinguish the possible signs of $x$ and $x'$. If $x\ge 0$ and $x'\ge 0$, then
\begin{align}
\left\langle
\phi_x,\phi_{x'}
\right\rangle_{L^2\left(\R\right)}
&\stackrel{(a)}{=}
\int_{\R}
\I_{\left[0,x\right]}\left(t\right)
\I_{\left[0,x'\right]}\left(t\right)\d t
\stackrel{(b)}{=}
\left|\left[0,x\right]\cap \left[0,x'\right]\right|
\stackrel{(c)}{=}
\min\left\{x,x'\right\}.
\end{align}
Here (a) follows from the definition of $\phi_x$ and $\phi_{x'}$, (b) uses that the integral of an indicator is the Lebesgue measure of the underlying set, and (c) is immediate since both intervals start at $0$. On the other hand,
\begin{align}
\kB\left(x,x'\right)
&\stackrel{(a)}{=}
\frac{x+x'-\left|x-x'\right|}{2}
\stackrel{(b)}{=}
\min\left\{x,x'\right\},
\end{align}
where (a) uses $|x|=x$ and $|x'|=x'$, and (b) is the elementary identity
\begin{align}
\min\left\{a,b\right\}
=
\frac{a+b-\left|a-b\right|}{2},
\qquad a,b\in\R.
\end{align}
If $x<0$ and $x'<0$, then
\begin{align}
\left\langle
\phi_x,\phi_{x'}
\right\rangle_{L^2\left(\R\right)}
&\stackrel{(a)}{=}
\int_{\R}
\I_{\left[x,0\right]}\left(t\right)
\I_{\left[x',0\right]}\left(t\right)\d t
\stackrel{(b)}{=}
\left|\left[x,0\right]\cap \left[x',0\right]\right|
\stackrel{(c)}{=}
\min\left\{\left|x\right|,\left|x'\right|\right\}.
\end{align}
Here (a) uses that the two minus signs cancel, (b) is again the integral-of-indicator identity, and (c) follows since both intervals end at $0$. Moreover,
\begin{align}
\kB\left(x,x'\right)
&\stackrel{(a)}{=}
\frac{\left|x\right|+\left|x'\right|-\left|x-x'\right|}{2}
\stackrel{(b)}{=}
\min\left\{\left|x\right|,\left|x'\right|\right\},
\end{align}
where (a) is the definition of $\kB$, and (b) is the same elementary identity applied to $\left|x\right|$ and $\left|x'\right|$. If $x\ge 0$ and $x'<0$, or vice versa, then the supports of $\phi_x$ and $\phi_{x'}$ are disjoint up to the single point $0$, hence
\begin{align}
\left\langle
\phi_x,\phi_{x'}
\right\rangle_{L^2\left(\R\right)}
=
0.
\end{align}
On the other hand, if $x\ge 0$ and $x'<0$, then
\begin{align}
\kB\left(x,x'\right)
&\stackrel{(a)}{=}
\frac{x+\left(-x'\right)-\left|x-x'\right|}{2}
\stackrel{(b)}{=}
\frac{x-x'-\left(x-x'\right)}{2}
\stackrel{(c)}{=}
0.
\end{align}
Here (a) uses $|x|=x$ and $|x'|=-x'$, (b) uses $x-x'>0$, and (c) is immediate. The case $x<0<x'$ is identical. Thus \eqref{eq:brownian-feature-representation} holds for all $x,x'\in\R$.

\medskip
\noindent
\textbf{Step 2} (construction of the candidate RKHS):
Let
\begin{align}
\mathcal{H}
\coloneqq
\left\{
h:\R\to\R:
\exists g\in L^2\left(\R\right)
\text{ such that }
h\left(x\right)=\int_0^x g\left(t\right)\d t
\text{ for all } x\in\R
\right\}.
\end{align}
For $h\in\mathcal{H}$ represented by $g\in L^2\left(\R\right)$, define
\begin{align}
\left\|h\right\|_{\mathcal{H}}
\coloneqq
\left\|g\right\|_{L^2\left(\R\right)}.
\end{align}
We first check that this is well defined. Suppose
\begin{align}
\int_0^x g\left(t\right)\d t
=
\int_0^x \tilde g\left(t\right)\d t,
\qquad \forall x\in\R.
\end{align}
Then
\begin{align}
\int_0^x \left(g-\tilde g\right)\left(t\right)\d t
=
0,
\qquad \forall x\in\R.
\end{align}
Hence the absolutely continuous function
\begin{align}
x\mapsto \int_0^x \left(g-\tilde g\right)\left(t\right)\d t
\end{align}
is identically zero, so its derivative vanishes almost everywhere. Therefore,
\begin{align}
g=\tilde g
\qquad \text{a.e. on } \R.
\end{align}
Thus the norm is well defined. Next, for $g\in L^2\left(\R\right)$ and $x\in\R$, one has
\begin{align}
\int_0^x g\left(t\right)\d t
=
\left\langle
g,\phi_x
\right\rangle_{L^2\left(\R\right)}.
\label{eq:integral-as-inner-product}
\end{align}
Indeed, if $x\ge 0$, then
\begin{align}
\left\langle
g,\phi_x
\right\rangle_{L^2\left(\R\right)}
\stackrel{(a)}{=}
\int_{\R} g\left(t\right)\I_{\left[0,x\right]}\left(t\right)\d t
\stackrel{(b)}{=}
\int_0^x g\left(t\right)\d t.
\end{align}
Here (a) uses the definition of $\phi_x$, and (b) restricts the integral to the support of the indicator. If $x<0$, then
\begin{align}
\left\langle
g,\phi_x
\right\rangle_{L^2\left(\R\right)}
&\stackrel{(a)}{=}
-\int_{\R} g\left(t\right)\I_{\left[x,0\right]}\left(t\right)\d t
\stackrel{(b)}{=}
-\int_x^0 g\left(t\right)\d t
\stackrel{(c)}{=}
\int_0^x g\left(t\right)\d t.
\end{align}
Here (a) again uses the definition of $\phi_x$, (b) restricts to the support, and (c) is the standard reversal-of-limits identity. Therefore each $h\in\mathcal{H}$ satisfies
\begin{align}
h\left(x\right)
=
\left\langle
g,\phi_x
\right\rangle_{L^2\left(\R\right)},
\qquad x\in\R.
\end{align}
This shows in particular that point evaluation is continuous on $\mathcal{H}$:
\begin{align}
\left|h\left(x\right)\right|
\stackrel{(a)}{=}
\left|
\left\langle
g,\phi_x
\right\rangle_{L^2\left(\R\right)}
\right|
\stackrel{(b)}{\le}
\left\|g\right\|_{L^2\left(\R\right)}
\left\|\phi_x\right\|_{L^2\left(\R\right)}
=
\left\|h\right\|_{\mathcal{H}}
\sqrt{\kB\left(x,x\right)}.
\end{align}
Here (a) is \eqref{eq:integral-as-inner-product}, and (b) follows from the CSB inequality together with
\begin{align}
\left\|\phi_x\right\|_{L^2\left(\R\right)}^2
=
\left\langle
\phi_x,\phi_x
\right\rangle_{L^2\left(\R\right)}
=
\kB\left(x,x\right)
\end{align}
by \eqref{eq:brownian-feature-representation}. Hence $\mathcal{H}$ is an RKHS.

\medskip
\noindent
\textbf{Step 3} (identification of the kernel and of the norm):
We now show that the reproducing kernel of $\mathcal{H}$ is exactly $\kB$. Fix $x\in\R$ and define
\begin{align}
k_x\left(\cdot\right)
\coloneqq
\kB\left(\cdot,x\right).
\end{align}
By \eqref{eq:brownian-feature-representation},
\begin{align}
k_x\left(y\right)
=
\left\langle
\phi_y,\phi_x
\right\rangle_{L^2\left(\R\right)},
\qquad y\in\R.
\end{align}
Hence $k_x\in\mathcal{H}$, namely it is the function associated with $g=\phi_x$. Moreover, if $h\in\mathcal{H}$ corresponds to $g\in L^2\left(\R\right)$, then
\begin{align}
\left\langle
h,k_x
\right\rangle_{\mathcal{H}}
\stackrel{(a)}{=}
\left\langle
g,\phi_x
\right\rangle_{L^2\left(\R\right)}
\stackrel{(b)}{=}
h\left(x\right).
\end{align}
Here (a) follows from the definition of the inner product on $\mathcal{H}$, and (b) is \eqref{eq:integral-as-inner-product}. Thus $\kB$ is the reproducing kernel of $\mathcal{H}$. By uniqueness of the RKHS associated with a given kernel, we conclude that $\mH_{\kB}
=
\mathcal{H}$. Finally, if $h\in\mH_{\kB}$, then by construction there exists $g\in L^2\left(\R\right)$ such that
\begin{align}
h\left(x\right)=\int_0^x g\left(t\right)\d t.
\end{align}
Therefore $h$ is absolutely continuous on $\R$, satisfies $h\left(0\right)=0$, and has weak derivative
\begin{align}
h'=g
\qquad \text{a.e. on } \R.
\end{align}
Conversely, every absolutely continuous $h:\R\to\R$ with $h\left(0\right)=0$ and $h'\in L^2\left(\R\right)$ admits the representation
\begin{align}
h\left(x\right)
=
\int_0^x h'\left(t\right)\d t,
\qquad x\in\R,
\end{align}
and hence belongs to $\mathcal{H}=\mH_{\kB}$. The norm is
\begin{align}
\left\|h\right\|_{\mH_{\kB}}^2
\stackrel{(a)}{=}
\left\|g\right\|_{L^2\left(\R\right)}^2
\stackrel{(b)}{=}
\int_{\R}\left|h'\left(t\right)\right|^2\d t.
\end{align}
Here (a) is the definition of the Hilbert norm on $\mathcal{H}$, and (b) uses $g=h'$ a.e. This proves the claim.
\end{proof}

\begin{lemma}[closed adaptive balls and lower semicontinuity]
\label{lem:lsc-CL}
Let $\mX\subset\R^d$ be compact, let $L\ge2$, and let Assumption~\ref{assumption:normalizing} hold. For every $r\ge0$, the adaptive ball
\begin{align}
\BL_r
\coloneqq
\left\{
 f\in\HL:
 \CL(f)\le r
\right\}
\end{align}
is closed in $\left(\mC(\mX),\|\cdot\|_\infty\right)$. Consequently, the extended canonical complexity
\begin{align}
\widehat C_{\mathrm{ext}}^{(L)}(f)
\coloneqq
\begin{cases}
\CL(f), & f\in\HL,\\[1mm]
+\infty, & f\in\mC(\mX)\setminus\HL,
\end{cases}
\end{align}
is lower semicontinuous with respect to uniform convergence. Equivalently, for every sequence
\begin{align}
(f_m)_{m\ge1}
\subseteq
\mC(\mX)
\end{align}
satisfying $f_m\to f$ uniformly on $\mX$, one has
\begin{align}
\widehat C_{\mathrm{ext}}^{(L)}(f)
\le
\liminf_{m\to\infty}
\widehat C_{\mathrm{ext}}^{(L)}(f_m).
\label{eq:adaptive-complexity-lsc}
\end{align}
\end{lemma}

\begin{proof}
We first prove that every adaptive ball $\BL_r$ is closed.

\medskip
\noindent
\textbf{Step 1} (compact ambient sets for all layer measures):
For every $l\in[L]$, let $\mathscr U_l$ denote the union of the closed unit balls of all $l$-th-layer RKHSs arising from canonical $l$-level \BK{}s. At the first layer, all canonical ladders use the common linear RKHS $\Honemu$. Define
\begin{align}
\mathscr K_l
\coloneqq
\overline{\mathscr U_l}^{\|\cdot\|_\infty}
\subseteq
\mC(\mX),
\qquad
l\in[L].
\end{align}
Set
\begin{align}
\alpha_l
\coloneqq
2^{-(l-1)},
\qquad
R_{\mX}
\coloneqq
\sup_{x\in\mX}
\|x\|_2.
\end{align}
Let $u\in\mathscr U_l$. Then $u$ belongs to the closed unit ball of the $l$-th-layer RKHS of some canonical ladder. By \Cref{lem:RKHSfunctions_regularity,lem:kernel_boud_induction},
\begin{align}
|u(x)|
&\le
R_{\mX}^{\alpha_l},
\\
|u(x)-u(x')|
&\le
\|x-x'\|_2^{\alpha_l},
\qquad
x,x'\in\mX.
\end{align}
Thus $\mathscr U_l$ is uniformly bounded and equicontinuous. Applying \Cref{thm:arzela-ascoli} gives that
\begin{align}
\mathscr K_l
\end{align}
is compact in $\left(\mC(\mX),\|\cdot\|_\infty\right)$ for every $l\in[L]$.

\medskip
\noindent
\textbf{Step 2} (stability of RKHS balls under simultaneous kernel and function convergence):
We first establish an auxiliary fact. Let $(k_m)_{m\ge1}$ and $k$ be continuous kernels on $\mX$ such that
\begin{align}
\|k_m-k\|_{\infty,\mX\times\mX}
\longrightarrow
0.
\end{align}
Let $u_m\in\mH_{k_m}$ satisfy
\begin{align}
u_m
\longrightarrow
u
\qquad
\text{uniformly on }\mX,
\end{align}
and assume that, for some $c\ge0$,
\begin{align}
\limsup_{m\to\infty}
\|u_m\|_{\mH_{k_m}}
\le
c.
\end{align}
We claim that
\begin{align}
u
\in
\mH_k,
\qquad
\|u\|_{\mH_k}
\le
c.
\end{align}
Fix $N\in\Zp$, points $x_1,\ldots,x_N\in\mX$, and coefficients $a_1,\ldots,a_N\in\R$. By the reproducing property and the Cauchy--Schwarz inequality,
\begin{align}
\left|
\sum_{i=1}^N a_i u_m(x_i)
\right|^2
&\stackrel{(a)}{=}
\left|
\left\langle
u_m,
\sum_{i=1}^N a_i k_m(\cdot,x_i)
\right\rangle_{\mH_{k_m}}
\right|^2
\\
&\stackrel{(b)}{\le}
\|u_m\|_{\mH_{k_m}}^2
\left
\|
\sum_{i=1}^N a_i k_m(\cdot,x_i)
\right\|_{\mH_{k_m}}^2
\\
&\stackrel{(c)}{=}
\|u_m\|_{\mH_{k_m}}^2
\sum_{i,j=1}^N
 a_i a_j k_m(x_i,x_j).
\end{align}
Here (a) uses linearity and the reproducing property, (b) applies the Cauchy--Schwarz inequality, and (c) expands the squared RKHS norm by bilinearity and the reproducing property. Passing to the limit gives
\begin{align}
\left|
\sum_{i=1}^N a_i u(x_i)
\right|^2
\le
c^2
\sum_{i,j=1}^N
 a_i a_j k(x_i,x_j).
\label{eq:rkhs-membership-finite-test}
\end{align}
Define a linear functional on the algebraic span of the kernel sections of $k$ by
\begin{align}
T\left(
\sum_{i=1}^N a_i k(\cdot,x_i)
\right)
\coloneqq
\sum_{i=1}^N a_i u(x_i).
\end{align}
If
\begin{align}
\sum_{i=1}^N a_i k(\cdot,x_i)
=
0
\qquad
\text{in }\mH_k,
\end{align}
then
\begin{align}
\sum_{i,j=1}^N a_i a_j k(x_i,x_j)
=
0,
\end{align}
and \eqref{eq:rkhs-membership-finite-test} gives
\begin{align}
\sum_{i=1}^N a_i u(x_i)
=
0.
\end{align}
Hence $T$ is well defined. Moreover, \eqref{eq:rkhs-membership-finite-test} gives
\begin{align}
|T(v)|
\le
c\|v\|_{\mH_k}
\end{align}
for every $v$ in the algebraic span of the kernel sections. Since this span is dense in $\mH_k$, the functional $T$ extends uniquely to a bounded linear functional on $\mH_k$ with norm at most $c$. By the Riesz representation theorem, there exists $h\in\mH_k$ such that
\begin{align}
T(v)
=
\langle v,h\rangle_{\mH_k},
\qquad
v\in\mH_k,
\end{align}
and
\begin{align}
\|h\|_{\mH_k}
\le
c.
\end{align}
Taking $v=k(\cdot,x)$ gives
\begin{align}
u(x)
&\stackrel{(a)}{=}
T\left(k(\cdot,x)\right)
\\
&\stackrel{(b)}{=}
\left\langle
k(\cdot,x),h
\right\rangle_{\mH_k}
\\
&\stackrel{(c)}{=}
h(x),
\qquad
x\in\mX.
\end{align}
Here (a) follows from the definition of $T$, (b) uses the Riesz representation, and (c) applies the reproducing property. Hence $u=h\in\mH_k$ and $\|u\|_{\mH_k}\le c$, proving the claim. In particular, for a fixed continuous kernel $k$, every closed RKHS ball is closed under uniform convergence in $\mC(\mX)$.

\medskip
\noindent
\textbf{Step 3} (almost-minimizing canonical ladders and simultaneous weak extraction):
Fix $r\ge0$. Let
\begin{align}
(f_m)_{m\ge1}
\subseteq
\BL_r
\end{align}
and assume that
\begin{align}
f_m
\longrightarrow
f
\qquad
\text{uniformly on }\mX
\label{eq:adaptive-ball-uniform-limit}
\end{align}
for some $f\in\mC(\mX)$. By the definition of $\CL$ as an infimum, for every $m\in\Zp$ there exists a canonical ladder
\begin{align}
\LH_{\hat\mu_m}
\in
\LHh
\end{align}
such that
\begin{align}
f_m
&\in
\mH_{\hat\mu_m}^{(L)},
&
\|f_m\|_{\mH_{\hat\mu_m}^{(L)}}
&\le
\CL(f_m)
+
\frac{1}{m}
\le
r
+
\frac{1}{m}.
\label{eq:adaptive-ball-near-optimal-ladder}
\end{align}
For every $l\in[L-1]$, the canonical measure $\hat\mu_m^{(l)}$ is supported on
\begin{align}
\S_1\left(
\mH_{\hat\mu_m}^{(l)}
\right)
\subseteq
\mathscr U_l
\subseteq
\mathscr K_l.
\end{align}
The identity embedding of $\mH_{\hat\mu_m}^{(l)}$ into $\mC(\mX)$ is continuous by the pointwise estimate in \Cref{lem:RKHSfunctions_regularity,lem:kernel_boud_induction}. We may therefore identify $\hat\mu_m^{(l)}$ with its pushforward to $\mathscr K_l$ without changing the corresponding Brownian mixture kernel. Since $\mathscr K_l$ is compact metric, the space $\mP(\mathscr K_l)$ is compact for weak convergence. Hence \Cref{lem:probability-measure-compactness}, followed by a finite diagonal extraction over $l\in[L-1]$, gives a subsequence, not relabeled, and measures
\begin{align}
\mu_\infty^{(l)}
\in
\mP(\mathscr K_l),
\qquad
l\in[L-1],
\end{align}
such that
\begin{align}
\hat\mu_m^{(l)}
\Longrightarrow
\mu_\infty^{(l)},
\qquad
l\in[L-1].
\label{eq:layer-measure-weak-convergence}
\end{align}

\medskip
\noindent
\textbf{Step 4} (limit kernels and support of the limit measures):
Set
\begin{align}
k_\infty^{(0)}(x,x')
&\coloneqq
x^\top x',
&
\hat k_m^{(0)}
&\coloneqq
k_\infty^{(0)},
&
\mH_\infty^{(1)}
&\coloneqq
\Honemu.
\end{align}
For every $l\in[L-1]$, define
\begin{align}
k_\infty^{(l)}(x,x')
&\coloneqq
\int_{\mathscr K_l}
\kB\left(u(x),u(x')\right)
\,\d\mu_\infty^{(l)}(u),
\nonumber\\
\mH_\infty^{(l+1)}
&\coloneqq
\mH_{k_\infty^{(l)}}.
\label{eq:limit-ladder-kernel}
\end{align}
Each $k_\infty^{(l)}$ is positive semidefinite because it is an integral of positive-semidefinite pullback kernels. It is continuous because the map
\begin{align}
(u,x,x')
\longmapsto
\kB\left(u(x),u(x')\right)
\end{align}
is continuous on the compact set $\mathscr K_l\times\mX\times\mX$. We first prove that
\begin{align}
\|\hat k_m^{(l)}-k_\infty^{(l)}\|_{\infty,\mX\times\mX}
\longrightarrow
0,
\qquad
l\in[L-1].
\end{align}
Fix $l\in[L-1]$ and define
\begin{align}
\Phi_l:
\mathscr K_l\times\mX\times\mX
&\longrightarrow
\R,
\\
\Phi_l(u,x,x')
&\coloneqq
\kB\left(u(x),u(x')\right).
\end{align}
The map $\Phi_l$ is uniformly continuous. Fix $\varepsilon>0$. There exist points
\begin{align}
z_1,\ldots,z_N
\in
\mX\times\mX
\end{align}
such that, for every $z\in\mX\times\mX$, one can choose $j\in[N]$ satisfying
\begin{align}
\sup_{u\in\mathscr K_l}
\left|
\Phi_l(u,z)-\Phi_l(u,z_j)
\right|
<
\frac{\varepsilon}{3}.
\end{align}
For every fixed $j\in[N]$, weak convergence in \eqref{eq:layer-measure-weak-convergence} gives
\begin{align}
\int_{\mathscr K_l}
\Phi_l(u,z_j)
\,\d\hat\mu_m^{(l)}(u)
\longrightarrow
\int_{\mathscr K_l}
\Phi_l(u,z_j)
\,\d\mu_\infty^{(l)}(u).
\end{align}
Since the set of indices $[N]$ is finite, for all sufficiently large $m$ the absolute difference of these two integrals is smaller than $\varepsilon/3$ simultaneously for every $j\in[N]$. Adding and subtracting the two integrals at the appropriate net point and applying the triangle inequality gives
\begin{align}
\sup_{x,x'\in\mX}
\left|
\hat k_m^{(l)}(x,x')-k_\infty^{(l)}(x,x')
\right|
<
\varepsilon.
\end{align}
This proves the asserted uniform kernel convergence. We next prove that
\begin{align}
\mu_\infty^{(l)}
\left(
\B_1\left(\mH_\infty^{(l)}\right)
\right)
=
1,
\qquad
l\in[L-1],
\label{eq:limit-measure-unit-ball-support}
\end{align}
where each RKHS ball is viewed as a Borel subset of $\mathscr K_l\subseteq\mC(\mX)$. Fix $l\in[L-1]$ and set
\begin{align}
B_{\infty,l}
\coloneqq
\mathscr K_l
\cap
\B_1\left(\mH_\infty^{(l)}\right).
\end{align}
By Step 2 with the kernel held fixed, $B_{\infty,l}$ is closed in $\mathscr K_l$. For every $q\in\Zp$, define the open set
\begin{align}
O_{q,l}
\coloneqq
\left\{
 u\in\mathscr K_l:
 \operatorname{dist}_\infty\left(u,B_{\infty,l}\right)
 >
 \frac{1}{q}
\right\}.
\end{align}
We claim that, for every fixed $q\in\Zp$, the closed unit ball of $\mH_{\hat\mu_m}^{(l)}$ is disjoint from $O_{q,l}$ for all sufficiently large $m$. Suppose otherwise. Then, after passing to a subsequence, there exist
\begin{align}
u_m
\in
O_{q,l}
\cap
\B_1\left(\mH_{\hat\mu_m}^{(l)}\right).
\end{align}
Because $u_m\in\mathscr U_l\subseteq\mathscr K_l$ and $\mathscr K_l$ is compact, a further subsequence converges uniformly to some $u\in\mathscr K_l$. For $l=1$, one has $\hat k_m^{(0)}=k_\infty^{(0)}$; for $l\ge2$, the preceding uniform kernel convergence gives
\begin{align}
\hat k_m^{(l-1)}
\longrightarrow
k_\infty^{(l-1)}
\qquad
\text{uniformly on }\mX\times\mX.
\end{align}
Applying Step 2 yields
\begin{align}
u
\in
B_{\infty,l}.
\end{align}
On the other hand, continuity of the distance to the closed set $B_{\infty,l}$ gives
\begin{align}
\operatorname{dist}_\infty\left(u,B_{\infty,l}\right)
\ge
\frac{1}{q},
\end{align}
which is a contradiction. Hence
\begin{align}
\hat\mu_m^{(l)}(O_{q,l})
=
0
\end{align}
for all sufficiently large $m$. Since $O_{q,l}$ is open, the Portmanteau theorem and \eqref{eq:layer-measure-weak-convergence} give
\begin{align}
\mu_\infty^{(l)}(O_{q,l})
&\le
\liminf_{m\to\infty}
\hat\mu_m^{(l)}(O_{q,l})
\\
&=
0.
\end{align}
Since $B_{\infty,l}$ is closed,
\begin{align}
\mathscr K_l\setminus B_{\infty,l}
=
\bigcup_{q=1}^{\infty}
O_{q,l}.
\end{align}
Therefore, countable subadditivity gives \eqref{eq:limit-measure-unit-ball-support}. We finally place each limit measure on its natural RKHS. Since $k_\infty^{(l-1)}$ is continuous on the compact metric space $\mX$, the RKHS $\mH_\infty^{(l)}$ is separable and the inclusion
\begin{align}
J_l:
\mH_\infty^{(l)}
\longrightarrow
\mC(\mX)
\end{align}
is continuous and injective. The Lusin--Souslin theorem implies that $J_l\left(\mH_\infty^{(l)}\right)$ is Borel in $\mC(\mX)$ and that $J_l^{-1}$ is Borel on its image. In view of \eqref{eq:limit-measure-unit-ball-support}, we may therefore regard $\mu_\infty^{(l)}$ as a Borel probability measure on $\B_1\left(\mH_\infty^{(l)}\right)$. Define
\begin{align}
m_l
\coloneqq
\int_{\B_1\left(\mH_\infty^{(l)}\right)}
\|u\|_{\mH_\infty^{(l)}}
\,\d\mu_\infty^{(l)}(u).
\label{eq:limit-ladder-first-moment}
\end{align}
Since the measure is supported on the closed unit ball,
\begin{align}
0
\le
m_l
\le
1,
\qquad
l\in[L-1].
\label{eq:limit-first-moment-bound}
\end{align}
Thus the kernels in \eqref{eq:limit-ladder-kernel}, together with the support sets
\begin{align}
\mS_\infty^{(l)}
\coloneqq
\B_1\left(\mH_\infty^{(l)}\right),
\qquad
l\in[L-1],
\end{align}
define a valid, possibly noncanonical, $L$-level \BK{}.

\medskip
\noindent
\textbf{Step 5} (membership and norm bound at the limiting top layer):
By \eqref{eq:adaptive-ball-near-optimal-ladder},
\begin{align}
\limsup_{m\to\infty}
\|f_m\|_{\mH_{\hat\mu_m}^{(L)}}
\le
r.
\end{align}
Together with \eqref{eq:adaptive-ball-uniform-limit} and
\begin{align}
\hat k_m^{(L-1)}
\longrightarrow
k_\infty^{(L-1)}
\qquad
\text{uniformly on }\mX\times\mX,
\end{align}
Step 2 gives
\begin{align}
f
&\in
\mH_\infty^{(L)},
&
\|f\|_{\mH_\infty^{(L)}}
&\le
r.
\label{eq:limit-top-layer-norm}
\end{align}

\medskip
\noindent
\textbf{Step 6} (canonicalization of the limiting ladder):
Assume first that
\begin{align}
m_{l_0}
=
0
\end{align}
for some $l_0\in[L-1]$. Since the integrand in \eqref{eq:limit-ladder-first-moment} is nonnegative, one has
\begin{align}
\|u\|_{\mH_\infty^{(l_0)}}
=
0
\qquad
\mu_\infty^{(l_0)}\text{-a.e.}
\end{align}
Hence $\mu_\infty^{(l_0)}$ is concentrated at the zero function, and therefore
\begin{align}
k_\infty^{(l_0)}
=
0,
\qquad
\mH_\infty^{(l_0+1)}
=
\{0\}.
\end{align}
Applying \eqref{eq:limit-measure-unit-ball-support} recursively shows that every subsequent limit measure is concentrated at zero and every subsequent kernel is zero. In particular,
\begin{align}
\mH_\infty^{(L)}
=
\{0\}.
\end{align}
By \eqref{eq:limit-top-layer-norm}, $f=0$. Since $0\in\BL_r$, the desired conclusion follows in this case. Assume now that
\begin{align}
m_l
\in
(0,1]
\qquad
\text{for every }l\in[L-1].
\end{align}
Apply \Cref{lem:measures_inequality}\ref{lem:finite-moment-canonicalization} to the finite-moment limiting ladder constructed in Step 4. There exists a canonical ladder
\begin{align}
\LH_{\hat\mu_\infty}
\in
\LHh
\end{align}
such that its top layer coincides with $\mH_\infty^{(L)}$ as a set and
\begin{align}
\|g\|_{\mH_{\hat\mu_\infty}^{(L)}}
=
q_L
\|g\|_{\mH_\infty^{(L)}},
\qquad
g\in\mH_\infty^{(L)},
\end{align}
where
\begin{align}
q_L
=
\prod_{l=1}^{L-1}
m_l^{2^{-(L-l)}}.
\end{align}
By \eqref{eq:limit-first-moment-bound},
\begin{align}
q_L
\le
1.
\end{align}
Therefore,
\begin{align}
\CL(f)
&\stackrel{(a)}{\le}
\|f\|_{\mH_{\hat\mu_\infty}^{(L)}}
\\
&\stackrel{(b)}{=}
q_L
\|f\|_{\mH_\infty^{(L)}}
\\
&\stackrel{(c)}{\le}
r.
\end{align}
Here (a) evaluates the infimum defining $\CL$ at the canonical limiting ladder, (b) applies \Cref{lem:measures_inequality}\ref{lem:finite-moment-canonicalization}, and (c) uses $q_L\le1$ together with \eqref{eq:limit-top-layer-norm}. Hence
\begin{align}
f
\in
\BL_r.
\end{align}
Thus $\BL_r$ is closed in $\left(\mC(\mX),\|\cdot\|_\infty\right)$.

\medskip
\noindent
\textbf{Step 7} (lower semicontinuity of the extended complexity):
Let $(f_m)_{m\ge1}\subseteq\mC(\mX)$ satisfy $f_m\to f$ uniformly, and define
\begin{align}
a
\coloneqq
\liminf_{m\to\infty}
\widehat C_{\mathrm{ext}}^{(L)}(f_m).
\end{align}
If $a=+\infty$, then \eqref{eq:adaptive-complexity-lsc} is immediate. Suppose that $a<+\infty$. Choose a subsequence, denoted by $(f_{m_j})_{j\ge1}$, such that
\begin{align}
\widehat C_{\mathrm{ext}}^{(L)}(f_{m_j})
\longrightarrow
a.
\end{align}
Fix $\varepsilon>0$. For all sufficiently large $j$,
\begin{align}
\widehat C_{\mathrm{ext}}^{(L)}(f_{m_j})
\le
a+\varepsilon,
\end{align}
and hence
\begin{align}
f_{m_j}
\in
\BL_{a+\varepsilon}.
\end{align}
Since $\BL_{a+\varepsilon}$ is closed and $f_{m_j}\to f$ uniformly,
\begin{align}
f
\in
\BL_{a+\varepsilon}.
\end{align}
Therefore,
\begin{align}
\widehat C_{\mathrm{ext}}^{(L)}(f)
\le
a+\varepsilon.
\end{align}
Letting $\varepsilon\downarrow0$ proves \eqref{eq:adaptive-complexity-lsc}.
\end{proof}

\begin{lemma}[relative compactness of adaptive balls]
\label{lem:HL_precompact}
Let $\mX\subset\R^d$ be compact and let $r\ge0$. Then $\BL_r$ is relatively compact in $\mC(\mX)$ equipped with $\|\cdot\|_\infty$.
\end{lemma}

\begin{proof}
Set
\begin{align}
\alpha_L
\coloneqq
2^{-(L-1)},
\qquad
R_{\mX}
\coloneqq
\sup_{x\in\mX}\|x\|_2.
\end{align}
For every $f\in\BL_r$, \Cref{thm:charachteriz_funspace}\ref{thm:HLp-analytical} gives
\begin{align}
|f(x)-f(x')|
&\le
r\|x-x'\|_2^{\alpha_L},\\
|f(x)|
&\le
r R_{\mX}^{\alpha_L},
\qquad
x,x'\in\mX.
\end{align}
Hence $\BL_r$ is equicontinuous and uniformly bounded. Applying \Cref{thm:arzela-ascoli} proves the claim.
\end{proof}

\begin{lemma}[common loss bounds]
\label{lem:loss-bounds}
Let $M>0$, and let $u,u_1,u_2,y\in[-M,M]$. For each loss below, the displayed constants $M_\ell$ and $L_\ell$ satisfy
\begin{align}
0
\le
\ell(u,y)
&\le
M_\ell,\\
\left|\ell(u_1,y)-\ell(u_2,y)\right|
&\le
L_\ell|u_1-u_2|.
\end{align}
Consequently, these constants verify
\eqref{eq:controlled-loss-boundedness}--\eqref{eq:controlled-loss-lipschitz}
after identifying the common range bound in \Cref{thm:excess-risk-\BK{}} with $M$.
The following choices are valid.
\begin{enumerate}
\item \textbf{Squared loss.}
For $\ell(u,y)=(u-y)^2$, one may take
\begin{align}
M_\ell
&=
4M^2,
&
L_\ell
&=
4M.
\end{align}

\item \textbf{Absolute loss.}
For $\ell(u,y)=|u-y|$, one may take
\begin{align}
M_\ell
&=
2M,
&
L_\ell
&=
1.
\end{align}

\item \textbf{Huber loss} with parameter $\delta>0$.
For
\begin{align}
\ell(u,y;\delta)
&=
\begin{cases}
\tfrac12(u-y)^2,
& |u-y|\le\delta,\\[3pt]
\delta|u-y|-\tfrac12\delta^2,
& |u-y|>\delta,
\end{cases}
\end{align}
one may take
\begin{align}
M_\ell
&=
\begin{cases}
2M^2,
& \delta\ge2M,\\[3pt]
2\delta M-\tfrac12\delta^2,
& \delta<2M,
\end{cases}
&
L_\ell
&=
\delta.
\end{align}

\item \textbf{$\varepsilon$-insensitive loss} with parameter $\varepsilon>0$.
For
\begin{align}
\ell(u,y;\varepsilon)
&=
\max\left\{0,|u-y|-\varepsilon\right\},
\end{align}
one may take
\begin{align}
M_\ell
&=
\max\left\{0,2M-\varepsilon\right\},
&
L_\ell
&=
1.
\end{align}
\end{enumerate}
\end{lemma}

\begin{proof}
For the squared loss, $|u-y|\le2M$ gives
\begin{align}
0
\le
(u-y)^2
&\le
4M^2.
\end{align}
Moreover,
\begin{align}
\left|(u_1-y)^2-(u_2-y)^2\right|
&\stackrel{(a)}{=}
|u_1-u_2|\,|u_1+u_2-2y|\\
&\stackrel{(b)}{\le}
4M|u_1-u_2|,
\end{align}
where (a) factors a difference of squares and (b) uses
$|u_1|,|u_2|,|y|\le M$. For the absolute loss,
\begin{align}
0
\le
|u-y|
&\le
2M,
\end{align}
and the reverse triangle inequality yields
\begin{align}
\left||u_1-y|-|u_2-y|\right|
&\le
|u_1-u_2|.
\end{align}
For the Huber loss, put $r=u-y$. The loss is nonnegative and nondecreasing as a function of $|r|$. Since $|r|\le2M$, its largest value is
\begin{align}
M_\ell
&=
\begin{cases}
\tfrac12(2M)^2,
& 2M\le\delta,\\[3pt]
\delta(2M)-\tfrac12\delta^2,
& 2M>\delta.
\end{cases}
\end{align}
The Huber function is differentiable in its first argument, and its derivative has magnitude at most $\delta$. The mean-value theorem therefore gives
\begin{align}
\left|
\ell(u_1,y;\delta)-\ell(u_2,y;\delta)
\right|
&\le
\delta|u_1-u_2|.
\end{align}
Thus the displayed value $L_\ell=\delta$ is a valid Lipschitz constant on the relevant range. Finally, the map $r\mapsto\max\{0,|r|-\varepsilon\}$ is nonnegative, nondecreasing in $|r|$, and $1$-Lipschitz. Since $|u-y|\le2M$,
\begin{align}
0
\le
\ell(u,y;\varepsilon)
&\le
\max\left\{0,2M-\varepsilon\right\},\\
\left|
\ell(u_1,y;\varepsilon)-\ell(u_2,y;\varepsilon)
\right|
&\le
|u_1-u_2|.
\end{align}
This proves all four claims.
\end{proof}

\section{External Lemmas}\label{app:Ext_Stats}
This section is dedicated to the external lemmas used in our proofs.

\setcounter{lemma}{0}
\renewcommand{\thelemma}{D\arabic{lemma}}

\begin{lemma}[Arzelà--Ascoli Theorem]
\label{thm:arzela-ascoli}
Let $(\mX,d)$ be a compact metric space and let $\F \subseteq \mC(\mX)$. Then $\F$ is relatively compact in $\left(\mC(\mX), \|\cdot\|_\infty\right)$ iff  $\F$ is equicontinuous and $\F$ is bounded pointwise (in other words, for all $x\in \mX$, $\sup_{f\in \F}|f(x)|<\infty$). 
\end{lemma}

\begin{lemma}[weak sequential compactness of probability measures]
\label{lem:probability-measure-compactness}
Let $(K,d)$ be a compact metric space. Then every sequence $(\mu_m)_{m\ge1}\subset\mP(K)$ admits a weakly convergent subsequence whose limit belongs to $\mP(K)$.
\end{lemma}

This is the compact-space case of Prokhorov's theorem; see, for example, Billingsley~\cite{billingsley12probability}.

\bibliographystyle{emss}
\bibliography{BIB/bkl_references}

\end{document}